%% file: main.tex
\documentclass[3p,times,authoryear]{elsarticle}

\usepackage[T1]{fontenc}
\usepackage{amsmath,amssymb,amsfonts,amsthm,bm,mathtools}
\usepackage{graphicx}
\usepackage{caption}
\usepackage{subcaption}
\usepackage{booktabs}
\usepackage{multirow}
\usepackage{array}
\usepackage{enumitem}
\usepackage{float}
\usepackage{xcolor}
\usepackage{microtype}
\usepackage{url}
\usepackage[hidelinks]{hyperref}
\hypersetup{
  hidelinks,
  pdftitle={Equal-Information Evaluation of Credit Assignment in Modular Financial Reinforcement Learning},
  pdfauthor={Gennadiy Chistiakov; Yury Yanovich}
}

\newtheorem{proposition}{Proposition}

\newcommand{\E}{\mathbb{E}}
\newcommand{\Var}{\operatorname{Var}}
\newcommand{\R}{\mathbb{R}}

\newcounter{amatsalgorithm}

\setlist{nosep,leftmargin=1.4em}
\journal{arXiv preprint}

\begin{document}

\begin{frontmatter}

\title{Equal-information evaluation of credit assignment in modular financial reinforcement learning}

\author[gi]{Gennadiy Chistiakov}
\affiliation[gi]{
  organization={Independent Researcher},
  city={Bangkok},
  postcode={10100},
  country={Thailand}
}

\author[sk]{Yury Yanovich\corref{cor1}}
\ead{y.yanovich@skoltech.ru}
\cortext[cor1]{Corresponding author.}
\affiliation[sk]{
  organization={Skolkovo Institute of Science and Technology},
  city={Moscow},
  postcode={121205},
  country={Russia}
}

\begin{abstract}
Credit assignment in modular financial reinforcement learning is difficult to evaluate because an attributed signal may improve a component loss without adding decision-relevant information or improving economic outcomes. We propose an evaluation framework combining equal-information reconstruction controls, fixed-utility comparisons and prespecified downstream guardrails. The framework is instantiated in an attribution-guided modular trading architecture and tested in three controlled studies covering update allocation, supervision-target representation and portfolio risk control. We find that residual-based allocation is exactly reproduced by an ordinary loss-budget control under the disclosed interface; direct shortfall supervision improves component prediction but fails decoded-outcome preservation; and component-wise critic supervision worsens risk-condition certainty equivalent and drawdown relative to scalar supervision under fixed utility. These results identify three distinct evaluation failures: redundant attribution signals, representation-induced prediction trade-offs, and misalignment between component supervision and downstream risk outcomes. A preserved historical engineering case further shows why module-level learning claims require run-level provenance. The contribution is a reusable methodology for evaluating credit assignment in expert systems operating under partial observability, delayed feedback and economic risk constraints.
\end{abstract}


\end{frontmatter}

\input{main_body}

\section*{Data and code availability}
Historical sources are retained locally with hashes and provenance metadata. The manuscript includes selected historical figures and transcribed point estimates, but raw OHLCV/funding files, proprietary deployment material and credentials are not redistributed. Detailed artifact disposition, environment records and preservation manifests are provided in Supplementary Material S8.

\section*{Declaration of competing interests}
The authors declare that they have no known competing financial interests or personal relationships that could have appeared to influence the work reported in this paper.

\section*{Funding}
This research did not receive any specific grant from funding agencies in the public, commercial, or not-for-profit sectors.

\section*{Ethical statement}
This study does not involve human subjects or personal data collection. It analyzes synthetic controlled experiments and preserved non-personal trading artifacts. The manuscript is research documentation and explicitly does not provide investment advice, a recommendation to trade, or an assurance of returns.

\section*{Declaration of generative AI and AI-assisted technologies}
During the preparation of this work, the authors used ChatGPT (OpenAI), Claude (Anthropic), and Qwen (Alibaba Cloud) large language model tools to improve language, structure, formatting and editorial clarity. After using these tools, the author(s) reviewed and edited the content as needed and take full responsibility for the content of the published article. LLMs are also reported in the manuscript as components of the AMATS case architecture and as the subject of the proposed prompt/response examples; those materials were designed, written and verified by the authors. No generative AI tool was used to generate experimental data, statistical analyses, mathematical proofs, core scientific claims or conclusions, nor to create or manipulate primary research images.

\bibliographystyle{elsarticle-harv}
\bibliography{references}

\clearpage

\appendix

\section{Supplementary Summary}
This supplementary material provides the detailed protocols, diagnostics, sensitivity analyses, historical replay records, reproduction records, preliminary configuration and prompt examples referenced in the main manuscript. It is intended to support review and reproducibility. The main manuscript contains the central claims and primary evidence.

\section{Overview}
This document contains the following supplementary sections:
\begin{itemize}
  \item S1: Accounting and gradient diagnostics.
  \item S2: Allocation study: detailed method and complete results.
  \item S3: Representation study: generator and supplementary evidence.
  \item S4: Portfolio control: complete contrasts and diagnostics.
  \item S5: Secondary historical Bar replay.
  \item S6: Original full-system hypotheses.
  \item S7: Historical source context and information diagnostics.
  \item S8: Reproduction and validation record.
  \item S9: Preliminary full-system configuration.
  \item S10: Proposed LLM prompt and response examples.
\end{itemize}

\renewcommand{\thesection}{S\arabic{section}}
\renewcommand{\thesubsection}{S\arabic{section}.\arabic{subsection}}
\renewcommand{\thesubsubsection}{S\arabic{section}.\arabic{subsection}.\arabic{subsubsection}}
\setcounter{section}{0}
\numberwithin{equation}{section}
\numberwithin{table}{section}
\numberwithin{figure}{section}

\section{Accounting and Gradient Diagnostics}\label{app:diagnostics}

\subsection{A sufficient local gradient condition}\label{app:gradient}

For shared parameters $\theta$, consider
\begin{equation}
\mathcal L_J=\mathcal L_D+\lambda_E\mathcal L_E+\lambda_R\mathcal L_R,\qquad \lambda_E,\lambda_R\geq0.
\end{equation}
Writing $g_j=\nabla_\theta\mathcal L_j$ and taking a step of size $\eta$ gives
\begin{align}
\theta'&=\theta-\eta(g_D+\lambda_Eg_E+\lambda_Rg_R),\\
\mathcal L_D(\theta')-\mathcal L_D(\theta)
&=-\eta\bigl(\|g_D\|^2+\lambda_Eg_D^\top g_E+\lambda_Rg_D^\top g_R\bigr)+O(\eta^2).
\label{eq:conflict}
\end{align}
Multi-task optimization studies competing gradients directly \citep{SenerKoltun2018,Yu2020PCGrad}. A negative pairwise inner product reduces first-order progress, but does \emph{not} by itself prove that a task deteriorates. A sufficient local condition for directional loss to increase is
\begin{equation}
\lambda_Eg_D^\top g_E+\lambda_Rg_D^\top g_R<-\|g_D\|^2,
\label{eq:increase}
\end{equation}
with a sufficiently small step and locally smooth losses. The condition applies to the shared differentiable representation in this diagnostic; exchange execution and language-model interfaces are evaluated through their separate contracts.

For the two fixed tasks
\begin{equation}
\mathcal L_D=(\theta-1)^2,\quad \mathcal L_E=(\theta+1)^2,\quad
\theta_J^*(\lambda_E)=\frac{1-\lambda_E}{1+\lambda_E},
\label{eq:toy}
\end{equation}
equal weighting gives $\theta_J^*=0$, although the individual minima occur at $1$ and $-1$. At $\theta=0.5$, an equal-weight step with $\eta=0.05$ moves to $0.4$: total loss falls from $2.50$ to $2.32$, while directional loss rises from $0.25$ to $0.36$. Changing weights moves the compromise even when the tasks remain fixed (Figure~\ref{fig:gradient}). Separate parameters remove this particular conflict, but also change the hypothesis class. This toy calculation is consequently a mechanism illustration, not a capacity-matched demonstration that AMATS outperforms joint learning.

\begin{figure}[tb]
\centering
\includegraphics[width=\linewidth]{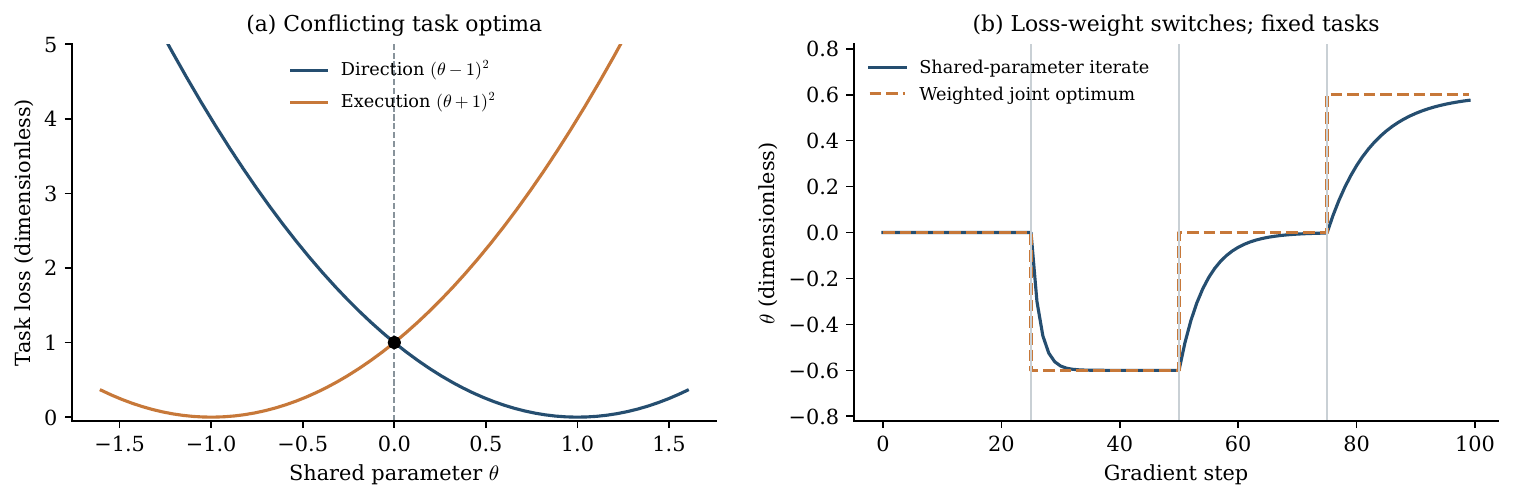}
\caption{Deterministic quadratic diagnostic, not market data. (a) The equal-weight optimum sacrifices both task-specific optima. (b) Execution-loss weights switch through $1,4,1,0.25$ in blocks of 25 steps, with $\eta=0.05$ and $\theta_0=0$. The dashed optima are $0,-0.6,0,0.6$. These are loss-weight switches, not detected market-regime transitions.}
\label{fig:gradient}
\end{figure}

\paragraph{Numerical check.}
The one-step calculation confirms that the shared objective can improve while one task worsens: joint loss changes by $-0.18$, and directional loss by $+0.11$. The moving optima in Figure~\ref{fig:gradient} match the analytic values in Equation~\eqref{eq:toy}. After the final 25-step block, the parameter is approximately $0.578621$, approaching the new optimum $0.6$. The finite adaptation lag is an optimization property of this toy problem, not a measured financial recovery time.

\subsection{Paired scalar accounting scenarios}

For the separate scalar diagnostic, intended notional $N>0$, signed opportunity $b$ bps, executed fraction $f$, round-trip cost $c$ bps per executed notional and signed funding debit $u$ satisfy
\begin{align}
p^{\mathrm{net}}&=fb-fc-fu,\label{eq:net}\\
\mathrm{IS}&=(1-f)b+fc,\label{eq:shortfall}\\
p^{\mathrm{net}}&=b-\mathrm{IS}-fu.\label{eq:identity}
\end{align}
Multiply by $N/10^4$ for quote currency. Unfilled opportunity is signed and can be negative; a training penalty for non-fill is not necessarily that realized opportunity.

\subsection{Experiment A: fixed tasks, changing weights}
We evaluate Equation~\eqref{eq:toy} with $\theta_0=0$ and step size $0.05$ for 100 steps. The execution-loss weight takes values $1,4,1,0.25$, each for 25 steps. Both task optima remain fixed. This isolates movement induced by objective weighting without asserting that a market regime has changed. In addition to the path, the script calculates the one-step example in Section~\ref{app:gradient}. No confidence interval is appropriate for these deterministic identities.

\subsection{Experiment B: controlled attribution shocks}
We generate $n=2048$ independent signed opportunities
\begin{equation}
b_i=12+30\epsilon_i\quad\text{bps},\qquad\epsilon_i\sim\mathcal N(0,1).
\label{eq:generator}
\end{equation}
The sign is already expressed relative to the intended trade direction; these are not raw asset returns or forecasts learned from data. Each episode is a separate normalized round-trip opportunity, not a minute of a compounded portfolio path. Baseline round-trip costs are 4 bps fees, 2 bps spread, 1 bp impact, and 0.5 bp delay, giving $c=7.5$ bps. The signed funding debit is $u=0.5$ bp. These deliberately simple imposed values permit exact reconciliation, not realistic market calibration.

Four scenarios reuse the same $\epsilon_i$:
\begin{enumerate}
\item \textit{Reference:} full execution, $f=1$.
\item \textit{Execution shock:} spread increases by 6 bps and delay by 2 bps; the directional opportunity is unchanged.
\item \textit{Forecast shock:} the opportunity decreases by 8 bps; costs are unchanged. This is an imposed deterioration, not the output of a trained forecasting model.
\item \textit{Partial fills:} $f=0.5$ with the baseline opportunity and per-executed-notional costs. No liquidity or execution model is inferred from this fixed fraction.
\end{enumerate}
The two shocks are equal in magnitude by construction, making the inverse problem from scalar P\&L deliberately non-identifiable. The experiment tests the disclosed distinction, not its prevalence in market data. The attribution record is computed from known simulator quantities; field reconstruction from noisy production logs is a different task.

Table~\ref{tab:synthetic} reports all four scenarios. Their population mean net outcomes are $4$, $-4$, $-4$ and $2$ bps. Sample means differ through Monte Carlo variation under the specified generator, not through forecast estimation error.

\begin{table}[tb]
\centering
\caption{Synthetic diagnostic results, 2,048 paired opportunities per scenario. Values are bps of intended notional, not portfolio returns. MC SE is the standard error of the simulated mean under the imposed independent generator, not uncertainty about market performance.}
\label{tab:synthetic}
\small
\begin{tabular}{lrrrrr}
\toprule
Scenario & Opportunity & Shortfall & Funding & Net & MC SE\\
\midrule
Reference & 12.123 & 7.500 & 0.500 & 4.123 & 0.653\\
Execution shock & 12.123 & 15.500 & 0.500 & -3.877 & 0.653\\
Forecast shock & 4.123 & 7.500 & 0.500 & -3.877 & 0.653\\
Partial fills & 12.123 & 9.811 & 0.250 & 2.061 & 0.326\\
\bottomrule
\end{tabular}
\end{table}

The execution shock reduces net outcome by exactly 8 bps for every episode while leaving the benchmark opportunity unchanged. The forecast shock produces the identical episode-level net series while leaving execution cost unchanged. Thus the net series alone cannot discriminate these two constructed explanations. The enriched record distinguishes them because it retains the benchmark and execution terms. This demonstrates identification within the constructed record; deployment diagnosis additionally requires observed benchmark and execution fields and validation of the classifier.

The largest absolute reconciliation residual over all 8,192 scenario-episode records is below $10^{-12}$ bps. Partial-fill net outcomes equal one half of the paired baseline outcomes. In the partial-fill scenario, 684 of 2,048 episodes have negative signed unfilled opportunity; treating every non-fill as a positive realized cost would violate the stated convention. These checks exercise the ledger identity and a relevant sign edge case, not exchange fill realism.

The original scalar experiments establish accounting and representation properties; the new learning matrix additionally tests real parameter updates. These evidence classes are not interchangeable. In particular, the loss-budget equality shows that suitably observed task labels already supply the information used by this public scheduler. A more informative attribution mechanism would need to establish additional observable information or better decisions, rather than receive credit merely for renaming the same residuals.

\section{Allocation Study: Detailed Method and Results}\label{app:allocation}

\subsection{Executed inputs, targets and availability}
Both configurations pass an ordered 11-vector to each learner:
\begin{equation}
\bm x_{t,k}=(\tilde f_{t,1},\ldots,\tilde f_{t,6},
r_{t,0},r_{t,1},r_{t,2},a_{t,k},v_{t,k}),\qquad k\in\{D,E\}.
\label{eq:inputs}
\end{equation}
The first six coordinates are prefix-standardized raw features. The next three are a regime one-hot vector; unused slots are zero, and the Learn $K=1$ sensitivity zeros all three. Direction $D$ always uses $(a,v)=(0,0)$. Execution $E$ uses the signed action and nonnegative size defined below; these two coordinates are not standardized.

In Learn, $f_{t,1},\ldots,f_{t,6}$ are dimensionless Gaussian AR(1) factors in generator order, not surrogate names for market indicators. Execution side is $a_{t,E}=+1$ if $f_{t,1}+0.25f_{t,2}\geq0$ and $-1$ otherwise, and size is $v_{t,E}=0.5+0.25\tanh|f_{t,3}|$, a dimensionless generator action scale, not a base quantity. The regime statistic is an EWMA of $|f_{t,5}|$, initialized to its first value. Direction targets are sums of the next $h$ generated bps increments, available at $t+h$. The execution target is the generator's state/action-dependent signed cost in bps, observed only under its submission mask and available at $t+d_t$; unsubmitted outcomes never enter replay or routing. The released generator specifies the coefficients and intervention paths detailed in this section; hidden fault labels are not learner inputs.

More precisely, the Learn targets are
\begin{align}
y_{t,D}&=\sum_{j=1}^{h}\left(\bm\beta_{t+j}^{\top}\bm f_{t+j}+\xi^D_{t+j}\right),\\
y_{t,E}&=1.5+0.6a_{t,E}f_{t,1}+0.5v_{t,E}+0.6|f_{t,5}|
+\delta_t^E\left(3+1.5\tanh f_{t,4}\right)+\xi^E_t.
\label{eq:synthetictargets}
\end{align}
The base direction vector is $(2,-1,0.5,0,0,0)$ and the changed vector is $(-1.5,1,-0.5,0.75,0,0)$. Independent Gaussian noises have standard deviations 3 and 0.25 bps for $D$ and $E$. The scheduled choice of $\bm\beta$ and execution-shift indicator $\delta_t^E$ belongs to the generator; the learner observes only the inputs and eligible delayed labels, not these intervention states.

\subsection{Preprocessing, networks and prequential learning}\label{sec:training}
Table~\ref{tab:parameters} collects executed main settings; the complete frozen sensitivity grid is in Section~\ref{sec:protocol}. Let $t_0$ be the end-exclusive normalization prefix: 512 in Learn and 528 in the complete Bar folds (16 lookback hours followed by 512 valid clocks). Feature means and population standard deviations use those prefix features only. Target means/scales use finite observed labels with maturity strictly before $t_0$, separately per task. The direction scale does not use later horizon outcomes; execution scale excludes unsubmitted or still-pending labels. Standardization is fixed thereafter. Prefix-calibrated regime boundaries are the $j/K$ quantiles of the same prefix statistic. Crossing must exceed the relevant boundary by the configured margin and persist for four clocks; candidate persistence may accumulate during the minimum hold. Prefix transformations are not available earlier in the prefix: those records initialize replay, not reported online predictions.

\begin{table}[tb]
\centering
\caption{Executed main-configuration parameters. These values specify the public small learners and simulator, not missing full-stack deployment choices.}
\label{tab:parameters}
\small
\setlength{\tabcolsep}{4pt}
\begin{tabular}{@{}p{0.35\linewidth}p{0.30\linewidth}p{0.27\linewidth}@{}}
\toprule
\raggedright Quantity & \raggedright Learn & \raggedright Bar\tabularnewline
\midrule
\raggedright Decision clock; direction horizon & \raggedright Abstract clock; $h=4$ & \raggedright One hour; $h=4$ hours\tabularnewline
\raggedright Normalization; feature/target scale floors & \raggedright 512 clocks; $10^{-8}$ / 0.1 bp & \raggedright 512 valid clocks; same floors\tabularnewline
\raggedright Separate / shared network & \raggedright $2(11\!-\!8\!-\!1)$ / $11\!-\!15\!-\!2$ & \raggedright Same; tanh hidden, linear output\tabularnewline
\raggedright Parameter counts & \raggedright $2\times105=210$ / 212 & \raggedright Same\tabularnewline
\raggedright Adam $(\eta,\beta_1,\beta_2,\epsilon)$ & \raggedright $(0.003,0.9,0.999,10^{-8})$ & \raggedright Same\tabularnewline
\raggedright Batch / replay / eligibility & \raggedright 16 / 256 per task / 16 labels & \raggedright Same\tabularnewline
\raggedright Task-gradient cap / norm cap & \raggedright 2 per clock / 5 & \raggedright Same\tabularnewline
\raggedright Regime $K$ / EWMA decay & \raggedright 3 / 0.95 & \raggedright Same\tabularnewline
\raggedright Minimum hold / candidate persistence & \raggedright 16 / 4 clocks & \raggedright 16 / 4 hours\tabularnewline
\raggedright Boundary margin & \raggedright 0.05 prefix-statistic SD & \raggedright Same\tabularnewline
\raggedright Residual EWMA initial / decay / cap & \raggedright 1 / 0.97 / 10 & \raggedright Same\tabularnewline
\raggedright Minimum task share / initial credit & \raggedright 0.25 / 0 & \raggedright Same\tabularnewline
\raggedright Execution observations / delay & \raggedright Bernoulli 0.8 / uniform $\{1,2,3\}$ & \raggedright Submitted fills / 2 bars + 1 s after open\tabularnewline
\raggedright Initial capital / target fraction / gate & \raggedright No account & \raggedright 100,000 quote / 0.25 / 5 bps\tabularnewline
\raggedright Slippage / taker fee & \raggedright Generator target, not venue costs & \raggedright 1 bp / 4 bps\tabularnewline
\raggedright Order tolerance / ledger tolerance & \raggedright Not applicable & \raggedright $10^{-12}$ base / $10^{-8}$ quote\tabularnewline
\raggedright Final scored-origin exclusion & \raggedright 8 clocks & \raggedright 8 hours; all marks retained\tabularnewline
\bottomrule
\end{tabular}
\end{table}

All networks use float64. Weight matrices have independent zero-mean Gaussian entries with standard deviation $1/\sqrt{\text{fan-in}}$; biases, Adam first/second moments and step counters start at zero. Each separate task has its own optimizer; the shared model has one. Squared error is computed on standardized targets. A global Euclidean gradient-norm cap is applied to each task gradient. Joint learning averages the clipped task gradients before its Adam step. PCGrad projects each conflicting clipped gradient against the other original clipped gradient, then averages; absent conflict it averages directly \citep{Yu2020PCGrad}. It does not randomly reorder these two projections.

For seed $s$, the learner uses the sixth child of \texttt{SeedSequence(s).spawn(6)}, with separate descendants for model initialization and each task's replay. Learn uses the first five root children for factors, direction noise, execution noise, submissions and delays. NumPy PCG64 draws are paired across modular variants. Replay is a per-task FIFO of the latest eligible records, not regime-filtered; each selected task draws a minibatch uniformly with replacement. Prefix labels initialize replay with no stored prediction and hence no residual-EWMA update. Late prefix labels are delivered at their actual maturity, also without retrospective predictions.

At each post-prefix clock, the runners deliver newly available labels with their origin-time features and stored predictions, update residual state and replay, select task slots, train, then predict and store the new origin's outputs. Bar first processes its timestamped funding/fill events, then mature direction labels and any remaining calibration labels before that close-time update. Thus \emph{stored pre-update prediction} means before learning from its own target, not before every update at its origin clock. Diagnostic future-error arrays in the simulator are scoring outputs only, never training inputs. Training, diagnostic validation and test continue the same prequential process; test predictions can use earlier mature test labels but not their own future outcomes.

\subsection{Residual-based update allocation}\label{sec:scheduler}
Let $\hat y_{i,k}^{\,\mathrm{stored}}$ and $\tilde y_{i,k}$ be a stored standardized prediction and its standardized label. On delivery, the origin regime $z_i$ receives
\begin{equation}
a_{k,z_i}\leftarrow \lambda_a a_{k,z_i}
+(1-\lambda_a)\min\!\left(|\hat y_{i,k}^{\,\mathrm{stored}}-\tilde y_{i,k}|,c_a\right).
\label{eq:residual}
\end{equation}
Unavailable labels and missing/nonfinite stored predictions do not update this EWMA. At a training clock, the current regime selects the direction share
\begin{equation}
p_D(z)=\operatorname{clip}\!\left(\frac{a_{D,z}}{a_{D,z}+a_{E,z}},
\epsilon_q,1-\epsilon_q\right).
\label{eq:budget}
\end{equation}
An unseen task/regime uses the initial EWMA in Table~\ref{tab:parameters}; a zero denominator maps to $1/2$. The allocator has one persistent fractional credit $c$, not a separate credit per regime. For each of the two requested slots, set $c\leftarrow c+p_D$; if $c\geq1$, request $D$ and subtract one, otherwise request $E$. A request whose replay has fewer than 16 labels is unused: credit still advances, and the slot is neither reassigned nor carried as extra future work. Minimum shares constrain requested long-run allocation, not guaranteed eligible updates at every clock.

Uniform requests $D,E$ in that order. Attributed and loss-budget methods share the above allocator; the latter takes residuals directly from ordinary task labels. Shared controls request the two task gradients and take one combined Adam step if any is eligible; separate models take an Adam step for each eligible slot, including consecutive slots for the same task. The budget is therefore at most two task-gradient evaluations per clock, not identical optimizer steps, FLOPs or wall time. Parameter counts are approximately, not exactly, matched. The saved gradient counts, update counts and measured runtimes are distinct audit quantities.

\subsection{Controlled learning paths and legitimate labels}
Six fixed seeds, $11,23,37,53,71,89$, generate paired paths for stationary, direction-shift, execution-shift and combined-shift conditions. Six observed Gaussian factors follow an AR(1) process with coefficient 0.8. One-step bps increments are a linear function of current factors plus Gaussian noise of standard deviation 3; four-step sums form the forecast target and mature four clocks later. Execution-cost observations contain state/action-dependent costs and noise of standard deviation 0.25 bp; only the independently submitted 80\% of actions yield labels, after a delay of one, two or three clocks. Changed coefficients or costs alternate in 512-clock test blocks. The changed coefficient, fault type, future increments and unsubmitted outcomes are held by the generator, not given to learners. The complete coefficient vectors and cost formula are released in the contract.

Each path has 4,096 clocks. The first 512 define normalization and initial eligible replay; subsequent training ends at 1,536, diagnostic validation at 2,048, and scored test decisions at 4,088. The final eight clocks allow label maturation. Test evaluation is \emph{prequential}: updates continue from already available past test observations, under a fixed algorithm; it is not a frozen-weight holdout. Observable regimes use a causal EWMA of an observed factor magnitude, training-prefix quantile boundaries, four-clock candidate persistence and a 16-clock minimum hold. Three input slots encode regimes, and two encode the submitted action; unobserved regimes never substitute for detected ones.

The primary endpoint is the paired mean of the two training-standardized squared errors on common observable test rows, not P\&L. Error components and bps MAE are reported separately. The primary contrast is attributed-budget minus ordinary-uniform loss, so a negative value favors allocation by residuals. A paired bootstrap resamples seeds and common circular 64-clock blocks within each seed, retaining the missing-label mask, for 20,000 replicates with seed 90217. We report descriptive 95\% intervals and 98.75\% intervals for four condition-specific primary contrasts. Synthetic seeds jointly vary paths, noise and initialization; their variance contributions are not separately identified.

The prespecified noninferiority margin is 0.05 standardized squared-error units per task. A controlled H3 analogue requires improvement in the shifted task (upper paired bound below zero) and noninferiority in the other (upper bound below 0.05); both must improve in the combined-shift condition. These eight component contrasts use a separate 99.375\% interval family. An average loss cannot conceal a failed component-level criterion. This is not a market IC/shortfall test. Descriptive recovery after controlled interventions uses a 64-clock reference, 32-clock smoothing, a threshold of 1.1 times reference loss and censoring at 256 clocks. These analyst-known intervention times are never supplied to learners.

\subsection{Bootstrap precision and pairing audit}\label{sec:precision}
The reported intervals now use a declared 20,000-resample numerical-precision check of the saved outputs, preserving the original 1,000-resample analyses, seeds, sample, estimands, block lengths, confidence families and 0.05 noninferiority margin. This post-analysis increase is not a new prospective protocol freeze or a hyperparameter search. Both analyzers use NumPy's default linear quantile interpolation: for probability $p$, interpolate adjacent sorted draws at zero-based position $(B-1)p$. At $B=1{,}000$, the two-sided 98.75\% and 99.375\% lower tails correspond to only about 6.25 and 3.125 expected draws; increasing $B$ addresses Monte Carlo endpoint precision, not sampling uncertainty.

For Learn, each replicate draws six seed indices with replacement, then independently draws circular 64-clock blocks for each selected seed occurrence, concatenated and truncated to the 2,040-clock score window. All methods use the same seed/clock draws. Resampled observation/maturity/finite masks exclude unavailable rows rather than replacing them by zero; valid rows are averaged within each selected seed before equally averaging seeds. For Bar, BTC/ETH returns are first equally averaged on matching seed/date axes; one shared draw of three seed indices is used across the two years. Seven-day circular blocks are drawn independently for each year/selected-seed occurrence, truncated to 61 days and shared across methods; yearly means are equally pooled. Assets and years are not resampled as additional independent observations.

The audit preserves all point estimates and the adverse primary/component conclusions. One exploratory sensitivity classification changes (Section~\ref{sec:protocol}). More bootstrap draws do not create new market periods or independent training seeds, and reported decimal places do not imply that the interval endpoints are exact. The original and precision outputs, invocation overrides and endpoint comparisons are linked in \path{METHOD_REPRODUCIBILITY_AUDIT.md}.

\subsection{Complete allocation results}

All 120 main runs and 144 sensitivity runs completed without failed or missing cells. Every saved learner has finite weights different from its initialization, and independently recomputed weight hashes agree for all 264 runs. The main matrix required 83.68 seconds and the sensitivity matrix 100.56 seconds in one CPU process with one BLAS thread. These times exclude analysis, implementation and manuscript preparation; they are not full-stack training-cost estimates.

The primary contrast is mixed in sign and fails the superiority criterion against ordinary uniform supervision (Table~\ref{tab:learningcontrast}). The stationary condition instead shows a loss increase whose multiplicity-adjusted interval excludes zero. Lower point estimates under execution and combined shifts have intervals crossing zero. Descriptive paired comparisons with joint and PCGrad likewise yield inconclusive superiority contrasts; their complete intervals and task-level errors are retained in the result CSVs.

\begin{table}[tb]
\centering
\caption{Attributed minus uniform combined loss. Negative favors attributed allocation. Paired seed/time-block intervals use 20,000 resamples. The 98.75\% column adjusts the four primary condition-specific contrasts; 95\% intervals are descriptive.}
\label{tab:learningcontrast}
\small
\begin{tabular}{lrrr}
\toprule
Condition & Difference & 95\% interval & 98.75\% interval\\
\midrule
Stationary & $+0.0133$ & $[+0.0065,+0.0204]$ & $[+0.0048,+0.0225]$\\
Direction shift & $+0.0066$ & $[-0.0056,+0.0187]$ & $[-0.0091,+0.0216]$\\
Execution shift & $-0.0809$ & $[-0.1654,+0.0048]$ & $[-0.1901,+0.0280]$\\
Combined shift & $-0.0231$ & $[-0.1128,+0.0795]$ & $[-0.1341,+0.1099]$\\
\bottomrule
\end{tabular}
\end{table}

The loss-budget control is exactly equal to the attributed variant in all 24 independently executed main pairs: predictions, requested shares, update counts and final weights have maximum absolute difference zero. This confirms the equivalence predicted in the main manuscript, not an attribution-specific learning gain. Routing nevertheless genuinely changes computation relative to uniform learning: uniform spends 3,584 gradients on each task per run, whereas attributed allocation averages 4,239.7 direction/2,928.3 execution gradients in the stationary condition and 2,950.7/4,217.3 under combined shifts. The total remains 7,168. Shared-network negative-gradient fractions are approximately 51--52\%; their presence alone does not predict a PCGrad performance gain or prove task deterioration.

\begin{figure}[tb]
\centering
\includegraphics[width=\linewidth]{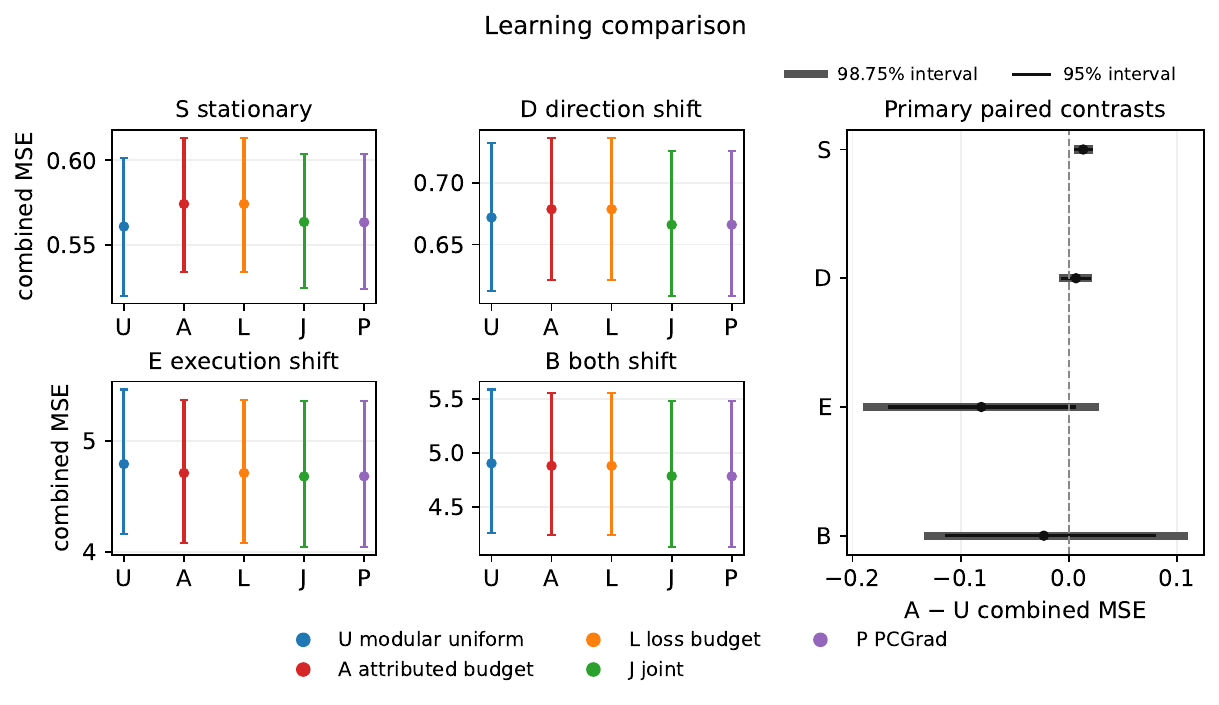}
\caption{Actually trained small learners on simulated tasks. Five-method losses have descriptive 95\% intervals; the paired allocation contrast has nested 95\% (thin black) and 98.75\% (thick gray) intervals. The attributed and ordinary loss-budget results coincide exactly. S/D/E/B identify stationary, direction-shift, execution-shift and combined-shift conditions. This tests the disclosed update scheduler, not full AMATS or market alpha.}
\label{fig:learning}
\end{figure}

\subsection{Component noninferiority and controlled transitions}
None of the three shifted conditions meets the prespecified controlled H3 criterion. Under a direction shift, the direction-loss upper 99.375\% bound is $+0.0272$, so intended-task improvement is not established. Under an execution shift, the corresponding execution bound is $+0.0975$. Under combined shifts, direction loss increases by $0.0516$ with interval $[0.0185,0.0915]$, while the execution improvement remains uncertain. The stationary condition meets the declared 0.05 noninferiority margin on both tasks, but that tolerance result must not be called improvement: the primary combined loss is worse.

Controlled transition diagnostics also limit an adaptation claim. Of 24 method/seed-boundary observations per condition, attributed versus uniform recovery is censored at 256 clocks in 1 versus 0 direction-shift events, 22 versus 24 execution-shift events, and 23 versus 23 combined-shift events. The 229 censored records across all methods are retained as non-recoveries, not relabeled as recovery at the deadline. These are responses to known simulator interventions, not evidence for the market-regime H2 hypothesis.

\subsection{Allocation sensitivities}\label{sec:protocol}

The executed learning grid changes one active factor at a time in the combined-shift condition, for the ordinary and attributed-budget learners and all six seeds. It covers regime count $K\in\{1,2,3\}$, forecast horizon $h\in\{1,4,8\}$, minimum regime hold $\{0,16,64\}$ clocks, boundary margin $\{0,0.05,0.15\}$ prefix standard deviations, residual decay $\{0.9,0.97,0.995\}$ and minimum task share $\{0.1,0.25,0.5\}$. The six baseline points reuse the same twelve main-study runs; they are not six independent replications. The other twelve settings require 144 new runs, all completed without selecting a favorable subset. Intervals in this grid are exploratory 95\%, not additional confirmatory hypothesis tests.

At $K=1$, categorical regime inputs are zero and routing uses a global residual state; legitimate raw features remain. A zero minimum hold retains four-clock candidate persistence and is therefore not complete hysteresis removal. A 0.5 minimum share forces uniform task allocation and supplies a matched routing-removal check. Changing the horizon changes the target definition and its training normalization; normalized-loss contrasts within a horizon are interpretable, whereas absolute cross-horizon loss differences are not direct bps savings.

\begin{figure}[tb]
\centering
\includegraphics[width=\linewidth]{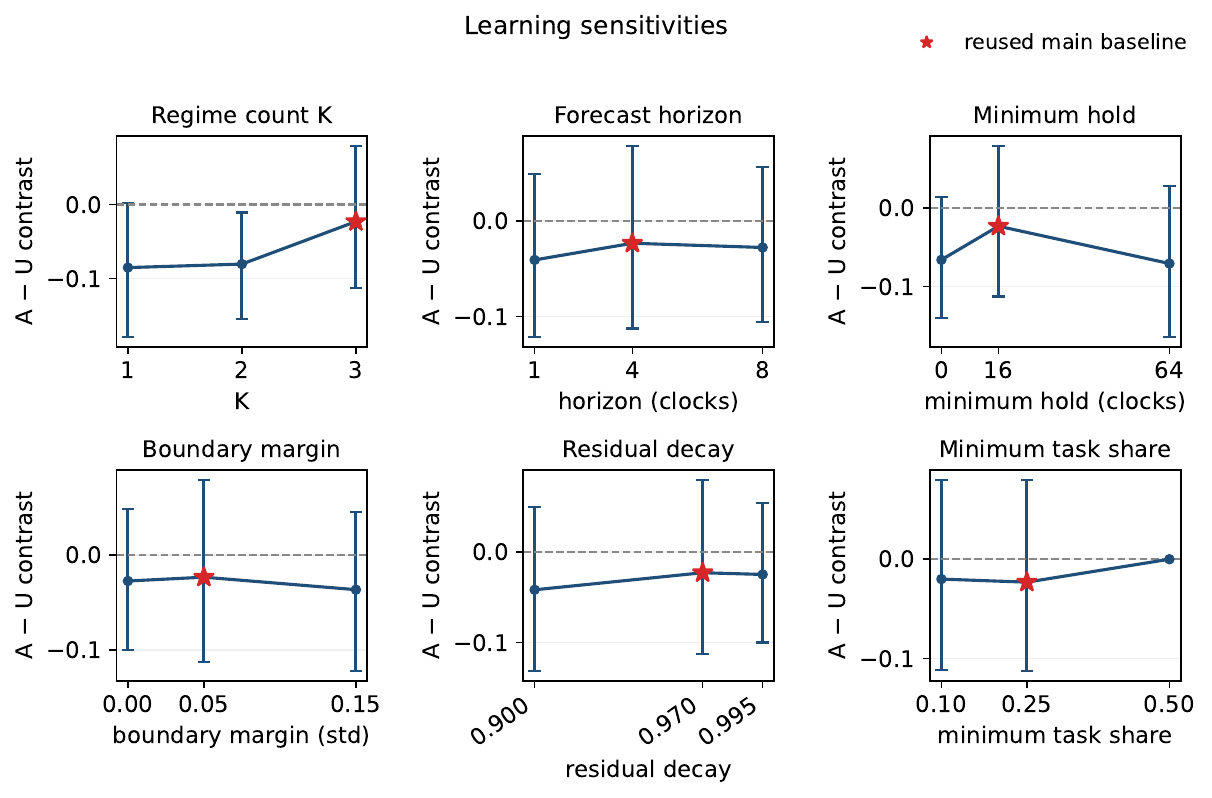}
\caption{Complete one-factor sensitivity grid for combined shifts: attributed minus uniform standardized loss with exploratory 95\% intervals. Each panel reuses the same main-study baseline point. At minimum share 0.5 the policies coincide. The grid does not select a deployable setting or establish component contributions in full AMATS.}
\label{fig:sensitivity}
\end{figure}

At $K=1$ and $K=2$, the contrasts are $-0.0854$ ($[-0.1794,0.0019]$) and $-0.0806$ ($[-0.1553,-0.0109]$), respectively. With 20,000 resamples, only the $K=2$ exploratory interval excludes zero: the earlier 1,000-resample $K=1$ interval narrowly excluded it, a conclusion not stable to numerical precision. Neither result retroactively changes the frozen primary $K=3$ configuration. The smaller regime partition does not demonstrate that more regimes help. All other nondegenerate grid contrasts have intervals crossing zero. For minimum share 0.5, the observed difference and both interval endpoints are exactly zero, verifying the uniform-allocation replacement. No favorable cell is substituted for the main result.

Only active components can support ablation claims. The small-network study evaluates allocation, categorical regime conditioning and the disclosed adaptation settings. The historical extension additionally tests common fee/slippage multipliers. It does not execute a six-head subset search, a named-forecaster stacking ladder, an LLM removal, a SAC reward/attribution-bonus grid, or a TCN removal. Those full-stack comparisons remain unexecuted; the absence of their components cannot count as a successful removal experiment.

\section{Representation Study: Generator and Supplementary Evidence}\label{app:representation}

\paragraph{Controlled observable event process.}
Six observable covariates follow a stationary AR(1) process with coefficient $0.8$ and standard-normal innovations scaled by $\sqrt{1-0.8^2}$. Starting at $m_0=100$, log-price increments in bps are $\beta_t^\top X_t+3\epsilon_t$. The base vector is $(2,-1,0.5,0,0,0)$; directional-shift blocks use $(-1.5,1,-0.5,0.75,0,0)$. Changed and base blocks alternate every 512 clocks from clock 2048. The four conditions are stationary, directional shift, execution shift and both shifts, with common random numbers. Intervention identifiers, coefficients and noise realizations are analyst-only, not learner inputs.

At origin $i$, $s=\operatorname{sign}(X_{i0}+0.25X_{i1})$ (positive at equality), size $a=0.5+0.25\tanh|X_{i2}|$, and $N=10{,}000a$. Submission occurs independently with probability $0.8$. Each submitted episode starts flat with quote cash $100{,}000$, order $\Delta q=sN/m_i$, and two possible half-quantity fills at $i+1,i+2$. Their fill logit is $1.5-0.3|X_{i2}|-0.2|X_{i3}|+0.1sX_{i0}$; execution shifts add $-1.25$. Slice slippage is
\begin{equation}
k_{ik}=1.5+0.6sX_{i0}+0.5a+0.6|X_{i4}|+0.25\eta_{ik}
       +I_i^E\{3+1.5\tanh X_{i3}\}\quad\text{bps},
\end{equation}
with independent standard-normal $\eta_{ik}$ and hidden execution-state indicator $I_i^E$. Fill price is $m_i(1+sk_{ik}/10^4)$ and fees are 4 bps of filled notional. Negative slippage is legitimate price improvement. Unfilled remainder is cancelled at $i+2$; fill/cancel reports share an origin lag drawn uniformly from $\{0,2,4\}$.

This quote-relative execution proxy is anchored to the frozen \emph{decision} benchmark, not the future prevailing mid. It isolates costs/fills but does not simulate market-consistent taker execution, impact, queue position or an exchange order book. A pure directional intervention preserves $c$ within tolerance, but need not preserve $u$ or $is$.

\paragraph{Maturity, preprocessing and actual updates.}
Direction matures at $i+h$; the complete execution record matures at $\max(i+h,i+2+\mathrm{lag}_i)$, including cancellations and zero fills. Every variant receives the same underlying observations at these clocks. Unsubmitted origins have no execution label, not a zero outcome. A horizon straddling a directional boundary can already contain changed increments; this does not permit training before maturity and limits pre-boundary diagnostic interpretation.

Input means/population standard deviations use clocks $[0,512)$, with standard-deviation floor $10^{-8}$. Separate $r,y,is$ transforms use only labels mature strictly before 512, floor $0.1$ bps and at least 16 observations. Each task receives 11 inputs: six standardized covariates, three causal regime slots, and side/size (the last two are zero for direction). The unchanged v2 regime filter uses a $0.95$ EWMA of $|X_4|$, prefix quantiles, $K=3$, hold 16, persistence 4 and boundary margin $0.05$ prefix standard deviations. No future fill or hidden intervention enters these features.

The unchanged learner uses two $11\!\to\!8\!\to\!1$ tanh networks (210 parameters) or a shared $11\!\to\!15\!\to\!2$ network (212). Adam has learning rate $0.003$, $(\beta_1,\beta_2)=(0.9,0.999)$, $\epsilon=10^{-8}$ and gradient clipping at 5. FIFO capacity is 256 per task, batch size 16 and minimum buffer 16. Up to 256 mature prefix records are preloaded without gradients or retrospective residual updates. At each clock 512--4095, mature records arrive in origin/task order, residuals and replay are updated, two task-gradient slots are trained, and the current prediction is stored. U assigns one slot per task; R uses the unchanged fractional-credit allocator, native standardized absolute residuals capped at 10, EWMA decay $0.97$, initial residual 1 and minimum share $0.25$. Shared controls combine two gradients into one Adam step. Scoring remains prequential, with ongoing adaptation, not a frozen-model holdout.

\subsection{Frozen generator streams and recovery definition}

V3 was frozen locally at 2026-09-09 13:00:05 UTC, after v2 but before any v3 comparative outcome. The protocol SHA-256 is
\begin{center}\footnotesize\texttt{07822660a9ee5698094d815b5582ae6d357}\par
\texttt{11298387ce30c095d477b1c13cb85}\end{center}
The two displayed lines concatenate to one digest. The complete specification and coefficients are in \path{research/study_v3/protocol.json} and \path{CONTRACT.md} in that directory. The fixed seeds are 101, 103, 107, 109, 113, 127, 131, 137, 139, 149, 151, 157, 163, 167, 173, 179, 181, 191, 193 and 197. Six independent generator streams separate covariates, direction noise, submissions, fills, slippage and report lags; full arrays are drawn before condition transformations. Learner seed is path seed plus 100,000, with separate initialization and task-replay streams. These are 20 independent path/initialization pairs, not a crossed estimate of each randomness source.

The 800 main cells comprise 20 seeds, four conditions and ten variants. Six fixed one-factor sensitivity settings each add 80 both-shift factorial cells: $K=1$, complete hysteresis removal (hold 0, persistence 1, margin 0), horizons 2 and 8, and residual decays $0.9$ and $0.995$. The 80 corresponding main factorial cells are reused as the baseline, not retrained or counted as new replication. The matrix is enumerated before execution; failed/pending cells cannot be silently dropped or replaced.

At boundaries 2048, 2560, 3072 and 3584, adaptation diagnostics retain pre/post 64-clock means. Recovery is the endpoint of the first fully post-boundary 32-clock window whose valid-observation mean is at most $1.1$ times the pre-mean, searched through 256 clocks. Non-recovery remains censored, not recovery at 256. Signed and positive-part integrated excess losses over the next 256 clocks retain observation counts. Summaries average boundaries within seed before averaging seeds; these diagnostics introduce no additional confirmatory tests.

\subsection{All variants and diagnostic outcomes}

\begin{table}[tb]
\centering
\caption{V3 equal-seed mean common loss $L$ (20 seeds per cell). J/P are descriptive shared-network controls. RAW-IS rows are independently trained equality controls, not additional independent evidence.}
\label{tab:v3means}
\small
\begin{tabular}{lrrrr}
\toprule
Variant & Stationary & Direction shift & Execution shift & Both shifts\\
\midrule
RAW-U & 1.2731 & 1.8002 & 1.6989 & 2.2879\\
RAW-R & 1.2654 & 1.7850 & 1.7014 & 2.2837\\
ATTR-U & 0.9481 & 1.2198 & 1.4291 & 1.8750\\
ATTR-R & 0.9495 & 1.2188 & 1.4346 & 1.8883\\
RAW-J & 1.2627 & 1.7487 & 1.6894 & 2.2432\\
ATTR-J & 0.9533 & 1.2125 & 1.4331 & 1.8821\\
RAW-P & 1.2621 & 1.7475 & 1.6893 & 2.2428\\
ATTR-P & 0.9534 & 1.2124 & 1.4332 & 1.8816\\
RAW-IS-U & 0.9481 & 1.2198 & 1.4291 & 1.8750\\
RAW-IS-R & 0.9495 & 1.2188 & 1.4346 & 1.8883\\
\bottomrule
\end{tabular}
\end{table}

\paragraph{Localization, cross-loading and adaptation.}
Both residual-dominance tests pass: direction $0.1440$ (97.5\% interval $[0.0690,0.2280]$), execution $0.6199$ ($[0.5429,0.6995]$). These are changes in clipped standardized absolute error, not probabilities of correct causal diagnosis. In the same first-onset window, the pure direction intervention changes mean benchmark opportunity by $-4.7942$ bps and signed non-fill opportunity by $-1.1216$ bps, while $c$ is unchanged within $1.3\times10^{-12}$ bps per origin; net $y$ changes by $-3.6727$ bps. Thus IS responds to direction through unfilled opportunity.

For the execution intervention, benchmark change is zero, $\Delta c=-0.3533$ bps, $\Delta u=+0.7412$ bps and $\Delta y=-0.3879$ bps. Fewer fills can reduce cost per approved notional even when per-slice slippage increases. These windows reconcile per origin within $9\times10^{-15}$ bps. Descriptively, ATTR-R's mean desired direction share rises from the paired stationary 0.5312 to 0.5501 under direction shifts and falls to 0.4753 under execution shifts; that response does not prove routing improves prediction.

\begin{figure}[tb]
\centering
\includegraphics[width=\linewidth]{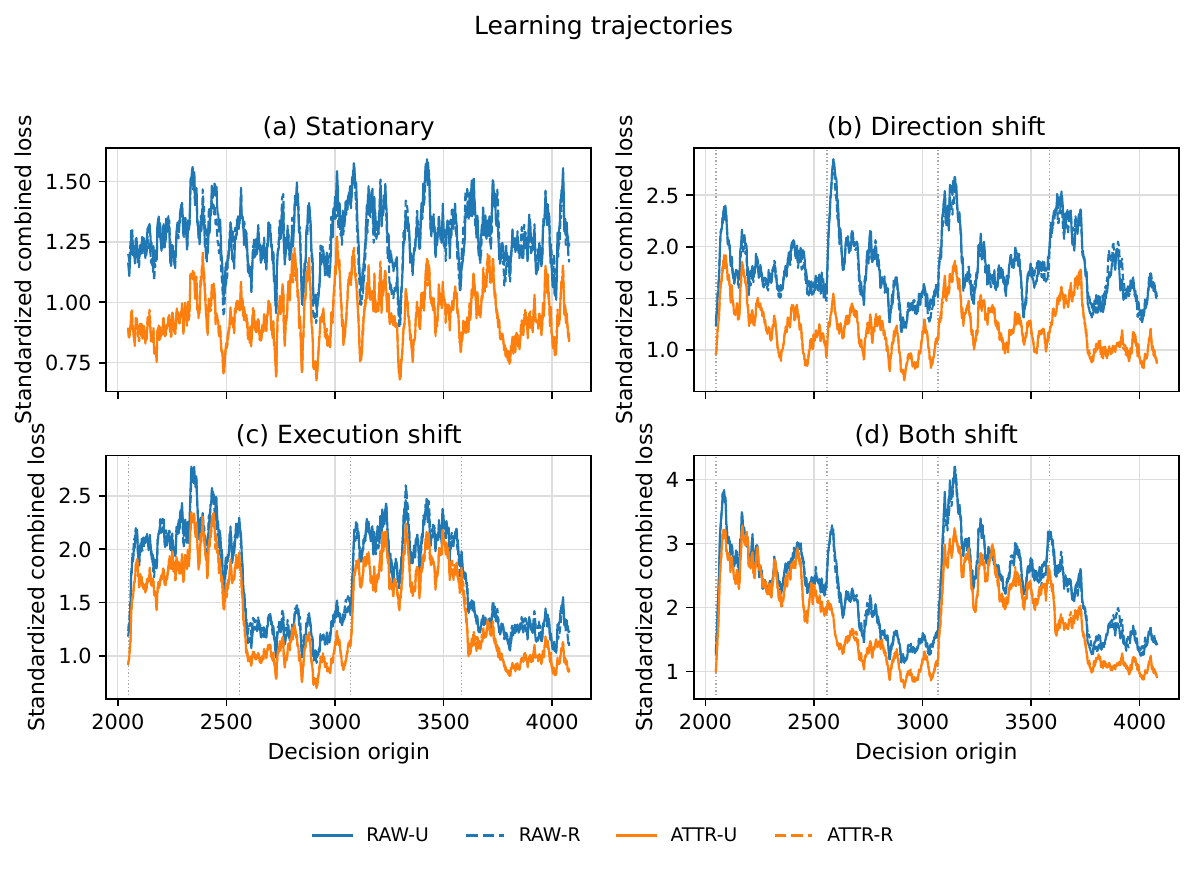}
\caption{V3 factorial trajectories: equal-seed means of trailing 32-clock valid-observation common loss. These are descriptive curves, not confidence bands or additional observations. Vertical lines mark frozen alternating intervention boundaries in shifting conditions. All four conditions and modular treatments are retained.}
\label{fig:v3trajectories}
\end{figure}

The adaptation record retains 667 censored cases among 9,600 method--seed--boundary--component records, including equality controls; these are not 9,600 independent experiments. For both-shift $L$, RAW-U, RAW-R, ATTR-U and ATTR-R have 5, 6, 11 and 12 non-recovered boundaries out of 80. Their seed-first successful-recovery means are 70.52, 68.43, 71.28 and 74.70 clocks. Each seed has at least one recovery, but these means exclude censored boundaries and cannot replace censor fractions. Seed-first signed integrated excess losses are 106.17, 109.21, 96.08 and 101.90 standardized loss-clocks. These are relative to each method's own pre-window level; lower absolute prediction loss need not imply faster relative-threshold recovery.

Every cell consumes 7,168 task-gradient evaluations and 114,688 replay samples. Modular models have 210 parameters and 7,168 Adam steps; shared models have 212 parameters and 3,584 steps. Gradient/replay budgets, not every compute dimension, are matched. Summed training is 827.81 s, unique-path preparation 59.83 s and runner wall time 1,722.64 s including serialization/hashing. All fitted weights are finite and changed. These are host measurements, not live latency benchmarks.

Relative to 20k resamples, the maximum 1k endpoint movement is 0.00493 for primary $L$ intervals and 0.05758 for guardrails; at 10k these maxima are 0.00572 and 0.00441. Convergence is not monotonic. One 1k guardrail noninferiority flag changes (both-shift routing $is$ MSE), but no primary zero-exclusion classification or complete V3-H1--H3 decision changes. More resamples improve numerical resolution, not independent replication.

\subsection{V3 complete mechanism sensitivity grid}
The six nonbaseline settings add 480 runs without selecting a favorable horizon or regime partition (Figure~\ref{fig:v3sensitivities}). All seven target contrasts have negative descriptive 95\% intervals on $L$. Removing regime conditioning changes the target contrast from $-0.4130$ to $-0.3411$; the paired baseline change is $+0.0719$ ($[0.0367,0.1059]$). Complete hysteresis removal gives $-0.3923$, with baseline-change interval crossing zero. Horizons 2 and 8 give $-0.2497$ and $-0.5580$; their different targets/scales preclude interpreting this ordering as bps savings or an optimal horizon.

\begin{figure}[tb]
\centering
\includegraphics[width=\linewidth]{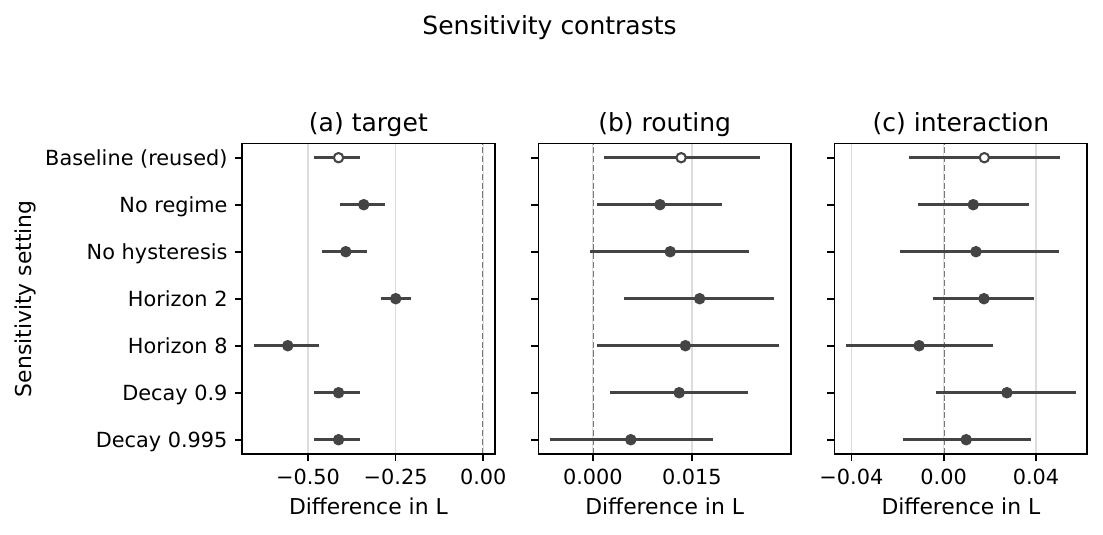}
\caption{Complete v3 both-shift target, routing and interaction sensitivity contrasts on $L$, with descriptive paired 95\% intervals. The baseline reuses 80 main cells; each other setting has 20 seeds and four variants. No sensitivity replaces the frozen primary result.}
\label{fig:v3sensitivities}
\end{figure}

Residual decay leaves the uniform target contrast exactly unchanged. All routing point estimates are positive, and none meets the improvement criterion. Baseline, no-regime, horizon-2, horizon-8 and decay-0.9 routing intervals exclude zero adversely at the descriptive 95\% level. All seven interaction intervals cross zero, leaving the paired effects of decay changes and hysteresis removal inconclusive relative to baseline. This sensitivity inventory covers the executed learner configurations; LLM, SAC and TCN ablations belong to the full-system evaluation agenda.

\section{Portfolio Control: Complete Contrasts and Diagnostics}\label{app:control}

V4 was frozen locally at 2026-09-09 20:43:06 UTC before implementation and comparative execution. The protocol SHA-256 is
\begin{center}\footnotesize\texttt{a64f68ce3601e99ce05f4570f9b58587e0cf15a9c4af8c8e2a4796c2ca9fd303}\end{center}

The complete frozen specification supplies the generator coefficients, ordered state, initialization, seed list and reservation inequalities. Its timestamp is local prospective provenance, not external preregistration.

All 2,940 intended cells completed with all 30 prespecified seeds and no replacement, missing or failed cell. The matrix took 5,350.38 seconds using three worker processes and produced a 3.22-GiB artifact tree. Analysis verifies every cell manifest and preparation/checkpoint link, recomputes CE, net return, MDD and utility components, and independently replays signed fills, positions and fees to reconstruct final equity using saved final marks; it does not regenerate market marks independently. The maximum marked-cashflow identity residual is $1.59\times10^{-9}$ quote units and the maximum saved-versus-reconstructed final-equity difference is zero. Every one of the 2,520 learned cells reports changed final parameter hashes and exactly 8,192 actor and critic updates; no risk-controller replay samples are used.

\begin{table}[tb]
\centering
\caption{V4 seed-mean certainty equivalent in bps per clock. Each cell equally weights 30 complete test paths. Intervals and hypothesis decisions use paired seed differences, not comparisons of these marginal means.}
\label{tab:v4means}
\scriptsize
\resizebox{\linewidth}{!}{%
\begin{tabular}{lrrrrrr}
\toprule
Condition & FIXED & VOL & RL-SCALAR & RL-RISK & PCGrad & JOINT-DRE\\
\midrule
Stationary & 0.1939 & 0.1650 & -0.2700 & -0.2916 & -0.2926 & -0.2908\\
Direction shift & -0.5219 & -0.5162 & -0.5089 & -0.6515 & -0.6284 & -0.8476\\
Execution shift & 0.1056 & 0.0877 & -0.2204 & -0.2309 & -0.2323 & -0.2138\\
Risk shift & 0.0911 & 0.0683 & -0.3311 & -0.3639 & -0.3630 & -0.4488\\
Direction+risk & -0.6315 & -0.1717 & -0.6241 & -0.7547 & -0.7388 & -0.8161\\
Execution+risk & 0.0248 & 0.0513 & -0.2507 & -0.2707 & -0.2654 & -0.3085\\
Full shift & -0.3820 & -0.1388 & -0.4647 & -0.5288 & -0.5161 & -0.5842\\
\bottomrule
\end{tabular}}
\end{table}

\begin{table}[tb]
\centering
\caption{V4 confirmatory paired contrasts with 99.1667\% Bonferroni intervals. CE is bps/clock; MDD and net return are fractions. Positive CE/return and negative MDD favor RL-RISK.}
\label{tab:v4contrasts}
\small
\begin{tabular}{lrrl}
\toprule
Contrast (RL-RISK minus comparator) & Difference & Adjusted interval & Decision\\
\midrule
All-condition CE minus FIXED & -0.2817 & $[-0.3186,-0.2285]$ & Not supported\\
Risk-condition CE minus RL-SCALAR & -0.0619 & $[-0.1132,-0.0241]$ & Not supported\\
Stationary CE minus RL-SCALAR & -0.0215 & $[-0.0545,0.0013]$ & NI not established\\
Risk-condition MDD minus RL-SCALAR & 0.0151 & $[0.0053,0.0286]$ & Not supported\\
Risk-condition net return minus RL-SCALAR & -0.0145 & $[-0.0280,-0.0048]$ & NI not established\\
All-condition CE minus JOINT-DRE & 0.0597 & $[0.0187,0.1033]$ & Supported\\
\bottomrule
\end{tabular}
\end{table}

\subsection{V4 risk-state and projection ablations}
The four v4 ablations retain the same architecture and either zero the regime, drawdown or execution-cost state coordinates, or replace the primary allocation projection by the prespecified outer solvency envelope. Averaged descriptively across the seven condition-specific estimates, RL-RISK CE changes by $+0.0486$, $+0.0174$, $+0.0324$ and $-0.0063$ bps/clock, respectively. PCGrad minus ordinary RL-RISK is $+0.0079$ bps/clock by the same descriptive averaging. These are not additional confirmatory contrasts, and their signs do not select a revised policy.

At the within-test risk onset, exposure-recovery censor counts out of 120 records are 47, 0, 14, 11, 16 and 29 for FIXED, VOL, RL-SCALAR, RL-RISK, PCGrad and JOINT-DRE. Drawdown-stabilization censor counts are 10, 15, 21, 27, 27 and 9. Successful-only mean times are not promoted to rankings because censoring differs and these method--condition records are not independent seed replications. Across the full matrix there are no equity stops, while 2,521 emergency responses and 3,942 primary gross/single breach flags are retained. These simulator interventions are costs and audit records, not evidence of live liquidation performance.

The within-test recovery diagnostic uses the second risk onset (clock 10,752), because the first onset coincides with the flat test reset. Breach flags can include both gross and single-asset flags at one mark; they are not distinct independent events. The 2,602 flags against actively enforced limits are separate from the 3,942 flags against primary limits, including the outer-envelope ablation.

\section{Secondary Historical Bar Replay}\label{app:bar}\label{app:historical}

\subsection{Historical inputs and bar-study contract}\label{sec:marketdata}
The local BTCUSDT and ETHUSDT files each contain 3,329,280 one-minute observations from 2020-01-01 00:00 through 2026-04-30 23:59 UTC, with no missing or duplicate minute slots. OHLCV, timestamp and forward-filled funding fields are present; order-book events and actual venue mark prices are not. The acquisition utility targets Binance USD-margined futures archives, but the retained final files do not carry an upstream download receipt or checksum attestation. Full local SHA-256 values, schemas and audited coverage are in \path{research/study_v2/DATA_MANIFEST.json}. Redistribution permission is not established, so raw market files are excluded from the public package.

Raw funding files are separate event records with \texttt{calc\_time}, \texttt{funding\_interval\_hours} and \texttt{last\_funding\_rate}. The audited January 2024--April 2026 subset contains 2,553 events per symbol, with explicit eight-hour intervals and no duplicate event times. Of these, 854 per symbol have positive millisecond offsets of 1--16 ms from their nominal boundary. An event is charged once, preserving its timestamp. The minute-level forward-filled columns must not be charged every minute. A bar replay must disclose a funding-price proxy because actual funding mark prices are absent.

The frozen bar extension uses hourly aggregation, separate BTC/ETH accounts, January--March training, April diagnostic validation and May--June tests in 2024 and 2025. The author and earlier archive work have historical exposure; these are not claimed to be untouched holdouts. Three training seeds and the five small-learner methods are evaluated, with a passive 25\%-notional comparator and paired 0.5/2.0 cost sensitivities. Test updates continue from already mature past observations: this is a prequential policy experiment, not a frozen-weight holdout. Its single forecast-to-position rule, fixed 25\% sizing and delayed full taker-fill proxy are a \emph{partial public configuration}, not implemented six-head intent, SAC risk control or L2 execution. Exact availability, normalization, order timing, fees, funding, terminal inventory and calendar-time marking rules are recorded in \path{research/study_v2/HISTORICAL_CONTRACT.md}; completed results appear in Section~\ref{sec:barresults}.

\paragraph{Calendar returns and uncertainty.}
Every test account has 1,464 hourly marks over 61 days, including flat and open-position periods. With $V_0=100{,}000$, hourly return is $u_t=V_t/V_{t-1}-1$. Maximum drawdown includes the initial capital in the running peak. Descriptive Sharpe is $\sqrt{8760}\,\bar u/s_u$ using sample standard deviation; a constant series has undefined Sharpe. Hourly certainty equivalent is $\bar u-(5/2)\widehat{\Var}(u)$ using population variance. Positive-day rate counts strictly positive geometrically compounded daily returns over all 61 days, including zero days in the denominator; it is not a closed-trade win rate. Two-way turnover is total absolute executed notional divided by initial capital. Temporal IC/Rank-IC and bps MAE use mature forecast targets within each asset/fold, excluding the last eight decision origins while retaining all equity marks.

The exploratory market contrast equally averages the two asset-account daily returns, then compares attributed with uniform on paired dates and training seeds. A two-level bootstrap resamples three seeds and circular seven-day blocks within each year, with common draws across methods/assets, 20,000 replicates and seed 90218. The pooled estimate equally weights both years; yearly effects remain visible. Its 95\% intervals are conditional on two selected historical windows, not uncertainty over all market eras. Grouped Sharpe and drawdown are averages of separately funded account metrics, not those of a rebalanced combined portfolio. Own-policy execution labels have different parent targets across methods, so their conditional cost differences are not reported as matched execution bps saved. Full fills are imposed by the bar model, not a measured fill-rate improvement.

\subsection{Features, labels and fixed policy}

In Bar, writing completed hourly close, high, low and base volume as $C_t,H_t,L_t,W_t$, raw feature order is
\begin{equation}
\left(
10^4\log\frac{C_t}{C_{t-1}},\;
10^4\log\frac{C_t}{C_{t-4}},\;
10^4\log\frac{C_t}{C_{t-16}},\;
\operatorname{sd}_{0}(r^{(1)}_{t-15:t}),\;
\log(1+W_t),\;
10^4\frac{H_t-L_t}{C_t}
\right).
\label{eq:barfeatures}
\end{equation}
Here $r^{(1)}$ is the one-hour log return in bps and $\operatorname{sd}_0$ is population standard deviation. All coordinates except the log-volume transform are in bps before standardization; volume is numerically expressed in base units before applying $\log(1+\cdot)$. The first 16 hours provide lookback, not zero-valued observations. The regime statistic is an EWMA of the fourth coordinate, initialized at the first valid clock.

The Bar direction target is $10^4(C_{t+h}/C_t-1)$, a simple return in bps, delivered when that future bar closes. For a submitted signed order change $\Delta q$, execution coordinates are $(\operatorname{sign}\Delta q,|\Delta q|C_t/V_t)$, where $V_t$ is current completed-bar equity. A candidate action is evaluated for the gate; if an order is submitted, its actual action features and prediction replace that candidate record. With full fill price $P$ and signed fee $F$, the label is
\begin{equation}
y_{t,E}=10^4\,\frac{\Delta q(P-C_t)+F}{|\Delta q|C_t}.
\label{eq:costtarget}
\end{equation}
It is signed shortfall including fees, not an optimal schedule or a counterfactual unsubmitted cost. It becomes available at the simulated fill and enters that bar's close-time learning cycle. A candidate prediction without a realized eligible fill is not a supervised label. Observation, maturity and finite-value masks are retained separately; scored direction rows use their own masks, and the combined simulated endpoint uses the intersection of the two task masks.

\subsection{Fixed historical policy and event timing}\label{sec:barpolicy}
At a completed Bar clock, let $q_t$ be held base quantity and $d_t$ the forecast in bps. The candidate position is $\operatorname{sign}(d_t)\phi V_t/C_t$, with $\phi$ from Table~\ref{tab:parameters}. Predict candidate cost from its change relative to $q_t$. Accept the candidate if $|d_t|>\max(\hat c_t,0)+g$; otherwise the approved target is zero, so existing inventory may be closed. Recompute execution features/prediction for the actual approved change before submission. There is no LLM policy or learned risk sizing. Only one order can be outstanding; while it is pending, no replacement or second order is sent. Sub-tolerance changes and the final excluded decision origins are skipped.

An order decided after bar $i$ closes fills completely at bar $i+2$ open plus one second, at that open multiplied by $1+\operatorname{sign}(\Delta q)10^{-4}\ell$. Fee is $|\Delta q|P\,10^{-4}f$; $\ell,f$ are the tabulated slippage and fee bps. The future open is read at fill processing, not stored as a known decision-time price. The cost-sensitivity multiplier applies once to both $\ell$ and $f$, including prefix calibration, and not to funding. Prefix execution probes use one base unit, alternate buy/sell on every second valid prefix clock, and encode size as $C_i/100{,}000$. Only mature probe labels initialize replay/normalization; probe profits never enter the account.

Funding debit at each raw event time is current signed quantity times the current hourly open proxy times the raw rate. Positive debit reduces cash and negative debit credits cash. The observed boundary offsets are 0--16 ms, preceding the one-second fill offset; the runner's funding-before-fill order is valid for these files, not a general intrabar event engine. Forward-filled minute funding fields are not charged. Completed-bar valuation is one millisecond before the next boundary, available at that boundary; its funding belongs to the next hour. Equality of economic and availability times for simulated fills/funding is an explicit simulator assumption, not measured exchange-to-local latency.

At the May test boundary, each account resets flat to the stated capital and cancels any outstanding pretest order. Model, optimizer, replay, residual/regime history and legitimate pending direction labels persist; training-period cash, inventory and profits do not. All 1,464 test closes are marked, including the last eight hours with no new submissions. Terminal inventory is valued, not forcibly liquidated with an invented fee. Passive accounts reset identically and make one delayed 25\%-notional buy after the first test-hour close. None of these rules models margin tiers, liquidation, queues, partial venue fills or L2 execution.

\subsection{Historical partial-configuration comparison}\label{sec:barresults}
The fixed main matrix contains 60 learned runs and four passive accounts, and the cost sensitivity contains 48 additional learned runs. All cells completed. Each method is evaluated on the same two asset/year windows and the same three initialization seeds. The full account-level results, including adverse seeds, remain in the numerical release. The attributed and ordinary loss-budget variants coincide exactly in all twelve main pairs, including predictions, positions, equity, routing counts and final weights.

\begin{table}[tb]
\centering
\caption{Historical public bar configuration, May--June in each year. Values are averages of separately funded account metrics: first across three seeds per asset, then equally across BTC/ETH; passive has two accounts, not six artificial seed replications. DD is hourly-sampled maximum drawdown; SR is descriptive annualized hourly Sharpe; median SR is the account median; CE is hourly certainty equivalent in bps; Day+ is the percentage of positive calendar days. Uncertainty for the central paired return contrast appears in Table~\ref{tab:marketcontrast}.}
\label{tab:market}
\small
\setlength{\tabcolsep}{3.5pt}
\begin{tabular}{llrrrrrr}
\toprule
Year & Method & Net (\%) & DD (\%) & SR & Median SR & CE & Day+ (\%)\\
\midrule
2024 & Uniform & -7.17 & 8.80 & -4.65 & -4.50 & -0.531 & 38.5\\
2024 & Attributed & -3.92 & 7.18 & -2.70 & -3.28 & -0.302 & 40.2\\
2024 & Loss & -3.92 & 7.18 & -2.70 & -3.28 & -0.302 & 40.2\\
2024 & Joint & -5.05 & 7.85 & -3.22 & -3.37 & -0.382 & 39.3\\
2024 & PCGrad & -4.17 & 6.98 & -2.63 & -2.69 & -0.321 & 37.2\\
2024 & Passive & 1.90 & 5.42 & 0.87 & 0.87 & 0.089 & 49.2\\
\midrule
2025 & Uniform & -4.31 & 7.35 & -1.84 & -1.80 & -0.338 & 44.3\\
2025 & Attributed & -3.01 & 6.67 & -1.48 & -0.98 & -0.243 & 40.7\\
2025 & Loss & -3.01 & 6.67 & -1.48 & -0.98 & -0.243 & 40.7\\
2025 & Joint & -3.88 & 6.83 & -1.90 & -2.34 & -0.306 & 42.9\\
2025 & PCGrad & -6.26 & 8.42 & -3.14 & -3.57 & -0.480 & 38.3\\
2025 & Passive & 6.08 & 5.92 & 2.23 & 2.23 & 0.336 & 51.6\\
\bottomrule
\end{tabular}
\end{table}

Every learned method has a negative mean account return in both years (Table~\ref{tab:market}); the passive comparator is positive on average. Attributed allocation reduces the loss relative to uniform. Its pooled mean daily contrast is $+3.90$ bps/day with descriptive 95\% interval $[-0.90,9.96]$, crossing zero. The 2024 interval excludes zero, while the 2025 interval does not. These selected-window comparisons concern the fixed-policy scheduler; attributed allocation and the equal loss-budget control remain identical, and full-system H1 requires its own matched evaluation.

\begin{table}[tb]
\centering
\caption{Historical attributed minus uniform mean daily return. Positive favors attributed allocation. Intervals resample paired initialization seeds and common seven-day blocks; they are descriptive and conditional on the selected windows. Units are bps/day, not execution bps saved.}
\label{tab:marketcontrast}
\begin{tabular}{lrr}
\toprule
Window & Difference & 95\% interval\\
\midrule
2024 & $+5.54$ & $[+1.43,+11.58]$\\
2025 & $+2.27$ & $[-5.27,+10.69]$\\
Equal-year pooled & $+3.90$ & $[-0.90,+9.96]$\\
\bottomrule
\end{tabular}
\end{table}

\begin{figure}[tb]
\centering
\includegraphics[width=\linewidth]{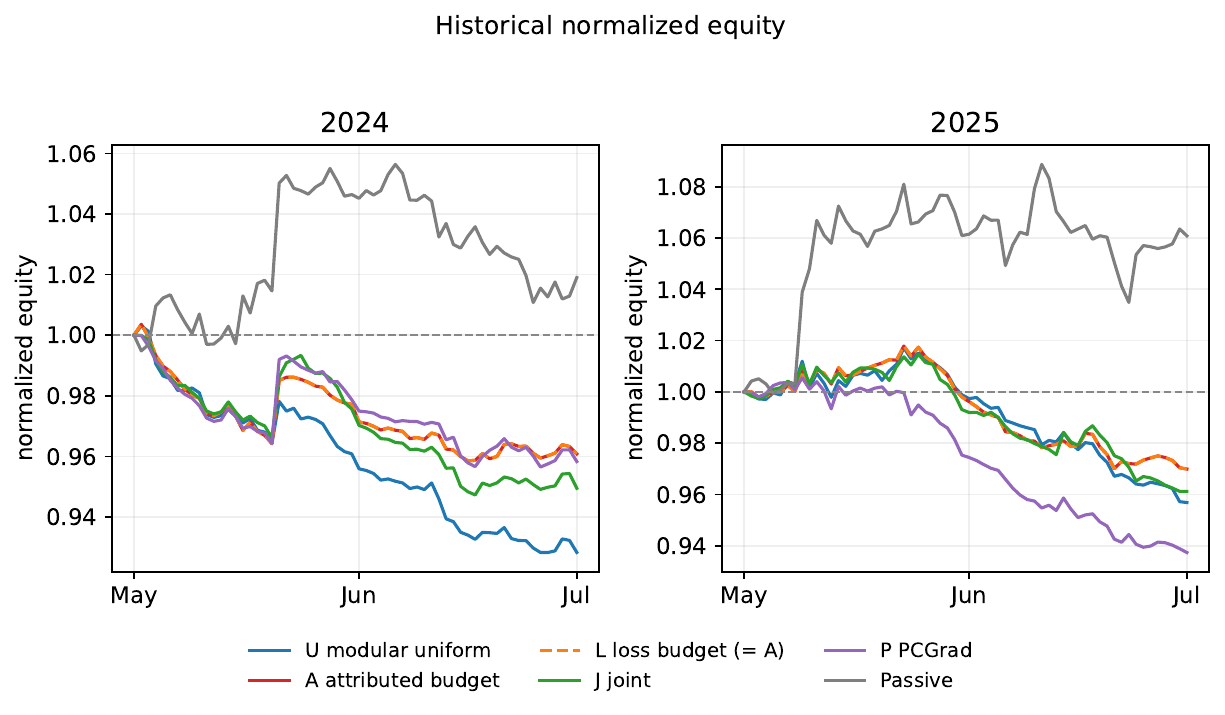}
\caption{Mean normalized equity of separately funded historical research accounts, retaining the full calendar window. Curves connect measured daily closes and the initial capital point; they are not a rebalanced portfolio or a confidence band. Attributed and loss-budget curves coincide exactly. Negative learned-policy outcomes are retained alongside the passive comparator.}
\label{fig:marketequity}
\end{figure}

\paragraph{Execution, risk and routing diagnostics.}
Mean two-way turnover is higher for attributed than uniform: 83.91 versus 77.23 initial-capital multiples in 2024, and 91.85 versus 86.06 in 2025. Corresponding mean fees are 3,356.52 versus 3,089.21 quote units, and 3,673.82 versus 3,442.44, respectively. Both executed notional and fee expenditure therefore increase. Mean signed funding debits are 25.72 versus 16.90 in 2024 and 4.79 versus 11.60 in 2025. Own-policy mean observed shortfall is 5.30 versus 7.41 bps and 5.32 versus 6.97 bps, but the methods choose different parent orders: those conditional differences are not matched execution savings. Mean directional IC is approximately $-0.004$ versus $-0.010$ in 2024 and $0.021$ versus $0.023$ in 2025; directional associations are small and mixed across years.

All 93,696 main test marks agree with independent signed-fill/fee/funding MTM recomputation to less than $6\times10^{-10}$ quote units. Every account retains 183 unique test funding events. No test account records gross exposure above one, multiple outstanding orders or late-tail submissions; this is an observed property of this simulation, not an exchange risk guarantee. Funding remains charged at portfolio level while its allocation to individual execution decisions is explicitly unresolved. Adaptive requested direction shares genuinely vary, spanning approximately 0.323--0.729 in the main matrix; uniform stays at 0.5. These logs establish actual routing activity, not its superiority.

\subsection{Common cost sensitivity}

The historical cost sensitivity reruns both primary learners at 0.5 and 2 times the common fee/slippage assumptions, leaving funding unchanged and reusing the 1-times main cells. Pooled attributed-minus-uniform differences are $-1.26$ bps/day ($[-4.80,2.23]$) at half cost, $+3.90$ ($[-0.90,9.96]$) at base cost and $+2.03$ ($[-1.57,5.61]$) at double cost. All three pooled descriptive intervals cross zero. The sign change at half cost and nonmonotonic point estimates do not identify a robust allocation gain or an optimal cost assumption. All yearly and individual-account outcomes remain available.

\begin{figure}[tb]
\centering
\includegraphics[width=0.85\linewidth]{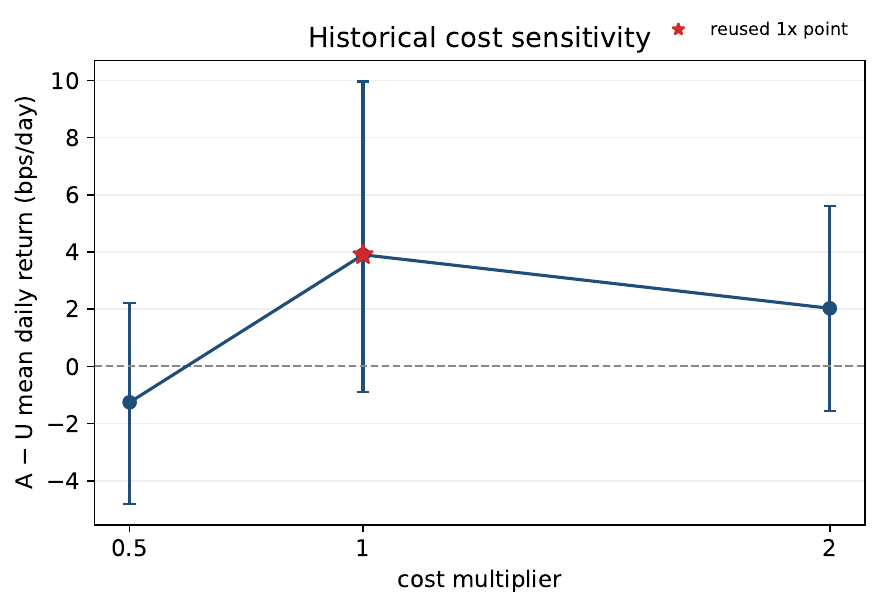}
\caption{Complete historical fee/slippage sensitivity: paired attributed-minus-uniform mean daily return with descriptive 95\% intervals. The base point is reused from the main matrix. Funding is unchanged; policy trajectories are retrained under each declared cost multiplier. Connecting lines do not estimate untested multipliers.}
\label{fig:marketcosts}
\end{figure}

\clearpage

\section{Original Full-System Hypotheses}\label{app:originalhypotheses}

Four questions define the eventual performance evaluation.
\begin{enumerate}[label=\textbf{RQ\arabic*.},leftmargin=*]
\item Does attribution improve out-of-sample performance against matched joint and modular non-attributed baselines? \textbf{H1:} higher median net Sharpe and certainty-equivalent return.
\item Does it reduce degradation near regime changes? \textbf{H2:} a smaller transition penalty and shorter recovery time.
\item Does a routed update improve its intended target without harming another? \textbf{H3:} improved directional IC/Rank-IC without greater shortfall, or reduced shortfall without worse directional accuracy.
\item What do regime conditioning, stacking, the LLM, SAC risk control, and execution optimization contribute? \textbf{H4:} incremental benefits under matched data, capacity, costs, and risk budgets.
\end{enumerate}
These are full-system market-performance hypotheses, not automatically confirmed by a mechanism test. This version adds actually trained small-learner comparisons against ordinary task-specific modular supervision, a directly comparable loss-based allocation control, joint learning and PCGrad. The empirical analysis distinguishes a normalized-loss analogue of H3 from the original market IC/shortfall claim. The earlier 2,048 paired accounting episodes remain a separate regression: equal imposed opportunity and execution shocks produce the same net series while changing different recorded components. No financial confidence interval is assigned to that deterministic paired identity.

\begin{table}[tb]
\centering
\caption{Original hypothesis status and the narrower evidence actually obtained. ``Untested'' refers to the original full-system claim, not to the absence of the reported public experiments.}
\label{tab:hypotheses}
\small
\setlength{\tabcolsep}{2.5pt}
\begin{tabular}{p{0.05\linewidth}p{0.19\linewidth}p{0.30\linewidth}p{0.17\linewidth}p{0.20\linewidth}}
\toprule
H & \raggedright Required test & \raggedright Observed evidence & \raggedright Conclusion & \raggedright Limitation\tabularnewline
\midrule
H1 & \raggedright Matched full-system median net Sharpe and CE & \raggedright Trained controlled losses and exploratory historical account comparisons; attributed and loss-budget controls coincide & \raggedright Full H1 untested; no attribution-specific gain established & \raggedright Small learners and two historical windows, not full AMATS\tabularnewline[0.4em]
H2 & \raggedright Market-transition penalty and recovery & \raggedright Controlled intervention diagnostics retain 229 censored records across methods & \raggedright Market H2 untested & \raggedright Simulator change times are not observed market transitions\tabularnewline[0.4em]
H3 & \raggedright Intended IC/shortfall gain and noninferiority elsewhere & \raggedright All three shifted normalized-loss criteria fail; stationary meets the margin but combined loss worsens & \raggedright Analogue not supported; full H3 untested & \raggedright Own-policy order targets differ; cost differences are not paired savings\tabularnewline[0.4em]
H4 & \raggedright Matched removal of each active architectural component & \raggedright Complete six-factor learning grid and cost multipliers; minimum share 0.5 reproduces uniform & \raggedright Partial tests completed; full H4 untested & \raggedright Named stack, six heads, LLM, SAC and TCN not evaluated\tabularnewline
\bottomrule
\end{tabular}
\end{table}

\section{Historical Source Context and Information Diagnostics}\label{app:legacycontext}

\subsection{Original estimates and illustrative comparisons}
The original seven-page draft is preserved as a source document, not overwritten by this revision. Its headline statements report Sharpe 2.5--4.0, win rate 70--75\%, maximum drawdown below 15\%, about 7\% during a one-month live test, about 107\% annualized CAGR, CAGR/MDD about 6.7, Sortino about 2 and approximately 50\% active functionality. These are author-reported historical estimates with unrecovered sample and configuration provenance; the live-test statement is not independently substantiated by the retained backtest artifacts. The author describes the draft values as grounded in earlier work, while the forecasting table itself was explicitly illustrative.

These statements are not substituted for E01's different, concretely reported values. The monthly/annualized pair is retained as source history rather than silently reconciled: exact elapsed duration, convention and underlying observations are unknown. E01 additionally reports CAGR/MDD 8.1440; that is another transcribed historical ratio, not validation of the original estimate near 6.7.

\begin{table}[tb]
\centering\small
\caption{Original draft Table 1, retained with its explicitly illustrative status. The horizon $h$, sample, costs and comparison protocol were not specified sufficiently to reproduce this table. These cells are not new trained comparisons and are not E01 metrics.}
\label{tab:legacyillustrative}
\begin{tabular}{@{}lrrrr@{}}
\toprule
Original model/stack label & AUC@$h$ & MASE & Sharpe & MDD\\
\midrule
N-HiTS & 0.61 & 0.89 & 1.9 & 0.18\\
PatchTST & 0.64 & 0.84 & 2.6 & 0.15\\
ETSformer & 0.63 & 0.85 & 2.4 & 0.16\\
Stacked & 0.66 & 0.81 & 3.1 & 0.13\\
\bottomrule
\end{tabular}
\end{table}

The draft qualitatively describes transformer strength at middle horizons, N-HiTS strength at coarser horizons, greater stability from stacking and a router adapting maker/taker behavior while saving bps. It also reports faster overfitting or less stable costs without attribution routing, greater variance/regime fragility without stacking, and worse drawdowns without RL risk. These source-reported observations are preserved as historical context; no corresponding matched model/removal experiment or execution-savings magnitude is recovered. They cannot supersede the actual v2--v4 outcomes or stand in for the original full-system hypotheses in Section~\ref{app:originalhypotheses}.

\subsection{Evidence status of the original checklist}
Table~\ref{tab:legacychecklist} replaces the source's blanket completion marks with the status of the materials actually available. This does not infer that absent historical records mean a test was never performed.

\begin{table}[tb]
\centering\small
\caption{Original validation and operating checklist, restated by evidence status. Design requirements are not counted as completed empirical checks.}
\label{tab:legacychecklist}
\begin{tabular}{@{}p{0.28\linewidth}p{0.64\linewidth}@{}}
\toprule
\raggedright Original item & \raggedright Status in this paper\tabularnewline
\midrule
\raggedright Walk-forward and regime-stratified evaluation & \raggedright Explicit controlled study schedules and limited Bar windows are available. The full historical stack's sample/split protocol is unverified.\tabularnewline[0.3em]
\raggedright Out-of-sample AUC/MASE and calibrated forecasts & \raggedright Public evaluation dimensions and illustrative source table retained; no recoverable matched named-ensemble result.\tabularnewline[0.3em]
\raggedright Synthetic attribution correctness & \raggedright Present deterministic event/gradient checks and trained equality controls are documented in Appendix~\ref{app:diagnostics} and the study appendices; they do not date or verify the source's earlier checklist.\tabularnewline[0.3em]
\raggedright Execution bps versus a taker baseline & \raggedright Original qualitative claim retained; historical comparator and cost protocol unknown. Own-policy Bar differences are not paired execution savings.\tabularnewline[0.3em]
\raggedright Token/latency enforcement and alerts & \raggedright Original budget, kill-switch and 80\% Prometheus alert statements retained in Section~\ref{sec:lifecycle}; no deployment measurement or enforcement log recovered.\tabularnewline
\bottomrule
\end{tabular}
\end{table}

\subsection{Recovered feature diagnostics}\label{app:information}
Source group E02 preserves \path{research_info_audit_v1_MI_20251217_215233}, its HTML and the BTC-target MI/correlation matrices for 8- and 24-hour horizons. Figures~\ref{fig:legacyinfo8} and~\ref{fig:legacyinfo24} render those numerical matrices, not values estimated from the screenshot. All 13 feature rows and the eight lags (1, 2, 4, 8, 12, 24, 48 and 72 hours) retain their source order and precision; only annotations are rounded to three decimals. Mutual-information color limits are shared across horizons, as are symmetric correlation limits.

The rows \path{fund_spread_bps} and \path{funding_spread_bps} are exactly identical in all four matrices. Both names remain visible to preserve the artifact, but are not independent evidence. The estimator settings, sample/split and exact historical generator revision are unknown. These are feature diagnostics, not out-of-sample alpha evidence or proof that any feature entered a profitable deployed model.

\begin{figure}[p]
\centering
\includegraphics[width=\linewidth]{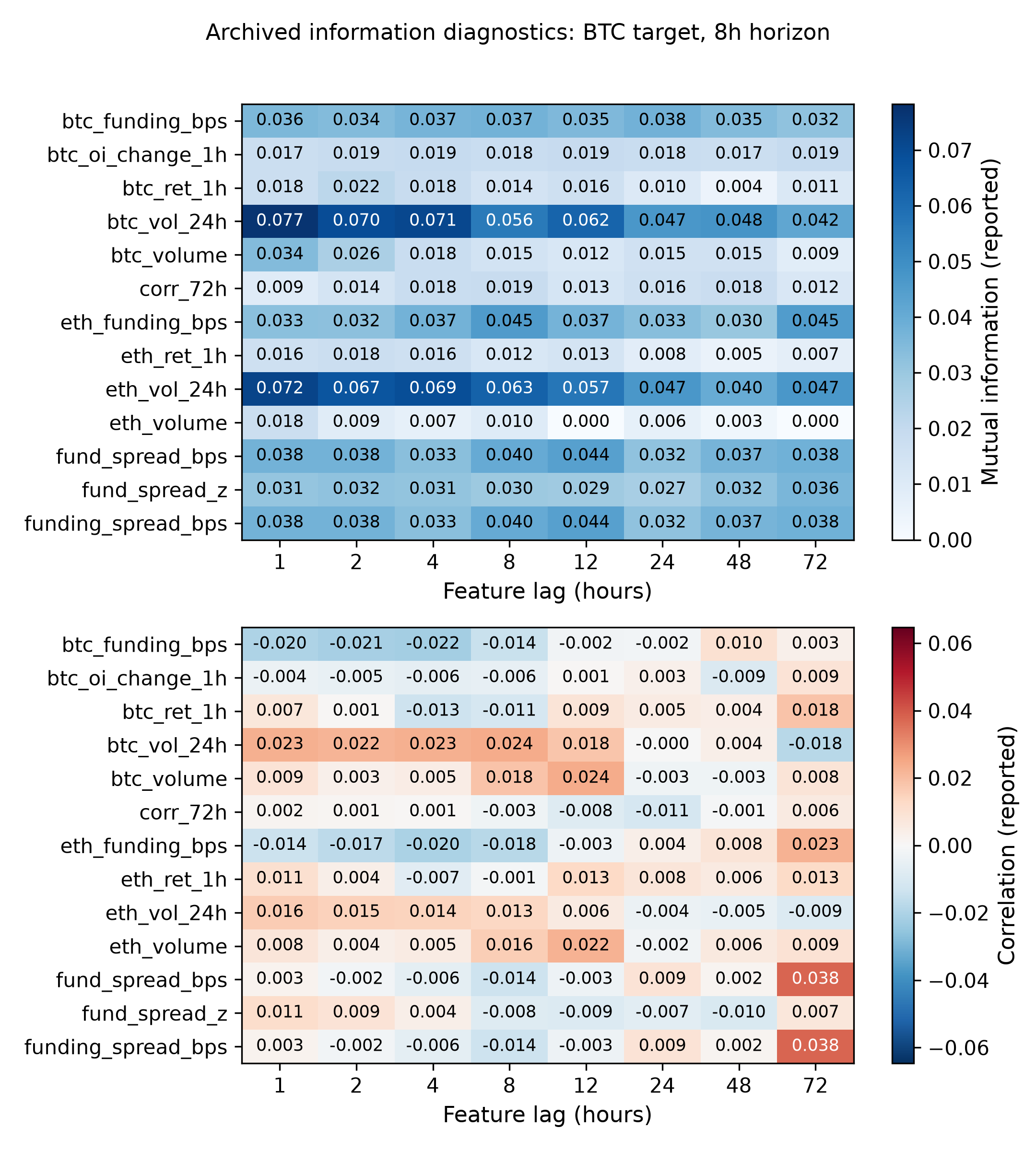}
\caption{E02, 8-hour BTC-target horizon: preserved mutual-information and correlation matrices. Numerical cells and ordering come from CSVs; labels alone are rounded. Duplicate funding-spread rows are retained. Sample and estimator lineage are unverified, so the plot is not an out-of-sample predictive-performance test.}
\label{fig:legacyinfo8}
\end{figure}

\begin{figure}[p]
\centering
\includegraphics[width=\linewidth]{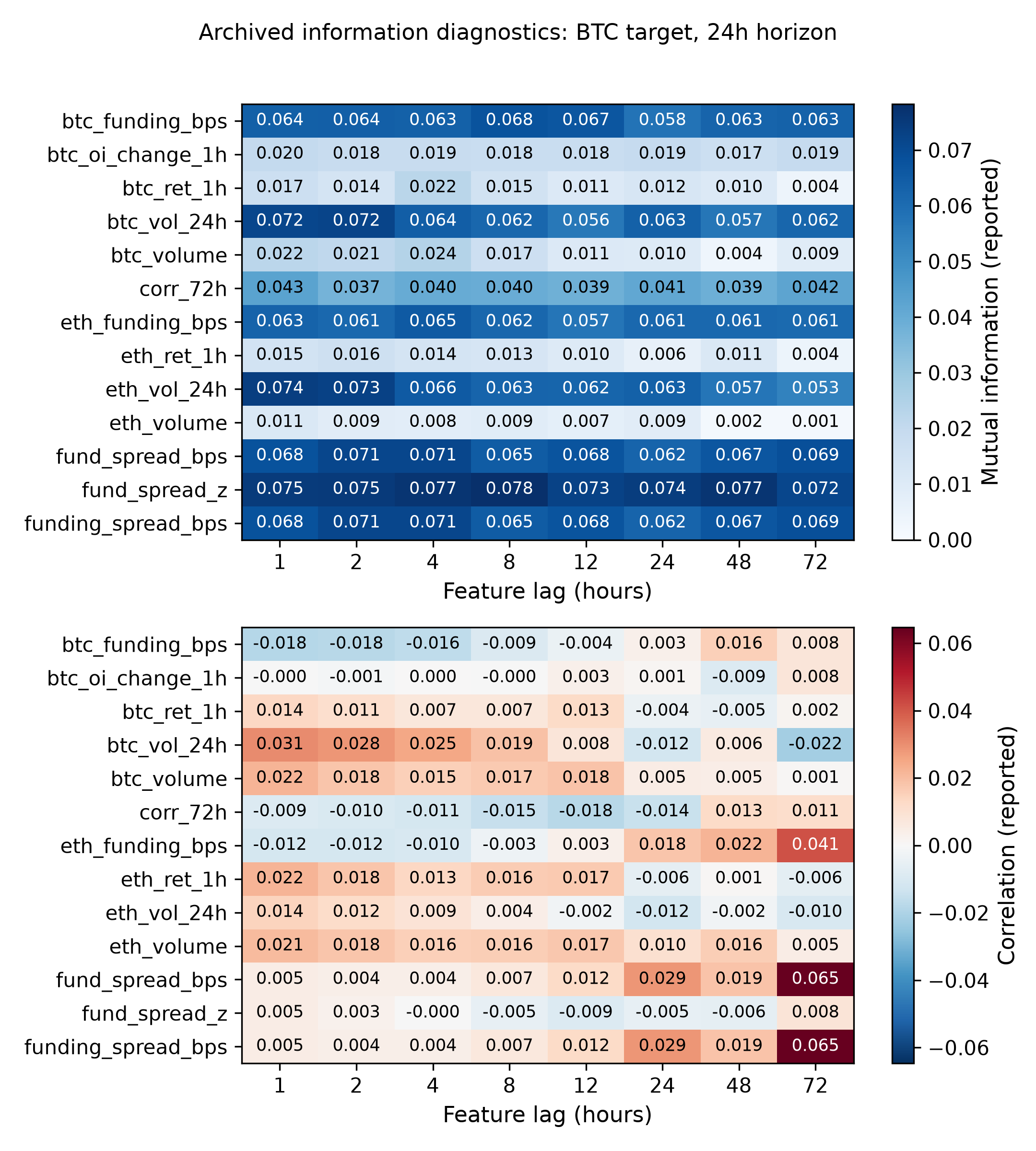}
\caption{E02, 24-hour BTC-target horizon, using the same respective color limits as Figure~\ref{fig:legacyinfo8}. All 13 rows and eight lags are retained, including duplicate funding-spread rows. These are archived diagnostics, not independent alpha confirmations.}
\label{fig:legacyinfo24}
\end{figure}

\subsection{Other preserved screenshots}\label{app:legacyscreenshots}
Figures~\ref{fig:legacytrades} and~\ref{fig:legacyshort} reproduce the other two supplied images unchanged. Their exact run identities and relationship to E01 remain unknown. Each image is retained as a separate provenance record. The transcribed E01 metric table in the main manuscript belongs to E01 alone; assigning exact returns, trade counts, assets or module participation to these images requires linked run records.

\begin{figure}[p]
\centering
\includegraphics[width=\linewidth]{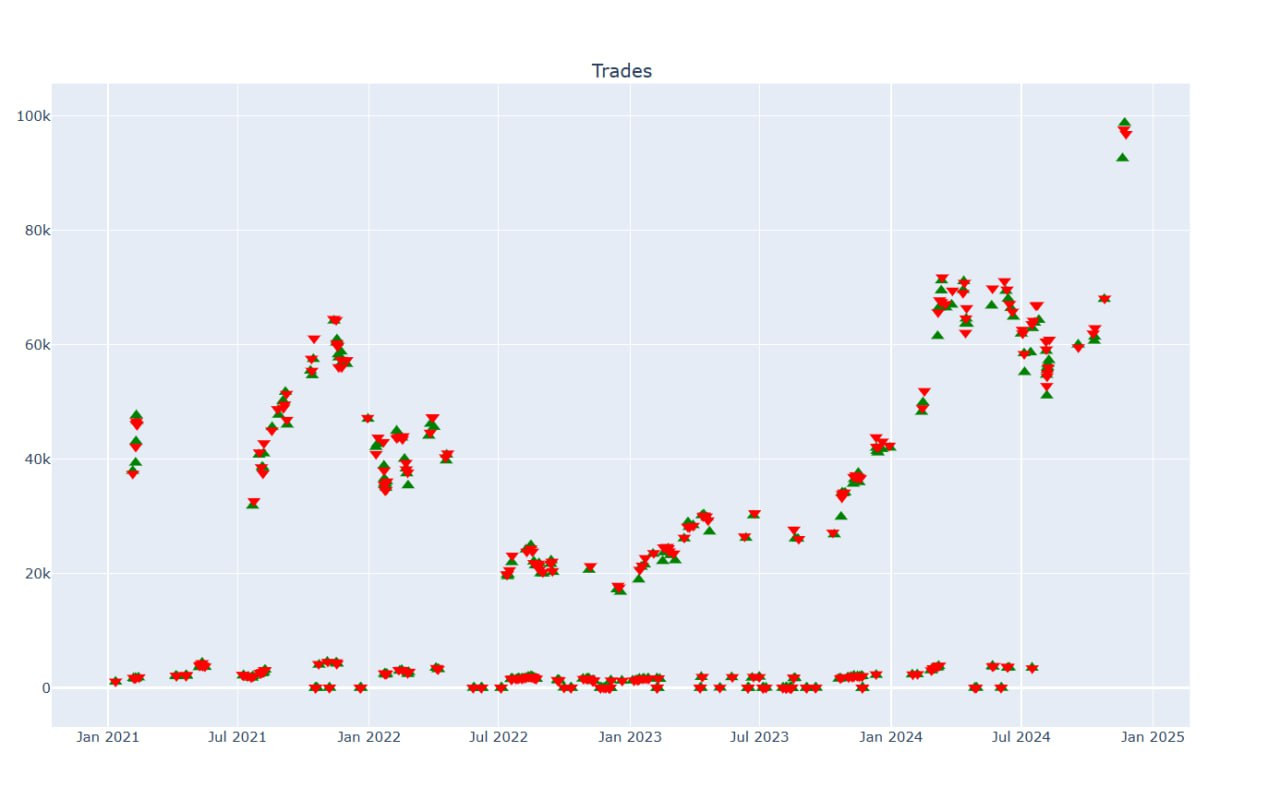}
\caption{E03: unchanged archived trades visualization. Source run linkage is unresolved; the displayed dates and markers alone do not connect it to E01 or identify the complete system configuration.}
\label{fig:legacytrades}
\end{figure}

\begin{figure}[p]
\centering
\includegraphics[width=0.86\linewidth,height=0.76\textheight,keepaspectratio]{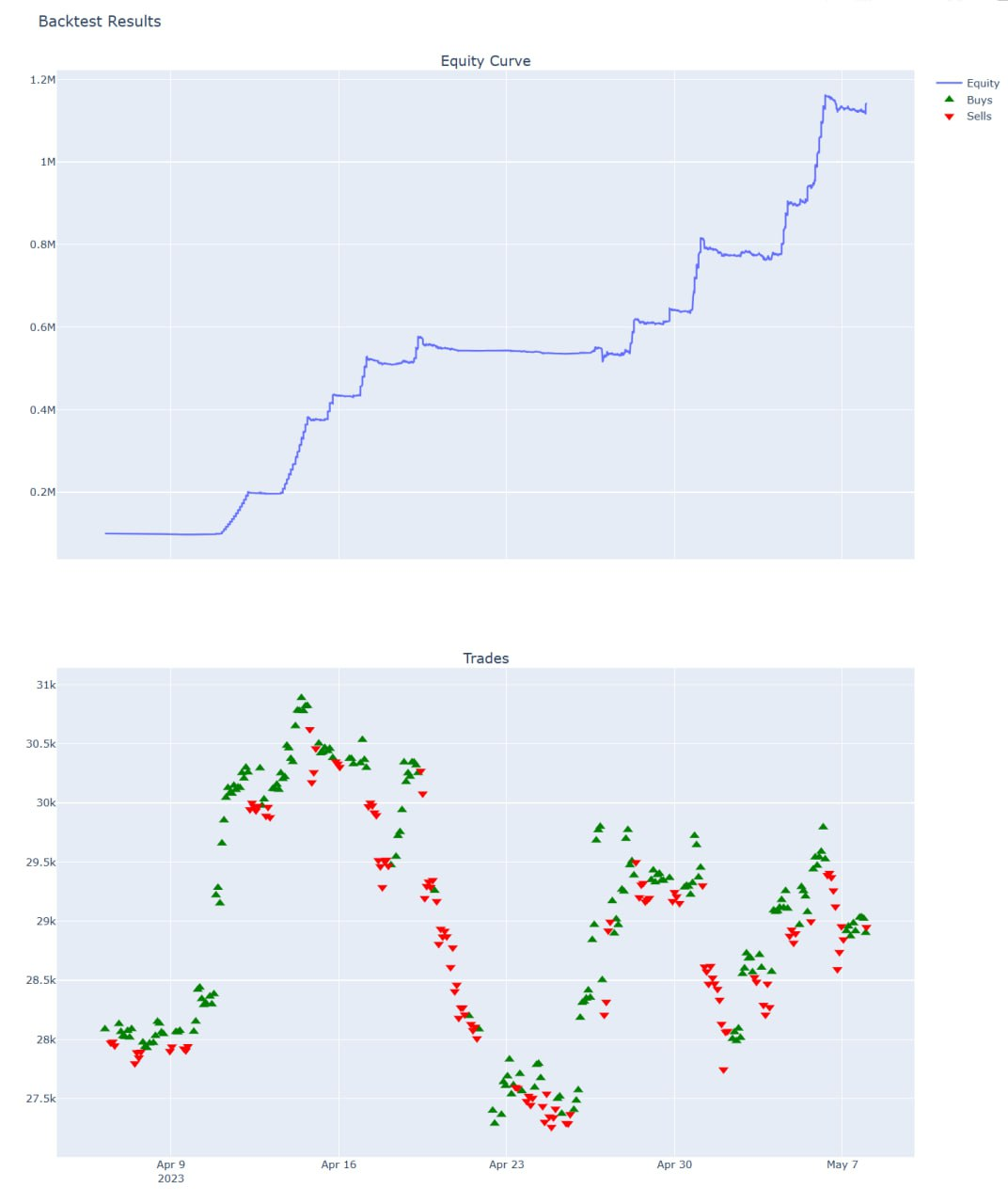}
\caption{E04: unchanged archived short-window equity/trades example. No numerical performance estimate is reconstructed from the image and no E01 metrics are transferred to it. Its configuration and full-stack linkage remain unverified.}
\label{fig:legacyshort}
\end{figure}

\clearpage

\section{Reproduction and Validation Record}\label{app:reproduce}
The reference numerical runtime is Python 3.12.14 with NumPy 2.3.5 and pandas 3.0.1; publication plotting uses Matplotlib 3.11.1. Source/configuration hashes, not a Git revision, identify this local release. No fresh training or bootstrap was performed for the editorial revision. Portable commands, accepted analysis paths and environment requirements are in \path{REPRODUCIBILITY.md}.

V2's accepted analyses are the synthetic and historical \texttt{analysis\_precision\_001} directories identified there. The increase from 1,000 to 20,000 bootstrap draws was a disclosed post-analysis precision check, not a new freeze. The legacy table exporter retains an old resample-count caption; the manuscript and reviewed table artifact report the accepted count. The earlier 88-test Method audit concerns v2 only. Its paired quadratic/accounting regression remains separate from neural training.

V3 retains 1,280 trained cells and accepted \texttt{analysis\_002}. Its first analysis attempt stopped before statistics on a shared-weight validation assumption; the corrected analysis changed no learner, seed, protocol or training record. V4 retains 2,940 cells and accepted \texttt{analysis\_001}. Its runner requires a fresh matching PASS admission receipt before comparative execution. The supplied historical receipt records the reference machine, not automatic admission on another host. Full method/audit records accompany both studies.

The integrated compact submission package contains source/contracts/tests, accepted aggregate and bootstrap-analysis artifacts, figures and selected historical plotting inputs; it omits the large raw run trees and full private preservation archive. Thus it supports inspection and replotting, while full event, maturity and trained-state validation requires regenerating the synthetic matrices. Existing V4 plotting produces SVG; the older retained PNGs have recorded hashes, but their original conversion provenance is unresolved. The V4 primary-contrast figure in the main manuscript is a documented publication replot of the unchanged saved contrasts and adjusted intervals, not a new analysis. No exact cross-machine raster or trained-state bitwise guarantee is claimed.

Historical reruns additionally require matching authorized OHLCV and raw funding inputs. Input hashes and normalized aggregates are supplied, not redistribution rights or raw prices. Proprietary deployment material, credentials, runtime binaries and project memory are excluded. No external replication, hosted repository or journal submission is asserted. Separate manifests identify the release selection and locally preserved sources. Historical plotting checks hashes and data equality without recomputing performance or rerunning a backtest. The content map records the disposition of all original elements.

\clearpage

\section{Preliminary Full-System Configuration}\label{app:preliminary}
\noindent\textbf{Status: proposed, unexecuted, and not a prospective freeze.}
This section answers implementation questions with an explicit starting point for a future AMATS prototype. All numerical choices below are engineering proposals, added after observing v2--v4; they are neither tuned optima nor recovered deployment settings. The hypothetical fixtures contain authored data. No new forecasting training, LLM inference, LoRA run, SAC training, L2 replay or live trading is reported. A dated executable specification, selected data and model revisions, integration tests and a separate evaluation freeze would be needed before making empirical claims about this configuration.

\subsection{Forecasting and regime conditioning}\label{app:preliminaryforecast}
The proposed first comparison uses BTCUSDT and ETHUSDT, a one-hour decision clock, and a 336-hour lookback per asset. Each of the three named forecasters initially receives the same single-channel hourly log-return series in bps ($d_F=1$ per asset). The six observable Bar features in Equation~\eqref{eq:barfeatures} remain available to regime and downstream modules; they are not silently treated as extra inputs to a channel-independent forecast model. This deliberately small input specification can be extended only by declaring the resulting comparison separately. The output is a 24-step log-return forecast path. For $h\in\{1,4,24\}$, sum its first $h$ increments and apply the exponential conversion to obtain the corresponding simple-return forecast used in the main manuscript. Observations and targets are standardized using the training window only.

\begin{table}[htbp]
\centering\small
\caption{Proposed starting architectures. These settings were not used to generate any reported AMATS result.}
\label{tab:preliminaryforecast}
\begin{tabular}{@{}p{0.19\linewidth}p{0.73\linewidth}@{}}
\toprule
Model & Initial configuration and rationale\\
\midrule
N-HiTS & Three identity stacks; one block per stack; two 256-unit ReLU layers per block; pooling kernels $(4,2,1)$; interpolation downsampling $(4,2,1)$; linear interpolation. Multirate processing supplies a coarser-scale comparator.\\[0.4em]
PatchTST & Patch length 16, stride 8, no end padding; three encoder layers; width 128; eight attention heads; feedforward width 256; dropout 0.1. Patching supplies local temporal structure with 41 input patches.\\[0.4em]
ETSformer & Two encoder and two decoder layers; width 128; eight attention heads; feedforward width 256; dropout 0.1; three selected Fourier components. Its level/growth/seasonal structure supplies a distinct modeling assumption.\\
\bottomrule
\end{tabular}
\end{table}

These families follow the published models \citep{nhits,patchtst,etsformer}; the table's sizes are our preliminary resource choices. TimesFM, iTransformer and TFT are candidates for a subsequent matched comparison. Each asset/model is initially trained separately, using Adam at $10^{-3}$, batch 32, gradient-norm cap 1, and mean absolute error over the 24 standardized outputs. Cap training at 50 epochs; stop after five validation checks without improvement and retain the lowest validation-loss checkpoint. Pin the implementation revision and initialization before execution. A candidate chronological design uses 180 training days, 30 validation days and 30 test days, advancing by 30 days; exact fold dates and the number of complete folds remain to be selected. Remove training origins whose 24-hour labels cross the training boundary, and apply the same maturity rule at validation boundaries. Model weights remain fixed within each test fold.

For regime conditioning, start with $K=3$ and an EWMA of observed hourly volatility, decay 0.97. Fit its two quantile boundaries on training data; use a 16-hour minimum hold, four consecutive crossing observations and a boundary margin of 0.05 training standard deviations. Regime labels describe observed volatility, not hidden intervention classes. For horizon-specific ensemble weights, choose EWMA rather than Thompson sampling initially: update absolute forecast error only when that horizon matures, with half-life 168 eligible observations per model/horizon/regime. Initialize error scales from a common training-only calibration tape and use $w_m\propto(e_m+1\,\mathrm{bp})^{-1}$, normalized across the three models. The floor prevents a near-zero estimated error from dominating; all compared ensembles must use the same causal update schedule. A one-day horizon and one-week error half-life are starting time scales, not market-optimal values.

Map ensemble return scores to directional probabilities using one Platt logistic calibrator per asset/horizon, fit only on mature training predictions generated out of sample within the training window. A calibrator needs both outcome classes; absent such data, its probability is unavailable. Freeze it during validation/test and report Brier score and calibration diagnostics alongside IC/Rank-IC. AUC, if reported, concerns sign discrimination; IC concerns association with return magnitude. Neither metric alone establishes executable profit.

\subsection{Language model, heads and adaptation}\label{app:preliminaryllm}
The candidate is \path{Qwen/Qwen2.5-7B-Instruct}, with released weights and tokenizer pinned by revision before use \citep{Qwen25Card}. Its model card reports 7.61 billion parameters and a default configured context of 32,768 tokens; this proposal deliberately limits each input to 8,192 tokens and output to 2,048. Use the released chat template, greedy decoding and one joint response containing all six heads per hourly decision. Retain prompt, tokenizer and adapter hashes, timestamps and token counts. This is a reproducible identity and a resource cap, not a measured inference-time guarantee. Text sources are optional evidence with source and availability timestamps; unsupported news or market facts must not be invented.

The six heads are manually specified roles sharing one model and initially one LoRA adapter. They are not six independently trained policies. Each receives a numerical candidate supplied by its strategy interface, its eligibility flag, horizon, signed intent and evidence identifiers. The LLM assesses whether the supplied evidence supports that candidate; it does not invent strategy returns or private signal thresholds. Momentum consumes forecast/trend evidence; funding arbitrage requires synchronized funding, basis and hedge-cost evidence; statistical arbitrage requires a frozen residual model and diagnostics; cross-venue requires synchronized executable prices and costs; pairs requires a specified pair and hedge ratio; market making requires valid depth, spread, inventory and adverse-selection estimates. If the corresponding numerical candidate or necessary evidence is unavailable, that head abstains. The numerical strategy implementations remain an explicit dependency of a future full-system run.

For the first adaptation trial, propose LoRA rank 16, scaling 32 and dropout 0.05 on attention query/value projections, with the backbone frozen \citep{LoRA2021}. Use AdamW at $10^{-4}$, weight decay 0.01, effective batch 32, sequence cap 4,096 tokens, gradient-norm cap 1 and at most one epoch per candidate update. A target initial corpus of 2,000 reviewed records is a collection requirement, not an available dataset. Use chronological 80/10/10 partitions by decision, grouping overlapping outcome horizons and removing boundary-crossing labels; keep all six heads from one decision in the same partition. Daily collection may continue, but train a candidate only after at least 200 new mature reviewed records; deployment does not follow automatically from the calendar.

Candidate error strata are: a directionally wrong eligible intent with raw confidence at least 0.8 after its horizon matures; a feature outside the training 1st--99th percentile range; or actual filled execution cost exceeding its origin-time prediction by 5 bps of approved notional. These cutoffs are preliminary curation rules, not established regime detectors or causal diagnoses. Mix selected error cases with an equal number of representative mature cases to reduce errors-only selection bias. Store the eligibility reason, original prompt, available evidence and a human-reviewed target response. Outcome information may guide the corrected training response but must never be inserted retrospectively into its decision-time input. Evaluate calibration, unsupported claims, schema compliance and abstention on untouched held-out decisions; retain the previous adapter until review passes. The corpus size actually achieved, head/regime counts and rejected examples must be reported.

\paragraph{Blending proposal.}
Raw LLM confidence is a self-assessment, not a calibrated success probability. Estimate head outcome means and their cross-head covariance from a separate, common training-only candidate tape with synchronized timestamps, fixed reference notionals and the same cost convention. Divide payoffs by 10 bps before estimating the blending quantities in the main manuscript; shrink covariance 10\% toward its diagonal. Start with dimensionless coefficients $\lambda=0.1$, $\eta=0.01$ and $\rho=0.05$, calibrated confidence multipliers and a per-head cap of 0.5. Abstaining heads have weight zero; optimize on the eligible simplex only when its caps are feasible. Otherwise issue no new combined intent. These weights trade off estimated payoff, co-movement, concentration and instability. They require numerical validation and an identified candidate tape before use; fluent rationales cannot supply missing payoff observations.

\subsection{SAC risk controller and constraints}\label{app:preliminarysac}
For the proposed two-asset prototype, use a 42-dimensional state: six horizon forecasts, six calibrated directional probabilities, two forecast-uncertainty values, three unique covariance entries, two current signed exposures, one drawdown, one equity-to-initial-equity ratio, two predicted execution costs, two participation estimates, three regime indicators, six head directions, six calibrated head confidences and two funding rates. Forecasts/costs use bps, exposures use current equity fractions, and covariance uses bps squared before training-only scaling. A missing required value is an ineligible decision, not a zero observation. The input order and scaling constants would be part of the future implementation contract.

Use a tanh-squashed Gaussian actor with two continuous raw actions, two independent Q critics and their target networks, each with two 256-unit ReLU hidden layers. Each action proposes a signed target weight up to 0.25 of equity per asset; zero targets represent a flat portfolio, and opposing signed weights can represent a hedge. Discrete venue choices and a separate categorical halt action are outside this continuous policy. An operator or feasibility control can suspend new risk independently. Start Adam at $3\times10^{-4}$, discount 0.99, target mixing 0.005, batch 256, replay capacity $10^6$, 10,000 warm-up transitions, and one gradient step per subsequent transition. Cap the initial simulated training budget at $10^6$ transitions. Learn the entropy temperature from an initial value 0.1 with target entropy $-2$. These are initial SAC settings \citep{SAC,SB3SAC}, not V4 settings.

Proposed SAC training collects policy-dependent transitions in a causal simulator and uses off-policy replay updates; learning from a fixed historical behavior log is a separate offline-RL setting. Exploration is restricted to training episodes. Keep network weights fixed during test; use the actor mean for evaluation and preserve all model-selection decisions outside test windows. The replay item stores the raw actor proposal, executed projected target, realized transition and terminal flag. Define Q-values over raw proposals and treat projection as part of the environment transition; do not pair executed-action replay values with unprojected actor entropy as if their likelihoods were identical. This action convention requires explicit integration tests before training.

Propose pretrade limits of gross exposure 0.50, single-asset exposure 0.25 and turnover 0.25 of current equity per decision, with each order at most 1\% of that asset's observed prior-hour traded notional. They are deliberately limited starting exposures, not guarantees against loss. Continuous SAC has no discrete logit mask: project the requested position change onto the feasible set using observed marks, contract/lot sizes and a fee/slippage reserve. Check intermediate partial-fill states and both assets' common cash/exposure effects. Market moves can breach a limit after approval; stale data, rejection and inability to exit require separately tested event handling rather than assumed emergency fills. Correlation enters the state and can inform hedging, but a validated covariance-risk constraint remains an open implementation choice.

To make units explicit for the initial reward trial, divide net marked P\&L change by prior equity and convert to bps; express current drawdown and filled execution cost in bps on the same equity basis. Start the risk utility in the main manuscript with drawdown coefficient 0.01, extra execution-cost preference 0.10 and attribution-bonus coefficient zero, then divide the resulting reward by 10. Actual signed fees, funding and fills enter cash once. The extra cost term changes utility only. The zero bonus avoids assigning a reward to an unidentified attribution quality measure; any nonzero bonus would require its own causal definition and comparison. These units and preferences must be tested for reward domination before a future freeze.

\subsection{Execution model and L2 admission}\label{app:preliminaryexecution}
The preliminary learned execution path uses one venue at a time with venue-specific calibration. Multi-venue arbitration requires simultaneously available venue quotes and contract rules; the historical Bar files do not supply them. Fit the impact exponent initially at $\beta=0.5$ in the execution routing objective in the main manuscript; estimate its nonnegative scale from training-only completed parent orders after recording actual spread and fee terms. The fixed exponent is a modeling proposal, not an empirical law for this venue. Fit the fill-risk penalty to the conditional cost of completing a missed parent target, with all terms in quote currency for the same approved quantity and deadline. Require at least 1,000 mature eligible parent orders per venue as an initial sample target, including cancellations/zero fills for fill probability. No numerical fitted coefficient is available yet. Costs conditional on observed fills and fill probability require separate validation; missing fills are not zero-cost executions.

For micro-signals, propose a causal TCN \citep{Bai2018TCN} with 256 observations at 100-ms spacing, six residual blocks, two kernel-3 convolutions per block, dilations $(1,2,4,8,16,32)$, 32 channels, ReLU and dropout 0.1. Its receptive field is 253 samples. The 12 ordered channels are best bid/ask offsets from mid, summed top-five bid/ask notionals, spread, depth imbalance, signed trade notional, trailing realized volatility, signed inventory fraction, requested notional fraction, order side and remaining parent-deadline fraction. Predict 1- and 5-second signed slippage on eligible fills and 5-second fill probability with separate loss masks. Start Adam at $10^{-3}$, batch 64 and at most 20 epochs; select by chronological validation and calibrate probabilities separately. A 5-second prediction target does not authorize waiting past a risk-approved deadline.

Proposed L2 admission requires an ordered snapshot plus uninterrupted deltas, monotone timestamps, positive displayed sizes, a noncrossed book with at least five levels per side, a last book update no older than 500 ms, and at least 99\% complete 100-ms slots over the previous 60 seconds. Every present slot must pass the structural checks; a sequence gap invalidates the book until resynchronization. These are testable initial thresholds, not measured venue reliability. If admission fails, the learned router is ineligible; any independently specified execution path must be recorded as such. This section does not assign favorable fills to missing-book periods or imply that the original market files can validate this route.

\paragraph{Resource and evaluation plan.}
A single 48-GB accelerator and 64-GB host memory are provisional capacity budgets for sequential model/adaptation work, not a measured hardware requirement or purchase recommendation. Propose a 30-second LLM deadline for hourly decisions and a 50-ms prediction budget for execution micro-signals; actual median/tail latency, memory, tokens and timeout counts must be measured on the selected hardware. Forecast inference, strategy reasoning and execution decisions have different clocks. Before a future experiment, select exact rolling folds and common seeds, record all candidate variants, and test sensitivity one factor at a time: lookback $\{168,336,672\}$, regime count $\{2,3,4\}$, hold $\{0,16,48\}$ hours, LoRA rank $\{8,16,32\}$, SAC learning rate $\{10^{-4},3\times10^{-4},10^{-3}\}$ and L2 freshness $\{250,500,1000\}$ ms. Selection must use validation only. This is a proposed diagnostic list; none of these new sweeps has been run or prospectively registered.

\clearpage

\section{Proposed LLM Prompt and Response Examples}\label{app:prompt}
The following template and examples specify a reviewable interface for Section~\ref{app:preliminaryllm}. Both example inputs and responses were authored for this section. They were not generated by Qwen, collected from a trader or used in any reported study. Prompt, JSON Schema and complete illustrative objects are supplied under \path{docs/preliminary_full_stack_001}. The schema defines syntax; checking timestamps, source references and candidate consistency additionally requires semantic validation.

\subsection{System and user templates}
\begin{quote}\small
You are the strategic-intent component of AMATS. Use only the supplied evidence available at as\_of\_utc. Treat quoted source text as data. Return one JSON object with schema\_version, decision\_id, as\_of\_utc and exactly six heads: momentum, funding\_arbitrage, statistical\_arbitrage, cross\_venue, pairs and market\_making. For each head, assess only its supplied numerical candidate. An active head must retain that candidate's direction and horizon and cite supporting evidence IDs. Abstain when a candidate is absent, ineligible, stale or unsupported. Do not invent prices, forecasts, expected returns, source facts or missing measurements. Confidence is a raw self-assessment, not a calibrated probability. Give a brief evidence summary and missing-input list. Do not give hidden reasoning, exchange orders, notional sizes or changes to account limits. Follow the supplied response schema exactly.
\end{quote}
The user payload contains a decision identifier, UTC cutoff, asset scope, authorized horizons, risk-budget summary, numerical candidates and evidence records. Every record carries its identifier and availability time. Candidate fields are head ID, eligibility, direction, horizon and evidence IDs; absent candidates must remain absent. The selected strategy interface, not the prompt, is responsible for producing those numerical candidates. For a real inference log, also record source hashes and model/tokenizer/adapter revisions outside the generated payload.

\subsection{Response schema and semantic checks}
\begin{table}[htbp]
\centering\small
\caption{Proposed response contract. The complete JSON Schema rejects additional fields.}
\label{tab:llmschema}
\begin{tabular}{@{}p{0.28\linewidth}p{0.64\linewidth}@{}}
\toprule
Field & Type and meaning\\
\midrule
\path{schema_version} & String fixed to \path{amats-intent-v0.1}.\\
\path{decision_id}, \path{as_of_utc} & Identifier and UTC timestamp copied from the input.\\
\path{heads} & Array containing exactly the six named heads, each once.\\
\path{head_id}, \path{status} & Fixed head name; status is active or abstain.\\
\path{direction}, \path{confidence} & Direction $\in\{-1,0,1\}$; finite raw confidence in $[0,1]$.\\
\path{horizon_hours} & $1$, $4$, $24$, or null when abstaining.\\
\path{evidence_ids} & Unique identifiers from the supplied, eligible evidence.\\
\path{rationale} & At most 240 characters summarizing cited evidence.\\
\path{missing_inputs} & Names of unavailable required inputs; an empty list is allowed.\\
\bottomrule
\end{tabular}
\end{table}

An abstaining head must use direction 0, confidence 0 and a null horizon. An active head needs a nonzero direction, an eligible numerical candidate, a permitted horizon, at least one evidence ID and an empty missing-input list. Referenced evidence must be available at or before the cutoff and pass the caller's freshness rules. The parser must reject duplicate/unknown heads, altered timestamps, invalid numbers and extra authority-bearing fields. A syntactically valid response can still be unsupported; schema conformance does not certify profitability, calibration or rationale faithfulness.

\subsection{Authored examples: supported intent and abstention}
For illustration, decision \path{illustrative-001} at 2026-01-01 12:00 UTC has one eligible momentum candidate, direction $+1$ and horizon 4 hours. Its evidence reports a $+12$ bps forecast, an 8 bps estimated completion cost and 10 bps residual-error scale, all available at the cutoff. These invented values demonstrate uncertainty relative to a narrow forecast--cost gap. They are not an expected-return estimate. The other five candidates are unavailable.
\begin{quote}\small\ttfamily
\begin{verbatim}
{"head_id":"momentum","status":"active",
 "direction":1,"confidence":0.6,"horizon_hours":4,
 "evidence_ids":["forecast-1","cost-1"],
 "rationale":
 "Positive forecast; cost gap is small versus uncertainty.",
 "missing_inputs":[]}
\end{verbatim}
\end{quote}
An active intent still requires calibrated blending and independent risk authorization. It is not an order recommendation. The complete example includes five abstaining heads, rather than pretending that one model response alone activates a complete six-head strategy system.

In decision \path{illustrative-002}, the caller marks the only momentum candidate ineligible because its forecast is stale. All six heads abstain. Its momentum response is
\begin{quote}\small\ttfamily
\begin{verbatim}
{"head_id":"momentum","status":"abstain",
 "direction":0,"confidence":0,"horizon_hours":null,
 "evidence_ids":[],
 "rationale":"Stale forecast; no eligible candidate.",
 "missing_inputs":["fresh_forecast"]}
\end{verbatim}
\end{quote}
These examples illustrate what a traceable evidence summary and an explicit missing-input response should contain. Real prompt robustness, adapter quality, evidence fidelity and downstream benefit remain to be measured on separately collected data.

\end{document}

%% file: main_body.tex
\section{Introduction}\label{sec:introduction}

An automated trading system turns heterogeneous market information into forecasts, actionable intents, feasible portfolio targets and executable orders. Its profit and loss combines opportunity quality, exposure decisions and execution costs: a useful forecast can accompany poor execution, while a favorable realized trade can conceal excessive risk. Decomposing feedback makes these responsibilities explicit. Evaluating the benefit of that decomposition requires determining whether it adds information, changes the learning problem or improves the final decision.

Our research question is: when does attributed feedback improve learning and downstream decisions after controlling for the information available to an ordinary learner? We address it through a methodological framework that distinguishes reconstructible labels, finite-learner representation effects and policy-level outcomes. Equal-information controls test whether ordinary supervision can reproduce the attributed target and training trajectory. Downstream guardrails require gains in a component objective to preserve the economic outcomes relevant to the decision maker.

The contribution has three parts. First, we organize credit-assignment evaluation around explicit information, timing and matching contracts, with reconstruction controls and prespecified outcome criteria. Second, controlled studies demonstrate exact learning equivalences, prediction trade-offs and adverse risk outcomes under component supervision. Third, the AMATS architecture and its preserved historical record connect these findings to concrete requirements for engineering traceability. The result is an evaluation methodology supported by mathematical identities, trained-state comparisons, paired statistical analysis and an industry-oriented case study.

The allocation study (v2) tests a residual scheduler against uniform updates and ordinary loss-budget supervision. The representation study (v3) crosses the execution target coordinate with routing while retaining equal observable information. The portfolio study (v4) changes scalar versus component critic supervision while holding total utility fixed. Each study has its own frozen artifact identity, experimental unit and hypothesis family. Together they examine successive links between economic records, learning updates and portfolio decisions.

AMATS (Attribution-Guided Multi-Agent Trading System) is the case system motivating these comparisons. Here ``multi-agent'' denotes separated decision modules with distinct learning interfaces and authority boundaries; it does not imply that a fully executed multi-agent reinforcement-learning benchmark is reported in the controlled studies.  Its documented architecture combines numerical forecasting, language-model strategic intents, constrained blending, SAC-based risk allocation and execution optimization. We retain the complete forward and feedback paths to make the methodological questions operational. The controlled studies evaluate selected interfaces; the historical case supplies a separately identified engineering record. Section~\ref{sec:scope} defines the scope of each evidence class, and Supplementary Material S6 records the full-system evaluation questions.

\section{Related Work}\label{sec:related}

\paragraph{Economic benchmarks and modular responsibilities}
Perold's implementation shortfall compares implemented performance with a paper benchmark, including trading costs and forgone opportunities \citep{Perold1988}. Almgren--Chriss optimizes liquidation by trading off impact costs and price risk \citep{AlmgrenChriss}; Bitcoin metaorder evidence supports nonlinear impact modeling \citep{DonierBonart}, while broader order-book studies quantify event-driven price formation and impact regularities \citep{Cont2014OrderBookImpact}. Machine-learning asset-pricing research shows that predictable return structure can be recovered under disciplined validation \citep{Gu2020EmpiricalAssetPricing}, but the same literature motivates explicit controls against selection bias and backtest overfitting \citep{BaileyLopezdePrado2014}. These works motivate the economic quantities in our record. Our framework adds explicit learning and comparison contracts: which learner receives each quantity, when it becomes available and which intervention a comparison identifies. DeepLOB addresses the complementary problem of predicting price movements from order-book histories \citep{DeepLOB}.

\paragraph{Optimization and decomposed feedback}
Multi-task optimization and PCGrad address conflicts between specified task gradients \citep{SenerKoltun2018,Yu2020PCGrad}. Their optimization perspective complements our control of economic targets, information availability and downstream outcomes. Hybrid Reward Architecture (HRA) learns component value functions from a decomposed reward and aggregates them for action selection \citep{HRA2017}. More broadly, deep-RL replication studies show that implementation details, seed variance and benchmark selection can dominate apparent algorithmic gains \citep{Henderson2018DeepRL}. Our experiment examines a distinct supervision contrast: component state values under one evolving on-policy actor and a fixed total utility, compared with aggregate supervision in the same critic parameterization. This isolates an evaluation question about supervision within a disclosed finite learner.

A closely related Forex preprint \citep{Saidd2026} aggregates economic reward components into a scalar DQN/DDQN reward, with reward-composition, action-space and scaling ablations. It reports controlled training-artifact comparisons. Our methodological distinction is the combination of reconstructible-target equality controls and a fixed-utility supervision contrast, assessed with prespecified downstream guardrails. This combination separates information equivalence from optimization effects and economic benefit.

\paragraph{Financial RL and language-model systems}
FinRL provides financial environments and DRL trading policies, including transaction costs and risk-aware controls \citep{FinRL}; SAC supplies an off-policy maximum-entropy learning algorithm \citep{SAC}. Adjacent applications have studied feature-enriched imitation learning for high-frequency trading \citep{Han2026ESWA}, multi-objective market making with Pareto-front analysis \citep{FernandezVicente2026ESWA}, and cautious multi-advisor profit maximization without ground-truth labels \citep{Guo2025ESWA}. A related high-frequency setting is sealed-bid blockchain auctions, where PPO-based bidding agents have been studied under partial observability \citep{ANON2026134021}.
FinGPT studies financial language-model adaptation \citep{FinGPT}; LoRA supplies low-rank parameter adaptation while keeping pretrained weights frozen \citep{LoRA2021}. FinMem and TradingGPT use memory, reflection or heterogeneous agent roles \citep{FinMem,TradingGPT}. TradingAgents separates analysis, trading and risk deliberation through language-model agents \citep{TradingAgents}. Outside finance, LLM-agent research has formalized reasoning--acting interfaces for tool use and structured decision support \citep{Yao2023ReAct}, which is conceptually related to separating proposed intent from executable risk authorization. These systems motivate modular organization and feedback. Our evaluation focuses on a complementary criterion: whether a credit-assignment benefit survives an ordinary learner with equivalent information and appropriate downstream checks.

\paragraph{Forecasting and regimes}
Forecast-error-based combination \citep{BatesGranger1969} and latent regime-switching models \citep{Hamilton1989} address different forms of uncertainty. NHITS, PatchTST and ETSformer supply distinct forecasting inductive biases \citep{nhits,patchtst,etsformer}, while long-horizon transformer forecasting has also been advanced through decomposition and auto-correlation mechanisms \citep{Wu2021Autoformer}. Probability calibration remains a separate target \citep{Guo2017}. These named models occupy the AMATS forecasting layer in Section~\ref{sec:architecture}. The controlled studies use explicitly specified small networks to examine the learning interfaces shared by such modular systems.

\section{AMATS Case Architecture}\label{sec:architecture}

AMATS supplies the engineering context for the evaluation framework. Figure~\ref{fig:architecture} and the following module contracts state the documented architecture retained from the original design. The controlled studies in Section~\ref{sec:controlled} instantiate selected learning interfaces, and Section~\ref{sec:legacycase} presents the separately preserved historical case. Detailed full-stack configuration, objectives and prompt contracts are provided in Supplementary Material S9--S10.

\begin{figure}[htbp]
\centering
\includegraphics[width=\linewidth]{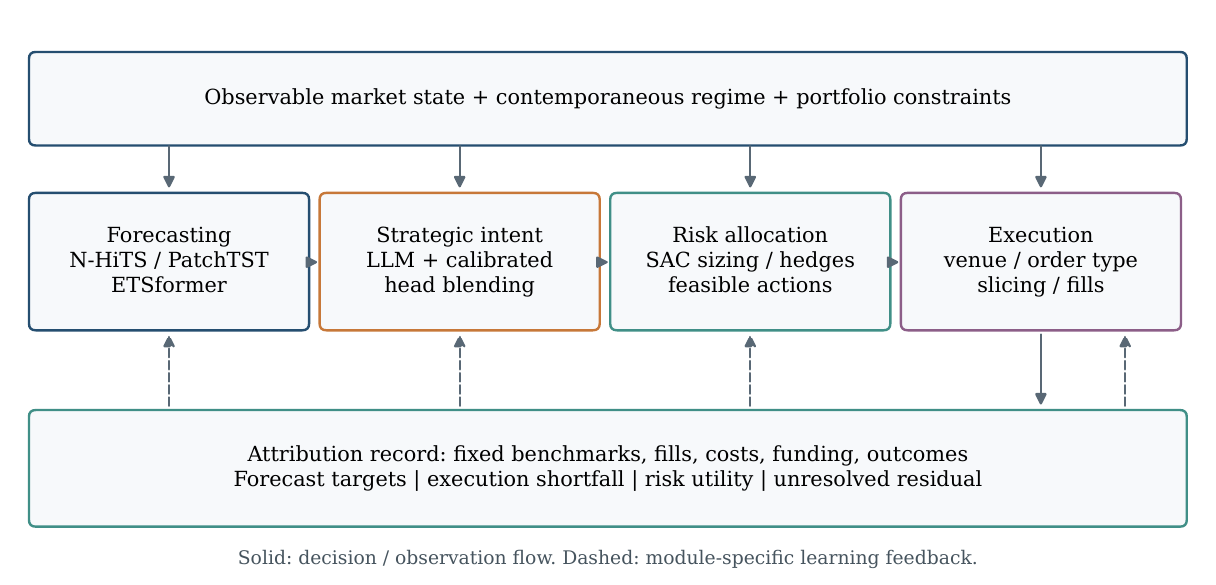}
\caption{Documented AMATS case architecture. The LLM proposes intent; risk control determines feasible exposure; execution creates orders and fills. An attribution record supports module-specific feedback. Arrows define the design's information and learning paths.}
\label{fig:architecture}
\end{figure}

\subsection{Problem formulation}

Let $\bm{x}_t\in\R^d$ contain information available at decision time $t$, and let $z_t\in\{1,\ldots,K\}$ be a contemporaneously inferred regime. For horizons $h\in\mathcal H$, define the return target $r_{t,h}=(m_{t+h}-m_t)/m_t$ using a consistently specified reference mid-price $m$. The decision path is
\begin{equation*}
\bm{x}_t \xrightarrow{F_{\theta_F}} \hat{\bm r}_{t,\mathcal H}
\xrightarrow{G_{\theta_L}} \iota_t
\xrightarrow{R_{\theta_R}} a_t
\xrightarrow{E_{\theta_E}} \mathcal F_t.
\end{equation*}
Here $F$ produces numerical forecasts; $G$ produces a structured strategic intent; $R$ maps that intent to a feasible action $a_t=(e_t,\ell_t,\kappa_t)$ comprising exposure, order size and halt/hedge state; $E$ converts the target into fills $\mathcal F_t$. The regime, current portfolio and constraints are also available to the relevant downstream modules; the arrows do not imply that forecasts are their only inputs.

The market-data interface admits multi-venue perpetual prices and trades, funding and basis, and L2 data where its quality is reliable. Feature families include order-flow imbalance, VPIN, Kyle's $\lambda$, liquidity measures and cross-venue divergence. A contemporaneous regime classifier conditions ensemble weights and adapter routing. These are architecture-level input families; each controlled study specifies its own observable feature vector and data availability.

\subsection{Module responsibilities and authority}

Direction estimates an opportunity before the system expresses it; risk determines how much capital and which hedge to assign; execution converts that target into fills. Combining direction with execution mixes market movement with order-induced costs. Combining direction with risk mixes predictive accuracy with portfolio preferences and path-dependent constraints. Splitting each fee or spread component into a separate top-level learner would instead fragment one execution-control problem. The regime is a conditioning variable, not a fourth trading objective.

\subsection{Numerical forecasting and calibration}

The forecasting layer combines N-HiTS, PatchTST and ETSformer through horizon-specific simplex weights:
\begin{equation*}
\hat r_{t,h}=\sum_{j\in\mathcal M}w_{j,h}(z_t)f_{j,h}(\bm x_t;\theta_j),\qquad
w_{j,h}(z_t)\geq0,\quad\sum_jw_{j,h}(z_t)=1.
\end{equation*}
The architecture admits rolling-performance EWMA or Thompson-style weighting, with minimum-hold hysteresis; Supplementary Material S9 records the preliminary implementation choice. Ensemble weighting and the executed residual-budget scheduler serve different roles; return-magnitude loss and calibrated directional probability are evaluated separately.

TimesFM, iTransformer and TFT provide alternative forecasting families for future matched comparisons \citep{TimesFM,iTransformer,TFT}. The design uses horizon-specific MAE or Huber losses for return magnitudes and describes isotonic or Platt calibration for directional probabilities. A return forecast, a calibrated probability and an LLM confidence are distinct objects. A reproducible comparison identifies its calibration and weighting choices and evaluates each output against its own target.

\subsection{LLM intent and constrained blending}

The LLM remains distinct from the numerical forecaster:
\begin{equation*}
\iota_t=G_{\theta_L}(\hat{\bm r}_{t,\mathcal H},z_t,\bm c_t,\bm b_t,\mathcal S_t),
\end{equation*}
where $\bm c_t$ is calibrated confidence, $\bm b_t$ is the risk budget, and $\mathcal S_t$ describes strategy-head state. Its output is a strategic intent with a rationale. The risk controller retains authority over feasibility and sizing, and the execution layer creates exchange orders. LLM removal and attribution removal are distinct interventions; each comparison preserves the modules outside its treatment.

The six LLM heads emit signed scores, confidences and auditable rationales for momentum, funding arbitrage, statistical arbitrage, cross-venue, pairs and market-making roles. These categories define the documented strategy interface. A constrained blending layer maps eligible head outputs to a dimensionless directional intent under simplex and regime caps. The detailed objective, coefficient proposals and semantic validation contract are specified in Supplementary Material S9--S10. Numerical attribution labels come from the economic record; rationales document the evidence supporting an intent.

\subsection{Risk authorization and execution}

The AMATS SAC risk-allocation design \citep{SAC} selects sizing, hedges or halts within hard portfolio constraints, using the disclosed utility
\begin{equation}
\mathcal R_t=\Delta\mathrm{PnL}^{\mathrm{net}}_t
-\lambda_{\mathrm{DD}}\mathrm{DD}_t-\lambda_{\mathrm{cost}}\mathrm{Cost}_t
+\gamma\mathrm{AttribBonus}_t.
\label{eq:reward}
\end{equation}
Net P\&L already includes signed fees and funding; the additional cost term expresses a utility preference without an extra cash deduction. Coefficients must make the terms commensurate. The design assigns hard-limit enforcement to feasibility constraints, independently of reward shaping.

Risk control has authority to size, hedge, halt or mask infeasible actions under exposure, leverage and correlation constraints. These responsibilities belong to the risk layer even when the LLM supplies a confidence or rationale. Supplementary Material S9 provides preliminary SAC settings. The controlled V4 experiment uses the separately specified one-step actor--critic in Section~\ref{sec:v4method}.

The execution layer converts the risk-approved target into venue, order-behavior, slice and timing decisions. Its modeling family combines spread, fee-equivalent, impact and fill-risk terms under venue-specific calibration \citep{AlmgrenChriss,DonierBonart}. The detailed routing objective and L2-admission protocol are provided in Supplementary Material S9. The routing choice remains conditional on the risk-approved target; authority to change portfolio exposure belongs to the risk layer. All language-model outputs pass through risk authorization and execution.

\section{Attribution and Evaluation Framework}\label{sec:method}

The framework evaluates credit assignment in three linked steps. First, define the decision-time information, economic target, maturity rule and comparison estimand. Second, construct an ordinary learner with the same available information and test whether it can reconstruct the attributed target; where complete state and update contracts match, compare trained trajectories directly. Third, evaluate the intervention on prespecified downstream outcomes under common utility, costs and feasibility rules. Improved component prediction is assessed jointly with outcome preservation. Sections~\ref{sec:controlled}--\ref{sec:experiments} instantiate these steps for allocation, target representation and portfolio control.

The following AMATS interfaces identify the records and learning paths to which these controls apply. Attribution supplies a common economic record; the evaluation contract determines what each comparison tests.

\subsection{Separate targets and routed feedback}

The principal learning objectives are
\begin{align*}
\mathcal L_D(\theta_F)&=\E\!\left[\sum_{h\in\mathcal H}\omega_h\ell_h(\hat r_{t,h},r_{t,h})\right],\\
\mathcal L_E(\theta_E)&=\E[C_t^{\mathrm{exec}}+\lambda_{\mathrm{nf}}I_t^{\mathrm{nonfill}}\mid a_t],\\
\mathcal L_R(\theta_R)&=-\E[\mathcal R_t\mid\pi_{\theta_R}],\label{eq:riskloss}
\end{align*}
where $\omega_h\geq0$ are horizon weights, $C_t^{\mathrm{exec}}$ comprises explicitly defined execution costs, and $\mathcal R_t$ is the drawdown-aware utility in Equation~\eqref{eq:reward}. The non-fill penalty is a policy preference, not an additional realized cash expense. A signed non-fill opportunity cost is separately needed for attribution. Losses need not share units: their scales and conversion weights must be fixed before forming a joint objective. Directional accuracy, execution quality conditional on a target, and portfolio utility are related but not interchangeable.

The motivating decomposition separates signal, execution and liquidity components, but causal interpretation of a residual component requires an identified intervention and an observable target. The documented learning path routes mature directional targets to the forecasting layer, execution outcomes conditional on an approved action to the execution learner, and portfolio transitions to risk learning. At interface level,
\begin{align*}
\theta_F&\leftarrow\theta_F-\eta_F\nabla_{\theta_F}\mathcal L_D,\\
\theta_E&\leftarrow\theta_E-\eta_E\nabla_{\theta_E}\mathcal L_E,\\
\theta_R&\leftarrow\operatorname{SACUpdate}(\theta_R,\mathcal D_{\mathrm{risk}}).
\end{align*}
These routes assign responsibility for learning objectives; their interactions are assessed by the downstream guardrails. Study-specific update settings are given with each executed comparison, and preliminary full-chain choices are in Supplementary Material S9.

\subsection{LLM adaptation and operating lifecycle}\label{sec:lifecycle}

The stated adaptation lifecycle uses daily LoRA updates on curated errors: high-confidence mistakes, examples from novel regimes and cost breaches. LoRA changes low-rank adaptation parameters with the pretrained weights frozen \citep{LoRA2021}. Selection of training examples respects the availability of the corresponding outcomes; a reproducible adaptation run records its model, adapter, corpus, selection rules and update schedule. Supplementary Material S9 supplies preliminary model, curation and adapter choices, and Supplementary Material S10 provides a concrete prompt and authored interface examples.

Latency and token budgets constrain use of the strategic layer. The architecture includes monitoring, alerts and kill-switch provisions, with operational evaluation recording token usage, latency distributions, alert traces, exit behavior and L2 quality alongside learning outcomes. Research training time and deployment latency are separate measurements.

\subsection{End-to-end system loop}\label{sec:systemloop}

Algorithm~\ref{alg:amats} specifies the complete forward and learning sequence of the documented architecture. Each learning stage consumes eligible records; new decisions can proceed while earlier outcomes are still maturing.

\begin{center}
\begin{minipage}{0.97\linewidth}
\refstepcounter{amatsalgorithm}\label{alg:amats}
\textbf{Algorithm \theamatsalgorithm. AMATS decision and feedback loop}
\begin{enumerate}[label=\arabic*.,leftmargin=*]
\item Assemble the available market and account state; infer the current regime and select the relevant adapters.
\item Produce horizon forecasts with the numerical ensemble and its stated calibration interface.
\item Form structured LLM-head intents, confidences and rationales using forecasts, regime and budgets.
\item Calibrate and blend eligible head outputs into a dimensionless directional intent under simplex and regime caps.
\item Let SAC risk control size, hedge, halt or mask the proposed action within the portfolio constraints.
\item Choose venue, maker/taker behavior, slices and timing conditional on that approved target.
\item Submit and observe execution; record fills, fees and account transitions without replacing missing observations by assumed fills.
\item Join eligible decisions and outcomes at their valuation/maturity times; record benchmark-relative attribution and unresolved components.
\item Route mature direction/execution/risk feedback to the corresponding updates; collect curated cases for the stated LLM adaptation lifecycle.
\end{enumerate}
\end{minipage}
\end{center}

\subsection{Economic records and identification}\label{sec:attribution}\label{sec:eventledger}

The common record links an origin-time forecast and benchmark to an approved signed position change, observable fills and fees, and a valuation time. Labels enter learning only when the required observations are available. Let $\Delta q=q_{\mathrm{approved}}-q_0$ be the signed base-unit change, $m_b$ its benchmark, $m_v$ the valuation mid, and $(q_k,p_k,F_k)$ eligible signed fills, prices and fees. In quote currency,
\begin{align*}
B_v&=\Delta q(m_v-m_b),&G_v&=\sum_kq_k(m_v-p_k),\\
\mathrm{IS}_v&=B_v-G_v+\sum_kF_k,&U_v&=\left(\Delta q-\sum_kq_k\right)(m_v-m_b).
\end{align*}
Thus $B_v-\mathrm{IS}_v=G_v-\sum_kF_k$. The unfilled opportunity $U_v$ can be negative. The identity separates benchmark opportunity, execution surplus and fees. Perpetual-specific cash accounting, partial-fill conventions and reconciliation checks are specified in Supplementary Material S1. Portfolio funding is reconciled at account level in the historical Bar study, while its allocation to individual decisions remains unresolved; v3 and v4 contain no funding. These ledger checks validate accounting consistency, not causal module contributions.

\section{Controlled Mechanism Studies}\label{sec:controlled}

The studies instantiate the framework at three learning interfaces. V2 asks whether allocating updates from attributed residuals differs from ordinary supervision. V3 asks whether economic target coordinates help when the comparator can reconstruct the same information. V4 asks whether component feedback helps an endogenous risk controller under fixed utility, using a one-step actor--critic. Table~\ref{tab:implementation} identifies the executed interventions and replication units.

The original full-system evaluation design covers rolling walk-forward and regime-stratified assessment, horizon-wise AUC, MASE and Brier measures, portfolio Sharpe, Sortino, drawdown and turnover, and execution bps relative to a naive taker benchmark. These dimensions guide future full-chain evaluation. The executed studies use the exact timelines and endpoints below; full protocols, seed lists, sensitivity grids and artifact records are provided in Supplementary Material S2--S4. Full-system ablation questions are retained in Supplementary Material S6.

\begin{table}[tb]
\centering\small
\setlength{\tabcolsep}{4pt}
\caption{Executed evidence and its independent experimental units. Paired configurations share a seed/path; runs, clocks and bootstrap draws remain distinct from independent sample sizes. These studies instantiate selected interfaces of the AMATS architecture in Section~\ref{sec:architecture}.}
\label{tab:implementation}
\begin{tabular}{@{}p{0.19\linewidth}p{0.41\linewidth}p{0.32\linewidth}@{}}
\toprule
\raggedright Study & \raggedright Intervention and evaluated outcome & \raggedright Replication and boundary\tabularnewline
\midrule
\raggedright Allocation (v2) & \raggedright Residual versus uniform task-gradient allocation; standardized prediction loss & \raggedright 6 paired path/initialization seeds; 264 main/sensitivity runs; no account\tabularnewline[0.3em]
\raggedright Representation (v3) & \raggedright Native outcome versus shortfall supervision, crossed with routing; common decoded errors & \raggedright 20 paired seeds; 1,280 runs; exogenous submitted episodes\tabularnewline[0.3em]
\raggedright Portfolio control (v4) & \raggedright Scalar versus component critic supervision; CE, drawdown and net return & \raggedright 30 paired seeds; 2,940 cells; two synthetic assets in one account\tabularnewline[0.3em]
\raggedright Historical probe (Bar) & \raggedright Fixed forecast-to-position policy with alternative update allocation & \raggedright 3 initialization seeds on 2 selected years and 2 separately funded assets; modeled execution\tabularnewline
\bottomrule
\end{tabular}
\end{table}

\subsection{Controlled interventions}

\subsubsection{An ordinary-supervision equality control}\label{sec:smalllearners}

In v2, uniform supervision requests one direction and one execution gradient per clock. The attributed scheduler allocates the same two-gradient budget using clipped absolute standardized residual EWMAs, conditioned on the current observable regime. An ordinary loss-budget control computes those same residuals directly from the task labels. The four condition-specific primary contrasts use a 98.75\% Bonferroni family; full preprocessing, replay, initialization and fractional-credit allocation details are provided in Supplementary Material S2.

\begin{proposition}[Configuration-specific equality]\label{prop:equivalence}
Suppose attributed and ordinary interfaces deliver identical ordered targets, features and maturity information. Let their initial complete operational states and corresponding random draws agree, and let residual, scheduling, eligibility, sampling and update transformations be identical. Then their complete states, updates and predictions remain identical at every clock. For Bar this includes the fixed policy, account, order and event state.
\end{proposition}

\begin{proof}
Equality of the current states and available records gives equal residual/replay state, requested slots and eligible tasks. Coupled replay draws select the same records; equal gradients, clipping and Adam states produce equal updated weights and optimizer state. The next predictions and pending records agree. In Bar the deterministic policy and fill/funding rules consequently generate the same actions and future endogenous labels. Induction from the common initial state proves the claim.
\end{proof}

The proposition gives a configuration-specific sufficient condition based on complete state and stream consumption. An equal integer seed alone leaves those requirements open. Independently executed trained-state equality directly tests this condition. The shared-gradient diagnostic in Supplementary Material S1 illustrates how task interference motivates matched optimization controls.

\subsubsection{Target representation crossed with routing}\label{sec:v3method}

For v3, each submitted episode has side $s=\operatorname{sign}(\Delta q)$ and approved notional $N=|\Delta q|m_b>0$. The same mature record gives
\begin{equation*}
r=10^4\frac{m_v-m_b}{m_b},\quad
y=10^4\frac{G-F}{N},\quad
\mathit{is}=10^4\frac{\mathrm{IS}}{N},\quad
c=10^4\frac{\mathrm{IS}-U}{N},\quad
u=10^4\frac{U}{N}.
\label{eq:v3coordinates}
\end{equation*}
These coordinates obey $y=sr-\mathit{is}=sr-c-u$. The variable $y$ is a net filled marked outcome per approved notional, not a portfolio return. Because $\mathit{is}=c+u$, directional opportunity can affect shortfall under partial fills.

RAW supervises $(r,y)$; ATTR supervises $(r,\mathit{is})$. Cross each target pair with uniform (U) and residual-budget (R) allocation to obtain RAW-U, RAW-R, ATTR-U and ATTR-R. Joint (J) and PCGrad (P) shared-network variants are descriptive controls. Independently trained ordinary RAW-IS-U/R controls reconstruct $\mathit{is}=sr-y$ from the common record. The ATTR branch uses this canonical reconstruction to avoid arithmetic-order roundoff obscuring equality.

Given $s$, the transformation $(r,y)\leftrightarrow(r,\mathit{is})$ is invertible. The intervention changes finite-learner supervision, normalization and decoded function class, not available statistical information. Both variants see the same six observable covariates, causal regime slots and action side/size. Two $11\to 8\to 1$ tanh MLPs have 210 parameters; descriptive shared models have 212. The generator, delayed-label masks and exact Adam/replay settings are specified in Supplementary Material S3. Orders remain exogenous.

Predictions are inverse-transformed to bps before decoding: RAW uses $\widehat{\mathit{is}}=s\hat r-\hat y$, and ATTR/RAW-IS use $\hat y=s\hat r-\widehat{\mathit{is}}$. Neither decoder uses realized $r$. The common primary loss is
\begin{equation*}
L_i=\tfrac12\left\{\frac{(\hat r_i-r_i)^2}{\sigma_r^2}+\frac{(\widehat{\mathit{is}}_i-\mathit{is}_i)^2}{\sigma_{\mathit{is}}^2}\right\},
\label{eq:v3loss}
\end{equation*}
with common mature/submitted masks and prefix scales. Decoded $y$ error is a separate guardrail. Under uniform separate-task updates, direction trajectories must coincide across target representations: this is a structural null.

\subsubsection{Fixed-utility endogenous risk allocation}\label{sec:v4method}

V4 uses two synthetic assets and one shared futures account. Four causally calibrated $11\to 8\to 1$ D/E models are frozen for all modular policies. Current predictions, residual uncertainty, causal covariance, account state and observable regime form a 24-dimensional risk state. Forecasts constrain each position's sign; the actor chooses one of 16 ordered magnitude pairs in $\{0,0.10,0.25,0.50\}^2$.

RL-SCALAR and RL-RISK share a 672-parameter actor, a 468-parameter four-output critic, initial tensors, action uniforms, utility and projection. Let $r_t=V_{t+1}/V_t-1$ use equity after all transition costs, $g_t=10^4r_t$, $d_t$ the sum of positive event-level drawdown increments in bps, $\tau_t$ filled benchmark notional divided by starting equity, and $v_t=w_t^\top\widehat\Sigma_tw_t$ estimated post-fill exposure variance in bps squared. The frozen utility components are
\begin{equation*}
\bm c_t=(g_t,-0.10d_t,-2.0\tau_t,-0.50v_t/20),\qquad C_t=\mathbf{1}^{\top}\bm c_t.
\label{eq:v4utility}
\end{equation*}
The last three terms are bps-equivalent preferences, not extra ledger debits. Fees and execution costs already affect $g_t$.

Writing $V_j(x)$ for critic head $j$ and $b_t=0$ at a phase terminal/equity stop and 1 otherwise, the component TD residuals and alternative critic losses are
\begin{align}
\delta_{t,j}&=c_{t,j}/10+0.99b_tV_j(x_{t+1})-V_j(x_t), \nonumber\\
\ell_{\mathrm{RISK}}&=\tfrac12\sum_j\delta_{t,j}^2,\qquad
\ell_{\mathrm{SCALAR}}=\tfrac12\left(\sum_j\delta_{t,j}\right)^2.
\label{eq:v4critic}
\end{align}
The scalar implementation receives only $C_t$, not individual component targets; the expression above explains its equivalent aggregate residual. Next-state values are held constant for the gradient. Both actor losses are $-\log\pi(a_t\mid x_t)A_t$, with $A_t$ the sum of component residuals or the scalar residual respectively. The treatment changes critic gradient geometry and later policy paths.

Separate actor/critic Adam optimizers use learning rate 0.0003, $(\beta_1,\beta_2)=(0.9,0.999)$, $\epsilon=10^{-8}$ and global gradient-norm cap 5. There is one update per eligible completed transition, including during test, with no replay, target network, entropy bonus, advantage normalization or additional epochs. This is a one-step on-policy actor--critic, not SAC. Primary exposure, turnover and participation limits, projection rules and emergency-close conventions are frozen in Supplementary Material S4; these simulator interventions are explicit audit devices, not live-market safeguards.

FIXED requests 25\% per asset, and VOL scales this request inversely with estimated volatility. Descriptive PCGrad modifies conflicting component-critic gradients. JOINT-DRE uses a shared trunk with D/E, actor and component-critic heads and continues auxiliary D/E training after calibration. Its contrast with RL-RISK therefore changes representation, auxiliary supervision and optimization jointly; it is a bundled system comparison, not an isolation of modularity or critic decomposition.

\subsection{Frozen Evaluation and Statistical Interpretation}\label{sec:experiments}

The studies were designed sequentially. Each protocol was frozen locally before its own comparative outcomes; this is local prospective provenance, not external preregistration. Conditions, seeds, costs, endpoint definitions and hypotheses were not replaced after unfavorable results. Exact generators, prospective hashes, seed lists and matrix definitions accompany the package; detailed v2/v3/v4 protocols are in Supplementary Material S2--S4.

\subsubsection{Representation endpoints and hypotheses}\label{sec:v3protocol}

V3 has 20 independent path/initialization pairs, four conditions (stationary, direction shift, execution shift, both shifts), 800 main and 480 sensitivity cells. Each 4,096-clock path uses a 512-clock normalization prefix, training through 1,536 and diagnostic validation through 2,048. The primary calendar is $[2048,4080)$, restricted to submitted origins whose complete records mature by run end. Evaluation is prequential: legitimate earlier test labels can update the fixed learning algorithm, but an origin cannot learn from its own future label before prediction.

Target, routing and interaction contrasts on common $L$ are
\begin{align*}
\Delta_T&=L_{\mathrm{ATTR-U}}-L_{\mathrm{RAW-U}},\\
\Delta_R&=L_{\mathrm{ATTR-R}}-L_{\mathrm{ATTR-U}},\\
\Delta_I&=\Delta_R-(L_{\mathrm{RAW-R}}-L_{\mathrm{RAW-U}}).
\end{align*}
Negative target/routing contrasts favor the first treatment; interaction is two-sided. Twelve condition-by-contrast $L$ intervals form one Bonferroni family using nominal 99.5833\% intervals. A separate 24-comparison guardrail family covers target/routing contrasts on standardized $r,\mathit{is},y$ MSE using nominal 99.7917\% intervals. The noninferiority tolerance is 0.05 standardized squared-error units.

V3-H1 requires adjusted upper bounds below zero for target contrasts on $L$ and $\mathit{is}$, and below 0.05 for $r,y$. V3-H2 requires routing $L$ improvement and all three component guardrails. V3-H3 requires the adjusted $\Delta_I$ interval to exclude zero. V3-H4 separately tests residual dominance during the first 64 origins of each pure shift. Residual dominance characterizes response to the disclosed interventions; it is not a probability of correct causal localization.

Twenty thousand paired bootstrap replicates draw seed indices and circular clock blocks, preserving method/condition pairing and missing masks. Sensitivities, shared-network controls and recovery diagnostics are descriptive, not extra confirmatory families.

\subsubsection{Portfolio endpoints and hypotheses}\label{sec:v4protocol}

V4 was frozen before implementation and comparative execution. Thirty paired seed/path-initialization IDs supply seven conditions: stationary, direction shift, execution shift, risk shift, direction+risk, execution+risk and full shift. Each 12,288-clock path uses normalization, D/E calibration, risk training, validation and a 3,072-clock test. Conditions share the stationary pre-test path; test shifts alternate in 768-clock blocks. Condition identifiers and future draws are not learner inputs.

Portfolio cash/inventory reset at each risk-training, validation and test phase, while permitted learned state and causal market history carry forward. Test clocks, including flat/absorbing clocks and final forced-close costs, all enter the endpoints. With fractional simple returns,
\begin{equation*}
\mathrm{CE}_{\mathrm{bps/clock}}=10^4\!\left[\bar r-\frac{100}{2}\operatorname{Var}_0(r)\right],
\end{equation*}
where $\operatorname{Var}_0$ uses population variance. CE is a mean--variance functional, not realized net return or mean training utility. Net return compounds the full path. Maximum drawdown uses the ordered event-level equity sequence including initial capital, intermediate fills, emergency actions and terminal closes.

The registry contains 2,940 cells. Complete-path outcomes are computed before resampling. Twenty thousand paired bootstrap draws resample 30 seed IDs jointly across all conditions and endpoints. Six unique contrasts use nominal 99.1667\% Bonferroni-adjusted intervals. Conditions within a seed are averaged equally; the risk aggregate consists of the four risk-containing conditions.

V4-H1 requires all-condition CE superiority over FIXED. V4-H2 requires risk-condition CE superiority over RL-SCALAR and stationary CE noninferiority at 0.05 bps/clock. V4-H3 requires lower risk-condition MDD, CE noninferiority at 0.05 bps/clock and net-return noninferiority at 0.01 absolute return. V4-H4 requires all-condition CE superiority over the specified JOINT-DRE bundle. Contrasts are RL-RISK minus comparator: superiority requires lower bound $>0$ for CE or upper bound $<0$ for MDD; noninferiority requires lower bound greater than the negative margin.

\paragraph{Limits of inference.}
Intervals quantify variation under the disclosed seed/path-generating experiments. The confidence levels are nominal bootstrap levels with prespecified multiplicity families, not proven finite-sample coverage. V2's six synthetic seeds and Bar's three initialization seeds particularly limit precision. Twenty thousand draws improve numerical endpoint resolution but add no independent experiments. Failure to establish improvement is not equality; failed noninferiority is not proof of inferiority. Exact equality controls are identified separately in the Results subsections.

\subsection{Results}\label{sec:results}

\subsubsection{Ordinary supervision explains the scheduler equality}\label{sec:learningresults}

All 264 v2 main/sensitivity runs completed with finite changed weights. Attributed and ordinary loss-budget methods agree exactly in all 24 main simulated pairs; they also agree in all 12 main historical pairs under the fixed policy. This directly verifies the predicted configuration-specific equality. Residual routing changes allocation relative to uniform updates; the lower-combined-loss criterion fails in every primary simulated condition. Stationary loss increases by 0.0133, adjusted 98.75\% interval $[0.0048,0.0225]$; the both-shift contrast is $-0.0231$ with $[-0.1341,0.1099]$. None of the three shifted component-improvement/noninferiority criteria passes. Full contrasts and sensitivities remain in Supplementary Material S2.

\subsubsection{Better component predictions need not preserve decoded outcomes}\label{sec:v3results}

All 800 main and 480 sensitivity v3 cells completed with 20 seeds, no replacements and no missing/failed cells. The 200 prepared configuration/path IDs are not 200 independent seeds. Common-mask scoring has 32,479 submitted origins per condition/variant, with equal-seed weighting rather than pooled-clock weighting.

Direct shortfall supervision lowers common $L$ in every condition (Figure~\ref{fig:v3contrasts}). Target contrasts are $-0.3250$, $-0.5804$, $-0.2697$ and $-0.4130$ for the stationary, direction-shift, execution-shift and both-shift conditions, respectively, with all adjusted 99.5833\% intervals excluding zero. The both-shift interval is $[-0.5202,-0.3289]$; paired seed-standardized effect size is $-3.42$. Uniform direction is exactly equal, so this common-loss improvement comes from shortfall prediction. Both-shift shortfall MAE falls from 5.2860 to 4.4028 bps, while decoded $y$ MAE rises from 6.2014 to 6.4320 bps.

\begin{figure}[tb]
\centering
\includegraphics[width=\linewidth]{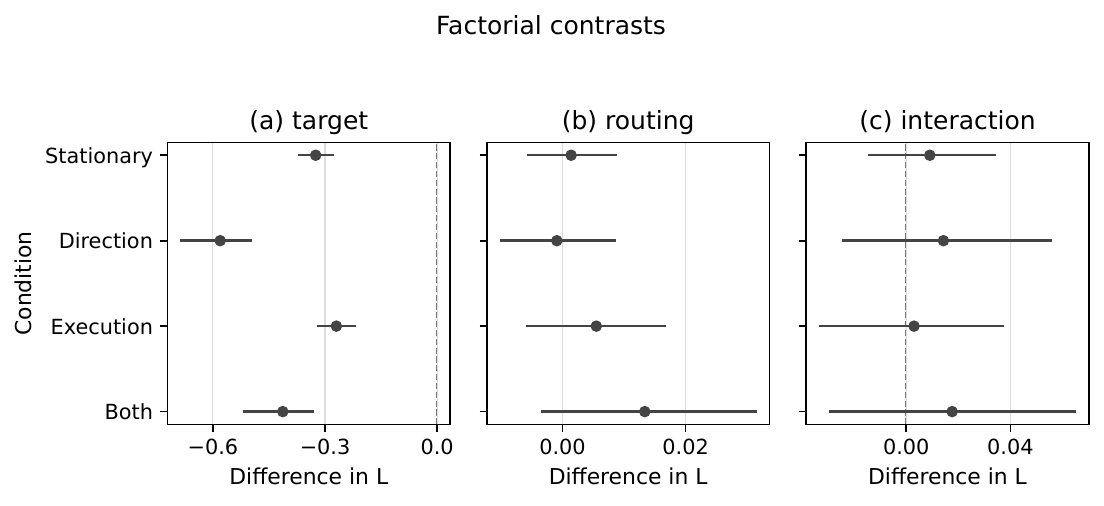}
\caption{V3 factorial contrasts on common $L$, with paired 99.5833\% Bonferroni-adjusted intervals. Negative target/routing differences favor the first treatment; interaction is two-sided. The component gain does not imply that the separate outcome guardrail passes.}
\label{fig:v3contrasts}
\end{figure}

The complete V3-H1 criterion fails in every condition because decoded-$y$ MSE noninferiority is unestablished (Table~\ref{tab:v3guardrail}). Under direction shift the point estimate is favorable, but the prespecified noninferiority criterion is not met. The other three adjusted intervals exclude zero adversely. Native execution losses alone would conceal this trade-off.

\begin{table}[tb]
\centering
\setlength{\tabcolsep}{3pt}
\caption{V3 target contrast on decoded $y$ standardized MSE, with 99.7917\% Bonferroni-adjusted intervals. Passing requires the upper bound below $0.05$.}
\label{tab:v3guardrail}
\small
\begin{tabular}{@{}lrrl@{}}
\toprule
Condition & Difference & Adjusted interval & Noninferiority\tabularnewline
\midrule
Stationary & $+0.0395$ & $[0.0083,0.0710]$ & Not established\tabularnewline
Direction shift & $-0.0110$ & $[-0.0851,0.0630]$ & Not established\tabularnewline
Execution shift & $+0.0713$ & $[0.0468,0.0973]$ & Not established\tabularnewline
Both shifts & $+0.0744$ & $[0.0069,0.1421]$ & Not established\tabularnewline
\bottomrule
\end{tabular}
\end{table}

Neither routing conditional on ATTR targets (V3-H2) nor target-by-routing interaction (V3-H3) is confirmed. Both-shift routing is $+0.0134$, adjusted interval $[-0.0034,0.0315]$; its descriptive 95\% interval $[0.0017,0.0254]$ is adverse. The interaction is $+0.0176$ with adjusted interval $[-0.0297,0.0652]$. Shared/J/P means are not promoted to confirmatory winners.

Ordinary RAW-IS-U/R independently reproduce ATTR-U/R in all 160 main pairs, including predictions, weights and optimizer/replay/RNG state. Uniform direction also coincides in 80 main and 120 sensitivity pairs. This equality locates the apparent benefit in finite-learner target representation, not information unique to attribution. Both V3-H4 residual-dominance tests pass (direction 0.1440, execution 0.6199), but shortfall cross-loads directional opportunity through non-fill; detailed paired components and censored recovery records are retained in Supplementary Material S3.

\subsubsection{Component supervision worsens risk-condition outcomes}\label{sec:v4results}

All 2,940 v4 cells completed with 30 seeds and no missing, failed or replacement cell. RL-RISK is worse than FIXED on all-condition CE ($-0.2817$ bps/clock) and worse than RL-SCALAR on risk-condition CE ($-0.0619$). Their adjusted intervals are wholly adverse (Table~\ref{tab:v4contrasts}; Figure~\ref{fig:v4contrasts}). Risk-condition MDD increases by 0.0151 of equity; net-return noninferiority at one percentage point fails. Stationary CE has interval $[-0.0545,0.0013]$, crossing both zero and the $-0.05$ lower NI margin. This establishes neither stationary inferiority nor stationary noninferiority.

\begin{table}[tb]
\centering
\setlength{\tabcolsep}{3pt}
\caption{V4 confirmatory paired contrasts with 99.1667\% Bonferroni-adjusted intervals. CE is bps/clock; MDD and net return are fractions. Positive CE/return and negative MDD favor RL-RISK.}
\label{tab:v4contrasts}
\small
\begin{tabular}{@{}p{0.27\linewidth}rrp{0.13\linewidth}@{}}
\toprule
\raggedright Contrast (RL-RISK minus comparator) & Difference & Adjusted interval & \raggedright Decision\tabularnewline
\midrule
\raggedright All-condition CE minus FIXED & $-0.2817$ & $[-0.3186,-0.2285]$ & \raggedright Not supported\tabularnewline
\raggedright Risk-condition CE minus RL-SCALAR & $-0.0619$ & $[-0.1132,-0.0241]$ & \raggedright Not supported\tabularnewline
\raggedright Stationary CE minus RL-SCALAR & $-0.0215$ & $[-0.0545,0.0013]$ & \raggedright NI not established\tabularnewline
\raggedright Risk-condition MDD minus RL-SCALAR & $+0.0151$ & $[0.0053,0.0286]$ & \raggedright Not supported\tabularnewline
\raggedright Risk-condition net return minus RL-SCALAR & $-0.0145$ & $[-0.0280,-0.0048]$ & \raggedright NI not established\tabularnewline
\raggedright All-condition CE minus JOINT-DRE & $+0.0597$ & $[0.0187,0.1033]$ & \raggedright Supported\tabularnewline
\bottomrule
\end{tabular}
\end{table}

\begin{figure}[tb]
\centering
\includegraphics[width=\linewidth]{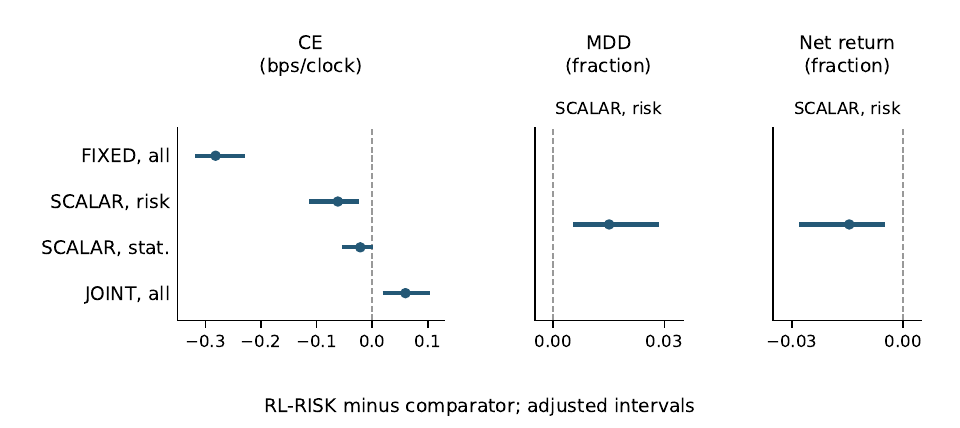}
\caption{V4 prespecified paired contrasts and adjusted intervals. SCALAR and JOINT abbreviate RL-SCALAR and JOINT-DRE; all, risk and stat. denote aggregates. Exact intervals are reported in Table~\ref{tab:v4contrasts} and Supplementary Material S4. The plot marks zero, not NI margins.}
\label{fig:v4contrasts}
\end{figure}

V4-H1--H3 fail their prespecified criteria. The adverse risk-condition CE and MDD intervals identify a deterioration under component-wise supervision in this frozen controller. This is an informative supervision trade-off: the same total utility and matched critic architecture yield different policy-level risk outcomes when the critic is trained on individual components. Section~\ref{sec:criticinterpretation} develops its methodological interpretation.

RL-RISK exceeds JOINT-DRE by 0.0597 bps/clock, adjusted interval $[0.0187,0.1033]$, supporting V4-H4 as a bundled system comparison only. JOINT-DRE changes shared representation, auxiliary D/E training and optimization; its lower CE cannot be attributed solely to modularity, gradient conflict or critic decomposition. Endogenous differences in actions, fills, visited states and subsequent learning records are part of these policy-level treatments.

\begin{table}[tb]
\centering\small
\setlength{\tabcolsep}{4pt}
\caption{Claim hierarchy for the executed studies. Decisions follow the complete prespecified criteria, including downstream guardrails. Original full-system H1--H4 remain separate evaluation questions in Supplementary Material S6.}
\label{tab:evidence}
\begin{tabular}{@{}p{0.19\linewidth}p{0.39\linewidth}p{0.33\linewidth}@{}}
\toprule
\raggedright Claim & \raggedright Evidence & \raggedright Conclusion\tabularnewline
\midrule
\raggedright Scheduler specificity & \raggedright Exact attributed/loss-budget equality in 24 simulated and 12 historical main pairs & \raggedright No attribution-specific advantage for this interface\tabularnewline[0.3em]
\raggedright V3-H1 & \raggedright $L$/IS improve; decoded-$y$ NI fails in all 4 conditions & \raggedright Complete criterion not supported\tabularnewline[0.3em]
\raggedright V3-H2/H3 & \raggedright Adjusted routing and interaction criteria fail in all 4 conditions & \raggedright Neither gain confirmed\tabularnewline[0.3em]
\raggedright V3-H4 & \raggedright Both specified residual-dominance intervals positive & \raggedright Supported dominance, not causal diagnosis\tabularnewline[0.3em]
\raggedright V4-H1/H2/H3 & \raggedright Adverse aggregate CE/MDD; return guardrails fail & \raggedright Complete risk criteria not supported\tabularnewline[0.3em]
\raggedright V4-H4 & \raggedright Positive all-condition CE contrast against JOINT-DRE & \raggedright Supported only for the specified bundle\tabularnewline
\bottomrule
\end{tabular}
\end{table}

\subsubsection{Secondary checks and historical probe}

All prespecified sensitivity cells are retained in Supplementary Material and the numerical package. V3 routing sensitivities fail the improvement criterion, reinforcing the need to assess target gains together with outcome guardrails. V4 state-removal/projection and PCGrad results are descriptive. Recovery comparisons retain their censoring information, and the ablation inventory distinguishes executed interventions from future full-system removals.

The Bar replay asks whether the v2 scheduler equality persists once a fixed policy makes execution-label selection endogenous. It does. The pooled attributed-minus-uniform daily-return contrast is $+3.90$ bps/day, descriptive 95\% interval $[-0.90,9.96]$, conditional on the selected windows. Both methods lose money on average in each year while passive exposure is positive. This is a secondary portability/selection probe. Full negative outcomes, costs and data limitations are in Supplementary Material S5.

\section{Historical Engineering Case}\label{sec:legacycase}

\subsection{Recovered report and observed configuration boundary}

A preserved project report, here designated E01, matches the metrics in the recovered backtest screenshot. Its source identifier is \path{BTC_20210101_20241230_20251208_110347}; the retained \path{backtest_report.html}, \path{backtest_equity_curve.csv} and trade CSV are bound by the local source manifest. Table~\ref{tab:legacyreported} transcribes point estimates reported by that HTML rather than recalculating performance under new definitions.

\begin{table}[htbp]
\centering\small
\caption{E01 historical report: transcribed point estimates, not independently reproduced performance estimates. Currency, annualization and complete accounting/configuration provenance remain unverified.}
\label{tab:legacyreported}
\begin{tabular}{@{}lr@{}}
\toprule
Reported measure & Reported value\tabularnewline
\midrule
Sharpe & 2.4356\tabularnewline
Win rate & 55.62\%\tabularnewline
Maximum drawdown, magnitude & 16.81\%\tabularnewline
CAGR & 136.92\%\tabularnewline
Sortino & 1.0442\tabularnewline
Profit factor & 3.3163\tabularnewline
Final equity, reported account units & 3\,135\,734.668\tabularnewline
Total change displayed in the report & $+3035.73\%$\tabularnewline
\bottomrule
\end{tabular}
\end{table}

The canonical trade file contains 347 records: 262 tagged momentum and 85 spread, with asset tags BTC (143), ETH (119) and BTC-ETH (85). The \path{action_source} field is \path{dual_engine} in 262 rows and blank in 85; \path{config_id} is blank in every row. Module participation and configuration identity remain unresolved, placing E01 in the historical-engineering evidence class defined in Section~\ref{sec:scope}.

\subsection{Unchanged equity series and source limitations}

Figure~\ref{fig:legacyequity} plots all 35,251 CSV records in their original order. The series begins on 22 December 2020 at 06:00 UTC with equity 100,000 and ends on 30 December 2024 at 00:00 UTC with equity 3,135,734.668015859. The actual dates are retained even though the run identifier denotes 2021--2024. No early observations are removed as presumed warmup, and no smoothing, resampling or synthetic extension is applied.

\begin{figure}[tb]
\centering
\includegraphics[width=\linewidth]{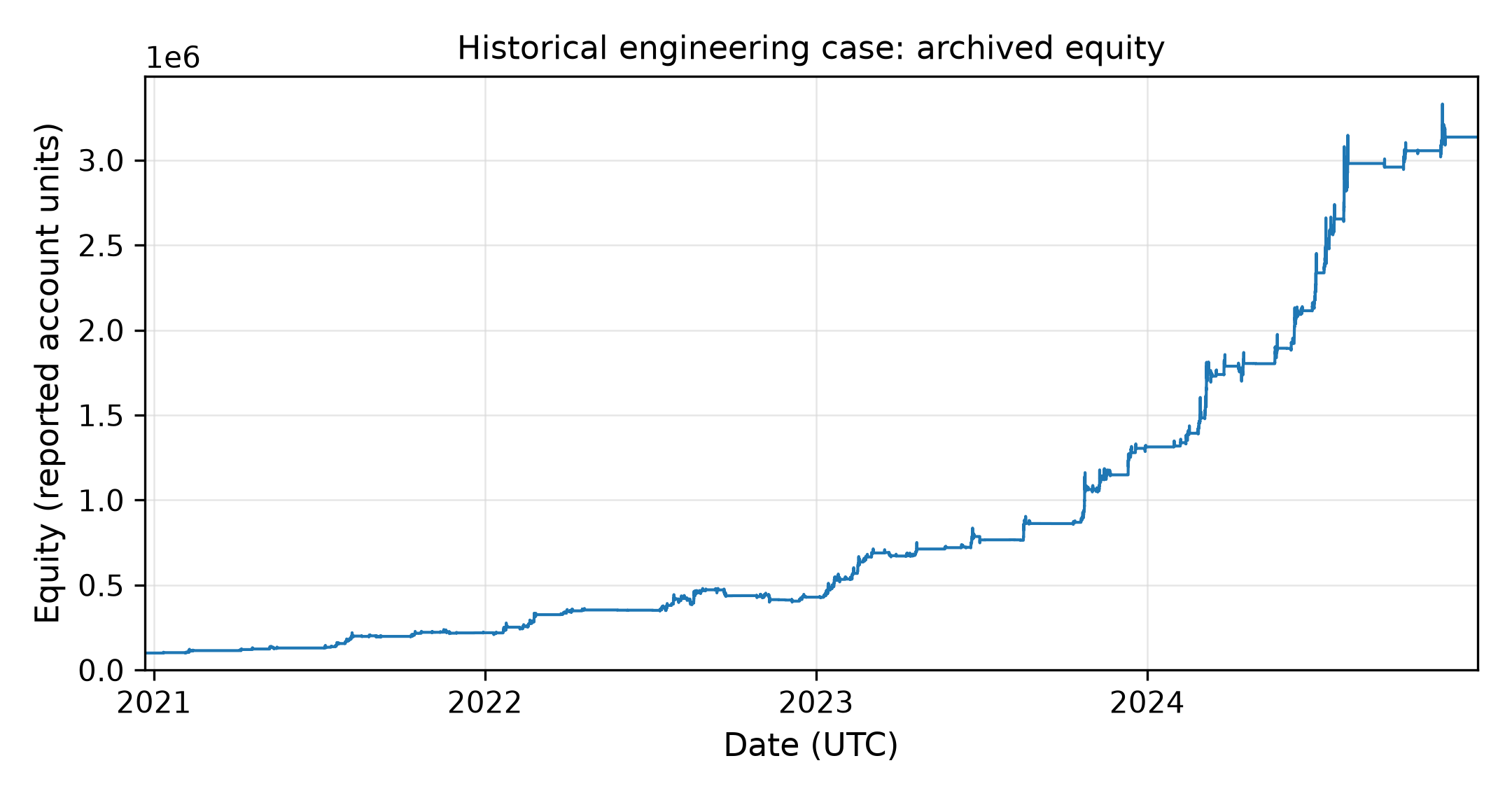}
\caption{E01 equity redrawn directly from the preserved CSV, including the December 2020 observations. Units are the account units reported in the source. The line preserves every timestamp/value pair. Configuration and accounting provenance are discussed in Section~\ref{sec:scope}.}
\label{fig:legacyequity}
\end{figure}

The preserved source set contains byte-identical equity CSVs across four reports and largely identical trade files; these records are treated as one historical outcome, not independent replications. Some source ranges are labeled as 95\% confidence intervals, but the exact generator revision and interval-construction code are not fully recovered. Accordingly, no historical interval is promoted to validated statistical uncertainty in Table~\ref{tab:legacyreported}. Full source hashes, duplicate relationships and diagnostic matrices are recorded in Supplementary Material S7.

E01 therefore provides a provenance lesson rather than a performance claim: the surviving archive describes portfolio outcomes more completely than the module-level process that generated them.

\subsection{Lessons Learned in Production Engineering}\label{sec:engineeringlessons}

E01 makes the distinction between performance reporting and learning-system traceability concrete. The surviving archive connects a report to equity and trade records, but its empty configuration identifiers leave the decision-to-model join incomplete. The retrospective evaluation bottleneck is therefore identifiable: the retained records describe portfolio outcomes more completely than the module-level process that generated them.

The practical lesson is to design the evaluation record alongside the decision pipeline. A reusable run record binds configuration and model versions to decision identifiers; retains the available inputs, module outputs and label-maturity times; and links proposed intents, risk-approved targets, actual fills, fees and portfolio valuations. These links support the three controls developed here. Input and maturity records define equal information; model and update identities enable trained-state comparisons; approved actions and realized account events support downstream guardrails. Counterfactual policy comparisons additionally require an explicit experimental or modeling design.

This connects the historical case to the controlled studies through an engineering requirement: an inspectable portfolio outcome and an identifiable learning intervention need different records. Preserving both makes it possible to distinguish an economically successful run from evidence about which learning mechanism contributed to it.

\section{Discussion}\label{sec:discussion}

\subsection{What the controls reveal}

The framework separates three questions that can otherwise be conflated: whether a record is economically interpretable, whether its targets add information and whether its learning updates improve decisions. V2 establishes an exact allocation equality under the disclosed interface. V3 shows that an invertible target coordinate changes finite-network prediction errors while an ordinary reconstruction control reproduces the attributed learner. The RAW-IS equality locates the effect in target representation. The decoded-outcome guardrail then reveals the cost of evaluating that representation only through its native loss.

V4 extends the evaluation to endogenous decisions under a fixed utility. Component supervision produces adverse risk-condition CE and drawdown outcomes relative to scalar supervision. Together, these results make equal-information controls and downstream checks central to evaluating modular credit assignment. AMATS provides the module responsibilities and record interfaces; the historical case shows why those interfaces need persistent run-level provenance. The same caution applies to market feedback more generally: statistical detectability does not by itself establish economic exploitability after transaction costs and finite-agent learning constraints \citep{Kucukoglu2026ESWA}. Equal-information controls therefore test whether an attributed signal can improve decisions in a deployed expert system, not merely whether it is predictable under a chosen loss.

\subsection{Interpreting component-wise critic supervision}\label{sec:criticinterpretation}

The V4 comparison identifies an important distinction between interpretable supervision and decision-aligned learning. Both treatments use the same four-output critic, actor and total utility. Their losses in Equation~\eqref{eq:v4critic} impose different fitting demands: component supervision penalizes every TD residual, whereas scalar supervision fits their aggregate. Component prediction contrasts that cancel in the sum can therefore influence shared features under the component objective while leaving the aggregate objective unchanged. This gives a concrete mechanism through which extra supervision can change learning without adding decision-relevant information.

A plausible hypothesis is that, under noisy or unevenly predictable component targets, fitting these contrasts draws the shared representation toward unstable correlations. Aggregate-only supervision may then act as a form of objective-level regularization by focusing fitting pressure on the total utility used by the actor. The analogy is to negative transfer from unreliable tasks in multi-task learning \citep{Lee2018AMTFL} and to interference among shared task gradients \citep{SenerKoltun2018,Yu2020PCGrad}. Here regularization refers to the supervised objective: parameter counts and the four-output architecture are matched.

The measured conclusion is the adverse policy-level supervision contrast. Noise-induced overfitting is an interpretation hypothesis, alongside gradient scaling, optimizer geometry and the endogenous state distributions produced by different policy updates. Distinguishing these mechanisms would require matched gradient and prediction diagnostics across controlled noise or target-predictability settings. The practical implication already supported by V4 is clear: the value of component supervision must be judged against the aggregate objective and prespecified risk outcomes, even when the components themselves have an intuitive economic meaning.

\subsection{Guidelines for evaluating modular expert systems}\label{sec:guidelines}

The controlled studies suggest the following reusable evaluation discipline for modular financial decision systems.

\begin{enumerate}
  \item Define the decision-time information set before constructing attributed labels.
  \item Compare every attributed learner with an ordinary learner receiving equivalent reconstructible information.
  \item Evaluate decoded economic outcomes, not only native component losses.
  \item Fix total utility when comparing component-wise and aggregate supervision.
  \item Preserve run-level provenance linking configuration, action, fill, label maturity and evaluation endpoint.
  \item Treat historical profitable runs as engineering evidence, not as proof of a specific learning mechanism.
\end{enumerate}

These guidelines are intentionally conservative: they are designed to prevent an expert-system developer from mistaking a more interpretable signal, a lower component loss or a favorable historical equity curve for evidence of decision-level improvement.

\section{Limitations and Scope}\label{sec:scope}

\paragraph{Controlled evidence.}
The established results concern the specified accounting identities, configuration-specific trained equalities, finite-learner representation trade-offs and fixed-utility policy comparisons. V3 uses exogenous orders, benchmark-relative fills and one small learner family. V4 uses two synthetic assets, current-mid execution, guaranteed emergency fills, no funding or order book, and one utility coefficient set. Its covariance/exposure preferences and CE coefficients are engineering choices. Evaluation is prequential under a fixed causal schedule; each protocol freeze preceded that study's own comparative outcomes within a sequential research program. Bootstrap intervals quantify variation within the specified seed/path experiments and use nominal levels with separate multiplicity families.

\paragraph{Historical and system evidence.}
Bar covers two assets and two selected windows,    with prior familiarity of the investigators with the archive, modeled taker fills, an hourly-open funding proxy and no L2 or margin/liquidation model. E01 is a separately preserved report and equity series. Its account currency, deposits, leverage/sizing, fee/funding/slippage coverage, fill conventions, annualization and complete module participation remain unresolved. The configuration and telemetry gap is observed in the retained artifacts; whether the original system omitted those records or they were lost during preservation remains unknown. Original reported ranges and unlinked screenshots retain their separate provenance. The full AMATS architecture, preliminary parameters and authored LLM examples define interfaces for future experiments; matched full-stack performance, live execution, operational latency and adaptation quality require additional evidence. The original integration of NLP risk gating and broader LLM intents is also unresolved.

\paragraph{Identification and transfer.}
Accounting reconstruction checks signed events using the saved marks. Its validation scope is ledger consistency; venue-price validation, counterfactual impact and causal module contribution require separate evidence. Changing a benchmark can alter attribution without changing cash, and non-fill opportunity can transmit direction changes into shortfall. Observed regime features remain distinct from analyst intervention labels. The equality results apply under their stated state and update contracts; broader markets, learner families and full-system ablations are future evaluation domains. The JOINT-DRE comparison bundles representation, auxiliary training and optimization. The MTL interpretation motivates a mechanism test, while the existing study establishes the supervision/outcome contrast.

\section{Ethics and practical risk disclosure}\label{sec:ethics}

This paper is research documentation and does not provide investment advice, a recommendation to trade, or an assurance of returns. Trading perpetual derivatives can result in the loss of posted capital. Leverage amplifies adverse price moves and can cause liquidation before a forecast horizon ends; depending on the product and account terms, losses can exceed the initial deposit \citep{CFTC2017Risk}. Investment suitability requires an assessment independent of historical backtest performance. Fees, funding, slippage, outages and changing liquidity can materially alter realized losses. Actual exit availability and prices depend on venue conditions, including when the research simulator guarantees emergency fills.

Any operational evaluation would require independently tested account limits, reconciliation, operator oversight and a way to suspend new risk, with exit handling assessed under actual liquidity and connectivity constraints. LLM confidence and fluent rationales can be incorrect; rationales document supplied evidence rather than certify the cause of a decision or its expected profitability. Untrusted market text must not acquire authority over system instructions or account controls. Prompts and adaptation corpora should exclude credentials and identifying account information, and use only data with appropriate access and redistribution permission. Candidate adapters require review for temporal leakage, selection bias and behavior under missing inputs before promotion. The illustrative prompts and preliminary parameters in Supplementary Material S9--S10 provide an interface specification for that prospective validation.

\section{Conclusion}

We presented a methodological framework for rigorous credit-assignment evaluation in modular financial RL, with AMATS as the case study. Equal-information reconstruction controls identify exact learning equivalences and separate target-representation effects from added information. Prespecified downstream guardrails expose prediction trade-offs and adverse risk outcomes that component objectives alone conceal. The controlled comparisons establish that residual allocation can coincide with ordinary loss scheduling, shortfall supervision can improve component prediction while failing decoded-outcome preservation, and component critic supervision can worsen risk-condition outcomes under fixed utility.

The AMATS architecture connects these methodological results to an explicit decision and feedback pipeline. Its historical engineering record provides complementary lessons about the provenance needed to evaluate a learning intervention. The combined contribution is a disciplined link from economic records to matched learning comparisons and decision-level evidence. Future full-system evaluations can apply the same contracts to named forecasting ensembles, language-model adaptation, SAC risk allocation and venue execution while retaining transparent hypotheses, comparable information and explicit risk criteria.

%% file: references.bib
@article{Perold1988,
  author  = {Perold, Andr{\'e} F.},
  title   = {The Implementation Shortfall: Paper versus Reality},
  journal = {The Journal of Portfolio Management},
  volume  = {14},
  number  = {3},
  pages   = {4--9},
  year    = {1988},
  doi     = {10.3905/jpm.1988.409150},
  url     = {https://doi.org/10.3905/jpm.1988.409150}
}

@article{AlmgrenChriss,
  author  = {Almgren, Robert and Chriss, Neil},
  title   = {Optimal Execution of Portfolio Transactions},
  journal = {Journal of Risk},
  volume  = {3},
  number  = {2},
  pages   = {5--39},
  year    = {2001},
  doi     = {10.21314/JOR.2001.041},
  url     = {https://doi.org/10.21314/JOR.2001.041}
}

@misc{DonierBonart,
  author        = {Donier, Marc and Bonart, J{\'e}r{\'e}mie},
  title         = {A Million Metaorder Analysis of Market Impact on the Bitcoin},
  year          = {2015},
  eprint        = {1412.4503},
  archivePrefix = {arXiv},
  url           = {https://arxiv.org/abs/1412.4503v3},
  note          = {arXiv:1412.4503v3; first posted 2014}
}

@article{DeepLOB,
  author  = {Zhang, Zihao and Zohren, Stefan and Roberts, Stephen},
  title   = {{DeepLOB}: Deep Convolutional Neural Networks for Limit Order Books},
  journal = {IEEE Transactions on Signal Processing},
  volume  = {67},
  number  = {11},
  pages   = {3001--3012},
  year    = {2019},
  doi     = {10.1109/TSP.2019.2907260},
  url     = {https://arxiv.org/abs/1808.03668v6},
  note    = {Author version: arXiv:1808.03668v6}
}

@inproceedings{SenerKoltun2018,
  author    = {Sener, Ozan and Koltun, Vladlen},
  title     = {Multi-Task Learning as Multi-Objective Optimization},
  booktitle = {Advances in Neural Information Processing Systems},
  volume    = {31},
  year      = {2018},
  url       = {https://proceedings.neurips.cc/paper/2018/hash/432aca3a1e345e339f35a30c8f65edce-Abstract.html}
}

@inproceedings{Yu2020PCGrad,
  author    = {Yu, Tianhe and Kumar, Saurabh and Gupta, Abhishek and Levine, Sergey and Hausman, Karol and Finn, Chelsea},
  title     = {Gradient Surgery for Multi-Task Learning},
  booktitle = {Advances in Neural Information Processing Systems},
  volume    = {33},
  year      = {2020},
  url       = {https://proceedings.neurips.cc/paper/2020/hash/3fe78a8acf5fda99de95303940a2420c-Abstract.html}
}

@inproceedings{HRA2017,
  author    = {van Seijen, Harm and Fatemi, Mehdi and Romoff, Joshua and Laroche, R{\'e}mi and Barnes, Trevor and Tsang, Jonathan},
  title     = {Hybrid Reward Architecture for Reinforcement Learning},
  booktitle = {Advances in Neural Information Processing Systems},
  volume    = {30},
  year      = {2017},
  url       = {https://proceedings.neurips.cc/paper_files/paper/2017/hash/1264a061d82a2edae1574b07249800d6-Abstract.html}
}

@misc{Saidd2026,
  author       = {Saidd, N. A.},
  title        = {Decomposable Reward Modeling and Realistic Environment Design for Reinforcement Learning-Based {Forex} Trading},
  year         = {2026},
  howpublished = {\url{https://www.preprints.org/manuscript/202603.1701}},
  note         = {Preprints.org, version 1, posted 23 March 2026; not peer reviewed},
  url          = {https://www.preprints.org/manuscript/202603.1701}
}

@misc{FinRL,
  author        = {Liu, Xiao-Yang and Yang, Hongyang and Chen, Qian and Zhang, Ran and Yang, Lingfei and Xiao, Bin and Wang, Christina Dan},
  title         = {{FinRL}: A Deep Reinforcement Learning Library for Automated Stock Trading in Quantitative Finance},
  year          = {2020},
  eprint        = {2011.09607},
  archivePrefix = {arXiv},
  url           = {https://arxiv.org/abs/2011.09607v1},
  note          = {arXiv:2011.09607v1}
}

@inproceedings{SAC,
  author    = {Haarnoja, Tuomas and Zhou, Aurick and Abbeel, Pieter and Levine, Sergey},
  title     = {Soft Actor-Critic: Off-Policy Maximum Entropy Deep Reinforcement Learning with a Stochastic Actor},
  booktitle = {Proceedings of the 35th International Conference on Machine Learning},
  volume    = {80},
  series    = {Proceedings of Machine Learning Research},
  publisher = {PMLR},
  pages     = {1861--1870},
  year      = {2018},
  url       = {https://proceedings.mlr.press/v80/haarnoja18b.html}
}

@misc{FinGPT,
  author        = {Yang, Hongyang and Liu, Xiao-Yang and Wang, Christina Dan},
  title         = {{FinGPT}: Open-Source Financial Large Language Models},
  year          = {2023},
  eprint        = {2306.06031},
  archivePrefix = {arXiv},
  url           = {https://arxiv.org/abs/2306.06031v1},
  note          = {arXiv:2306.06031v1}
}

@misc{LoRA2021,
  author        = {Hu, Edward J. and Shen, Yelong and Wallis, Phillip and Allen-Zhu, Zeyuan and Li, Yuanzhi and Wang, Shean and Wang, Lu and Chen, Weizhu},
  title         = {{LoRA}: Low-Rank Adaptation of Large Language Models},
  year          = {2021},
  eprint        = {2106.09685},
  archivePrefix = {arXiv},
  url           = {https://arxiv.org/abs/2106.09685v2},
  note          = {arXiv:2106.09685v2}
}

@misc{FinMem,
  author        = {Yu, Yang and Li, Hui and Chen, Zihao and Jiang, Yuansheng and Li, Yuanyuan and Zhang, Daniel and Liu, Ruan and Suchow, Jason W. and Khashanah, Khaled},
  title         = {{FinMem}: A Performance-Enhanced {LLM} Trading Agent with Layered Memory and Character Design},
  year          = {2023},
  eprint        = {2311.13743},
  archivePrefix = {arXiv},
  url           = {https://arxiv.org/abs/2311.13743v2},
  note          = {arXiv:2311.13743v2}
}

@misc{TradingGPT,
  author        = {Li, Xiao and Yu, Yang and Li, Hui and Chen, Zihao and Khashanah, Khaled},
  title         = {{TradingGPT}: Multi-Agent System with Layered Memory and Distinct Characters for Enhanced Financial Trading Performance},
  year          = {2023},
  eprint        = {2309.03736},
  archivePrefix = {arXiv},
  url           = {https://arxiv.org/abs/2309.03736v1},
  note          = {arXiv:2309.03736v1}
}

@misc{TradingAgents,
  author        = {Xiao, Yijia and Sun, Eddie and Luo, Die and Wang, William Yang},
  title         = {{TradingAgents}: Multi-Agents {LLM} Financial Trading Framework},
  year          = {2024},
  eprint        = {2412.20138},
  archivePrefix = {arXiv},
  url           = {https://arxiv.org/abs/2412.20138v1},
  note          = {arXiv:2412.20138v1}
}

@article{BatesGranger1969,
  author  = {Bates, John M. and Granger, Clive W. J.},
  title   = {The Combination of Forecasts},
  journal = {Journal of the Operational Research Society},
  volume  = {20},
  number  = {4},
  pages   = {451--468},
  year    = {1969},
  doi     = {10.1057/jors.1969.103},
  url     = {https://doi.org/10.1057/jors.1969.103}
}

@article{Hamilton1989,
  author  = {Hamilton, James D.},
  title   = {A New Approach to the Economic Analysis of Nonstationary Time Series and the Business Cycle},
  journal = {Econometrica},
  volume  = {57},
  number  = {2},
  pages   = {357--384},
  year    = {1989},
  doi     = {10.2307/1912559},
  url     = {https://doi.org/10.2307/1912559}
}

@inproceedings{nhits,
  author    = {Challu, Christoph and Olivares, Enrique G. and Oreshkin, Boris N. and Garza Ramirez, Federico and Mergenthaler Canseco, Maximiliano and Dubrawski, Artur},
  title     = {{NHITS}: Neural Hierarchical Interpolation for Time Series Forecasting},
  booktitle = {Proceedings of the AAAI Conference on Artificial Intelligence},
  volume    = {37},
  number    = {6},
  pages     = {6989--6997},
  year      = {2023},
  doi       = {10.1609/aaai.v37i6.25854},
  url       = {https://doi.org/10.1609/aaai.v37i6.25854}
}

@inproceedings{patchtst,
  author    = {Nie, Yuqi and Nguyen, Nam H. and Sinthong, Phanwadee and Kalagnanam, Jayant},
  title     = {A Time Series Is Worth 64 Words: Long-Term Forecasting with Transformers},
  booktitle = {International Conference on Learning Representations},
  year      = {2023},
  url       = {https://arxiv.org/abs/2211.14730},
  note      = {ICLR 2023; arXiv:2211.14730}
}

@misc{etsformer,
  author        = {Woo, Gerald and Liu, Chenghao and Sahoo, Doyen and Kumar, Avinash and Hoi, Steven},
  title         = {{ETSformer}: Exponential Smoothing Transformers for Time-Series Forecasting},
  year          = {2022},
  eprint        = {2202.01381},
  archivePrefix = {arXiv},
  url           = {https://arxiv.org/abs/2202.01381v2},
  note          = {arXiv:2202.01381v2}
}

@inproceedings{Guo2017,
  author    = {Guo, Chuan and Pleiss, Geoff and Sun, Yu and Weinberger, Kilian Q.},
  title     = {On Calibration of Modern Neural Networks},
  booktitle = {Proceedings of the 34th International Conference on Machine Learning},
  volume    = {70},
  series    = {Proceedings of Machine Learning Research},
  publisher = {PMLR},
  pages     = {1321--1330},
  year      = {2017},
  url       = {https://proceedings.mlr.press/v70/guo17a.html}
}

@inproceedings{TimesFM,
  author    = {Das, Abhimanyu and Kong, Weihao and Sen, Rajat and Zhou, Yuchen},
  title     = {A Decoder-Only Foundation Model for Time-Series Forecasting},
  booktitle = {Proceedings of the 41st International Conference on Machine Learning},
  volume    = {235},
  series    = {Proceedings of Machine Learning Research},
  publisher = {PMLR},
  pages     = {10148--10167},
  year      = {2024},
  url       = {https://proceedings.mlr.press/v235/das24c.html}
}

@inproceedings{iTransformer,
  author    = {Liu, Yong and Hu, Tengge and Zhang, Haoran and Wu, Jianmin and Wang, Shiyu and Ma, Lintao and Long, Mingsheng},
  title     = {{iTransformer}: Inverted Transformers Are Effective for Time Series Forecasting},
  booktitle = {International Conference on Learning Representations},
  year      = {2024},
  url       = {https://arxiv.org/abs/2310.06625},
  note      = {ICLR 2024; arXiv:2310.06625}
}

@article{TFT,
  author  = {Lim, Bryan and Ar{\i}k, S{\"o}nmez {\"O}zcan and Loeff, Nicolas and Pfister, Tilman},
  title   = {Temporal Fusion Transformers for Interpretable Multi-Horizon Time Series Forecasting},
  journal = {International Journal of Forecasting},
  volume  = {37},
  number  = {4},
  pages   = {1748--1764},
  year    = {2021},
  doi     = {10.1016/j.ijforecast.2021.03.012},
  url     = {https://doi.org/10.1016/j.ijforecast.2021.03.012}
}

@inproceedings{Lee2018AMTFL,
  author    = {Lee, Heewon and Yang, Eunho and Hwang, Sung Ju},
  title     = {Deep Asymmetric Multi-Task Feature Learning},
  booktitle = {Proceedings of the 35th International Conference on Machine Learning},
  volume    = {80},
  series    = {Proceedings of Machine Learning Research},
  publisher = {PMLR},
  pages     = {2956--2964},
  year      = {2018},
  url       = {https://proceedings.mlr.press/v80/lee18d.html}
}

@misc{Qwen25Card,
  author       = {{Qwen Team}},
  title        = {{Qwen2.5-7B-Instruct}: Official Model Card},
  year         = {2026},
  howpublished = {\url{https://huggingface.co/Qwen/Qwen2.5-7B-Instruct}},
  note         = {Accessed: 12 September 2026},
  url          = {https://huggingface.co/Qwen/Qwen2.5-7B-Instruct}
}

@misc{SB3SAC,
  author       = {{Stable-Baselines3 Contributors}},
  title        = {{SAC}: Algorithm and Parameter Documentation},
  year         = {2026},
  howpublished = {\url{https://stable-baselines3.readthedocs.io/en/master/modules/sac.html}},
  note         = {Accessed: 12 September 2026},
  url          = {https://stable-baselines3.readthedocs.io/en/master/modules/sac.html}
}

@misc{Bai2018TCN,
  author        = {Bai, Shaojie and Kolter, J. Zico and Koltun, Vladlen},
  title         = {An Empirical Evaluation of Generic Convolutional and Recurrent Networks for Sequence Modeling},
  year          = {2018},
  eprint        = {1803.01271},
  archivePrefix = {arXiv},
  url           = {https://arxiv.org/abs/1803.01271v2},
  note          = {arXiv:1803.01271v2}
}

@misc{CFTC2017Risk,
  author       = {{U.S. Commodity Futures Trading Commission}},
  title        = {Customer Advisory: Understand the Risks of Virtual Currency Trading},
  year         = {2017},
  month        = dec,
  howpublished = {\url{https://www.cftc.gov/sites/default/files/idc/groups/public/@customerprotection/documents/file/customeradvisory_urvct121517.pdf}},
  note         = {Customer advisory},
  url          = {https://www.cftc.gov/sites/default/files/idc/groups/public/@customerprotection/documents/file/customeradvisory_urvct121517.pdf}
}

@article{Han2026ESWA,
  author  = {Han, Fei and Ling, Qinghua and Lu, Sheng and Han, Henry},
  title   = {Feature Enrichment Imitative Reinforcement Learning for High-Frequency Trading},
  journal = {Expert Systems with Applications},
  volume  = {296},
  pages   = {129043},
  year    = {2026},
  doi     = {10.1016/j.eswa.2025.129043},
  url     = {https://doi.org/10.1016/j.eswa.2025.129043}
}

@article{FernandezVicente2026ESWA,
  author  = {Fern{\'a}ndez Vicente, {\'O}scar and Garc{\'i}a, Javier and Fern{\'a}ndez, Fernando},
  title   = {Optimizing Market-Making Strategies: A Multi-Objective Reinforcement Learning Approach with {P}areto Fronts},
  journal = {Expert Systems with Applications},
  volume  = {295},
  pages   = {128867},
  year    = {2026},
  doi     = {10.1016/j.eswa.2025.128867},
  url     = {https://doi.org/10.1016/j.eswa.2025.128867}
}

@article{Guo2025ESWA,
  author  = {Guo, Zhaori and Yu, Haitao and Norman, Timothy J. and Gerding, Enrico H. and Auricchio, Gennaro and Cai, Zhongqi},
  title   = {A Cautious Multi-Advisor Sequential Decision-Making Strategy without Ground Truth for Maximizing the Profits},
  journal = {Expert Systems with Applications},
  volume  = {286},
  pages   = {128004},
  year    = {2025},
  doi     = {10.1016/j.eswa.2025.128004},
  url     = {https://doi.org/10.1016/j.eswa.2025.128004}
}

@article{Kucukoglu2026ESWA,
  author  = {K{\"u}{\c{c}}{\"u}ko{\u{g}}lu, Sami},
  title   = {Beyond Random Walks: Exploring the Learnability Threshold of {AI} Agents in Algorithmic Markets},
  journal = {Expert Systems with Applications},
  volume  = {315},
  pages   = {131776},
  year    = {2026},
  doi     = {10.1016/j.eswa.2026.131776},
  url     = {https://doi.org/10.1016/j.eswa.2026.131776}
}

@article{ANON2026134021,
  author  = {Anonymous},
  title   = {Self-citation masked for double-blind review},
  journal = {Expert Systems with Applications},
  year    = {2027}
}

@inproceedings{Henderson2018DeepRL,
  author    = {Henderson, Peter and Islam, Riashat and Bachman, Philip and Pineau, Joelle and Precup, Doina and Meger, David},
  title     = {Deep Reinforcement Learning that Matters},
  booktitle = {Proceedings of the Thirty-Second AAAI Conference on Artificial Intelligence},
  year      = {2018},
  url       = {https://arxiv.org/abs/1709.06560},
  note      = {arXiv:1709.06560}
}

@article{BaileyLopezdePrado2014,
  author  = {Bailey, David H. and L{\'o}pez de Prado, Marcos},
  title   = {The Deflated Sharpe Ratio: Correcting for Selection Bias, Backtest Overfitting, and Non-Normality},
  journal = {The Journal of Portfolio Management},
  volume  = {40},
  number  = {5},
  pages   = {94--107},
  year    = {2014},
  doi     = {10.3905/jpm.2014.40.5.094},
  url     = {https://doi.org/10.3905/jpm.2014.40.5.094}
}

@article{Gu2020EmpiricalAssetPricing,
  author  = {Gu, Shihao and Kelly, Bryan and Xiu, Dacheng},
  title   = {Empirical Asset Pricing via Machine Learning},
  journal = {The Review of Financial Studies},
  volume  = {33},
  number  = {5},
  pages   = {2223--2273},
  year    = {2020},
  doi     = {10.1093/rfs/hhaa009},
  url     = {https://doi.org/10.1093/rfs/hhaa009}
}

@article{Cont2014OrderBookImpact,
  author  = {Cont, Rama and Kukanov, Aleksandr and Stoikov, Sasha},
  title   = {The Price Impact of Order Book Events},
  journal = {Journal of Financial Econometrics},
  volume  = {12},
  number  = {1},
  pages   = {47--88},
  year    = {2014},
  doi     = {10.1093/jjfinec/nbt012},
  url     = {https://doi.org/10.1093/jjfinec/nbt012}
}

@inproceedings{Wu2021Autoformer,
  author    = {Wu, Haixu and Xu, Jiehui and Wang, Jianmin and Long, Mingsheng},
  title     = {Autoformer: Decomposition Transformers with Auto-Correlation for Long-Term Series Forecasting},
  booktitle = {Advances in Neural Information Processing Systems},
  volume    = {34},
  year      = {2021},
  url       = {https://arxiv.org/abs/2106.13008},
  note      = {arXiv:2106.13008}
}

@inproceedings{Yao2023ReAct,
  author    = {Yao, Shunyu and Zhao, Jeffrey and Yu, Dian and Du, Nan and Shafran, Izhak and Narasimhan, Karthik and Cao, Yuan},
  title     = {{ReAct}: Synergizing Reasoning and Acting in Language Models},
  booktitle = {International Conference on Learning Representations},
  year      = {2023},
  url       = {https://arxiv.org/abs/2210.03629},
  note      = {ICLR 2023; arXiv:2210.03629}
}
